%% file: main.tex
\documentclass{article}

\usepackage{iclr2026_conference,times}

\usepackage{hyperref}       % hyperlinks
\usepackage{url}            % simple URL typesetting
\usepackage{booktabs}       % professional-quality tables
\usepackage{amsfonts}       % blackboard math symbols
\usepackage{nicefrac}       % compact symbols for 1/2, etc.
\usepackage{microtype}      % microtypography
\usepackage{amsthm}
\usepackage{mathtools}
\usepackage{graphicx}
\usepackage{subcaption}
\usepackage{dsfont}
\usepackage{multirow}
\usepackage{wrapfig}

\usepackage{amssymb}
\usepackage{multirow}
\usepackage[capitalize,noabbrev]{cleveref}

\usepackage{enumitem}
\input{math_commands.tex}

\usepackage{upgreek}
\usepackage{rotating}

\theoremstyle{definition}

\newtheorem{theorem}{Theorem}[section]

\newcommand{\td}{{\bf\color{red} FIXME~}}

\title{\textit{First is Not Really Better Than Last}: Evaluating Layer Choice and Aggregation Strategies in Language Model Data Influence Estimation}

\author{%
  Dmytro~Vitel \& Anshuman~Chhabra 
  \\
  Bellini College of Artificial Intelligence, Cybersecurity, and Computing, \\
  University of South Florida\\
  \texttt{\{dvitel,anshumanc\}@usf.edu}\\
}

\iclrfinalcopy

\begin{document}

\maketitle
\vspace{-5mm}

\begin{abstract}
\vspace{-2mm}

% Identifying influential training samples allows for model decision explanation and data debugging. To a high degree, the task of detecting data anomalies aligns with the search for the least influential samples according to the present data attribution methods. Because these attributions are done across many model parameters, selecting those parameters that give the best downstream task performance in a feasible computational time is essential. Previous works consider either a classification head or the embeddings for such selection. This work, in contrast, demonstrates that middle attention layers of deep models capture detrimental samples more efficiently than others. We also target a more general question of combining attributions measured around the network with minimal influence score compensation. As an alternative to commonly used influence averaging, we propose ranking and voting that significantly boost model accuracy after anomaly filtering across many configurations. Since the validation of new methods requires expensive model retraining after filtering, it is crucial to have a good proxy measure, the predictive indicator of expected performance. We show that the previously proposed cancellation effect~\citep{Yen-2022} does not serve as a good indicator, while the noise detection rate demonstrates moderate-to-strong correlation.

Identifying how training samples influence/impact Large Language Model (LLM) decision-making is essential for effectively interpreting model decisions and auditing large-scale datasets. Current training sample influence estimation methods (also known as \textit{influence functions}) undertake this goal by utilizing information flow through the model via its first-order and higher-order gradient terms. However, owing to the large model sizes of today consisting of billions of parameters, these influence computations are often restricted to some subset of model layers to ensure computational feasibility. Prior seminal work by \cite{Yen-2022} in assessing which layers are best suited for computing language data influence concluded that the first (\textit{embedding}) layers are the most informative for this purpose, using a hypothesis based on influence scores canceling out (i.e., the \textit{cancellation effect}). In this work, we propose theoretical and empirical evidence demonstrating how the cancellation effect is \textit{unreliable}, and that \textit{middle} attention layers are better estimators for influence. Furthermore, we address the broader challenge of \textit{aggregating} influence scores across layers, and showcase how alternatives to standard averaging (such as ranking and vote-based methods) can lead to significantly improved performance. Finally, we propose better methods for evaluating influence score efficacy in LLMs without undertaking model retraining, and propose a new metric known as the Noise Detection Rate (NDR) that exhibits strong predictive capability compared to the cancellation effect. Through extensive experiments across LLMs of varying types and scales, we concretely determine that the \textit{first} (layers) are not \textit{necessarily better} than the \textit{last} (layers) for LLM influence estimation, contrasting with prior knowledge in the field.

%for downstream task improvement in a computationally feasible manner is critical. Prior work has focused on either the classification head or embeddings, but we show that the middle attention layers of deep models capture detrimental samples more effectively. We further address the broader challenge of aggregating influence scores across layers while minimizing signal cancellation. As alternatives to standard averaging, we propose Rank- and Vote-based aggregation methods, which substantially improve model accuracy after anomaly filtering across multiple configurations. Finally, we evaluate proxy metrics for method validation without retraining, showing that the previously proposed cancellation effect is unreliable, whereas the Noise Detection Rate exhibits moderate-to-strong predictive correlation.

\end{abstract}\vspace{-5mm}

\section{Introduction}
\vspace{-2mm}

\looseness-1Large Language Models (LLMs) have demonstrated stellar performance on tasks across a number of applications and domains~\citep{Street-2024,Mittelst-2024,Marco-2025}. 
% \looseness-1Large Language Models (LLMs) have achieved remarkable performance across a wide range of applications and domains. 
Despite these advancements, current state-of-the-art models still 
exhibit
suboptimal behavior on complex reasoning tasks~\citep{Jiang-2025}, hallucinate responses and facts~\citep{Cleti-2024}, and can make biased and unfair decisions~\citep{Gallegos-2024, Peters-2025}. To improve the trust and safety of LLMs, it is imperative to better interpret and understand model decision-making~\citep{Singh-2024}. 
% Despite that, they have deficiencies and can make incorrect decisions -- these require the model to be interpretable

Recently, data attribution and influence methods~\citep{Hammoudeh-2024} have shown great promise 
% in helping interpret LLM decisions and explain model behaviors 
in interpreting LLM behavior
from the perspective of the training data \citep{Grosse-2023, Chhabra-2025}. These approaches seek to detect 
% detrimental, 
noisy,
anomalous, out-of-distribution, or problematic training samples and conceptually link them to model output performance~\citep{Pleiss-2020,Yang-2024,Jiang-2021}. 
% It is important to note that detrimental samples have been shown to occur (likely due to human error or malicious adversaries) even in highly curated datasets~\citep{Ekambaram-2017}. 
Even highly curated datasets can contain detrimental samples, often introduced unknowingly through human error or surreptitiously by malicious adversaries~\citep{Ekambaram-2017}.
%Training-time poisoning attacks insert anomalies into input data to gain control over model predictions~\citep{Shafahi-2018,Hammoudeh-2023}.
%In addition, datasets frequently exhibit encoded biases on protected characteristics~\citep{Zhang-2020, Hutchinson-2020}. 
%Regardless of the cause, defective training samples degrade a model’s performance. Different methods of signal measuring across the model structure could help detect these anomalies (figure~\ref{fig:nn-infl-cmp}).
More specifically, \textit{influence functions}~\citep{Garima-2020,Sui-2021,Kwon-2024} employ gradient-based analysis to help developers understand how models' decisions are influenced by training data.
% samples.
% Despite their obvious benefits, a fundamental issue in influence computation is the computational overhead of utilizing the entire gradient space from all the layers of the LLM, which, for the models of today, can constitute several billion parameters.
Despite their clear benefits, a fundamental issue in influence computation is the computational overhead of utilizing the entire gradient space from all the layers of the LLM, which, for modern models, amount to billions of parameters.

\begin{wrapfigure}{r}{0.6\textwidth}
    \centering
    \includegraphics[width=\linewidth]{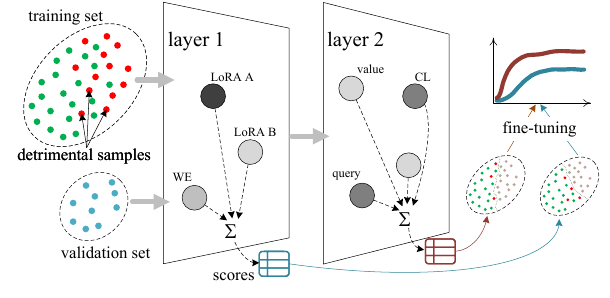}\vspace{-4mm}
    \caption{The influence estimation pipeline for LLMs.
    % \td change noisy samples -> detrimental samples. Also move the text for noisy samples and validation set a little lower so it is not being cut by the circles
    }
    \label{fig:nn-infl-cmp}
\end{wrapfigure}\vspace{-3mm}

% Thus, 
A common workaround (Figure \ref{fig:nn-infl-cmp}) when employing influence functions in LLMs is to restrict the gradient input to only certain layers~\citep{Yen-2022,Chhabra-2025}. However, there is little consensus on what the best layers are for influence estimation, and it is not known whether the output head, attention layers, or word embedding layers are better suited for this purpose. This is partly because very little work has been undertaken to analyze which layers are more useful for influence analysis. More specifically, one prior seminal work by \citet{Yen-2022} analyzes layer choice in computing language model data influence and finds that the first few (embedding) layers are most beneficial for influence estimation. However, as we discuss subsequently in our paper, these conclusions are derived from a strong assumption about the observed effectiveness of the layers being used, which might not hold in practice, requiring a deeper investigation into the effect of layers. Additionally, the work only studies one influence method and smaller-sized language models such as RoBERTa~\citep{Liu-2019-R}, whereas we undertake our study across various LLM parameter sizes and recent models, thereby capturing general patterns of influence estimation as models scale.

% Moreover, influence functions can leverage gradients from various layers of a model, and a commonly used approach is to aggregate these influences by averaging across parameter subsets. 
Moreover, influence functions that utilize gradients from multiple layers generally aggregate their contributions through averaging across parameter subsets.
However, this averaging can obscure important distinctions by allowing compensatory effects between layers, potentially masking the true contribution of individual training samples. In this work, we also explore and propose alternative strategies for combining influence estimates computed independently across different layers, aiming to preserve layer-specific insights and improve the interpretability of influence attribution.

% New data attribution methods that rely on influence functions and various aggregation strategies must be evaluated to determine how effectively they can identify detrimental training samples. 
Finally, assessing new influence-based attribution methods involves measuring how accurately they identify detrimental training samples.
In practice, this evaluation is typically performed by removing the least influential samples, fine-tuning the model on the filtered dataset, and then observing the impact on downstream task performance. However, this process is computationally expensive, particularly when influence is measured across different layers or parameter subsets of a model. As a result, it becomes valuable to identify a reliable extrinsic indicator -- one that can be computed efficiently and correlates well with the final task accuracy. Such a metric can guide the development and selection of influence-based attribution methods without requiring exhaustive fine-tuning experiments. To this end, we also propose a novel metric (the Noise Detection Rate) and show how it can be used as a reliable proxy indicator of downstream influence estimation performance.

\textbf{Contributions.} In sum, our contributions are as follows:\vspace{-2mm}
\begin{itemize}[leftmargin=1em, nosep]
\item 
% We study how well a previously proposed measure, the cancellation effect~\citep{Yen-2022}, determines layers best suited for detrimental samples filtering. Our empirical and theoretical evidence suggests that in the context of filtering of many samples, the cancellation effect is not a viable indicator metric. 

We evaluate the cancellation effect metric proposed by ~\cite{Yen-2022} for identifying layers most suitable for influence estimation. Our theoretical analysis and experiments demonstrate that, when filtering large numbers of samples, the cancellation effect is an unreliable indicator.

\item 
% We analyze how detrimental training samples influence language model behavior, focusing on identifying which model layers are most effective for influence estimation (RQ1). To this end, we conduct a comprehensive empirical study using mislabeled or noisy examples in the context of the widely used GLUE benchmark datasets. Our evaluation spans several state-of-the-art language models -- including Llama-3.2, Qwen-2.5, and Mistral 7B -- and helps address open research questions from previous work.
We quantify the impact of detrimental training samples on language model behavior and identify the model layers most informative for influence estimation. We conduct extensive experiments using noisy versions of the GLUE benchmark as in prior influence work \citep{Kwon-2024, Chhabra-2025} and employ several LLMs such as Llama-3.2, Qwen-2.5, and Mistral 7B.
% thereby addressing open questions from prior work.

\item 
% We propose novel strategies for aggregating influence scores across layers, validating that these methods often outperform the standard averaging approach in terms of effectiveness (RQ2).
We also introduce novel strategies for aggregating influence scores across layers, and demonstrate that they frequently outperform the standard averaging method used in current functions.
\item 
% We assess the suitability of several proxy measures (RQ3) -- training gradient cancellation and noise detection rate -- as reliable indicators of downstream influence estimation performance, with the goal of minimizing reliance on exhaustive fine-tuning for assessing influence estimation accuracy.
We propose and evaluate the validity of noise distribution-based proxy measures: Noise Detection Rate (NDR) metric and Area Under the Curve (AUC), as reliable indicators of downstream influence estimation performance, with the goal of minimizing reliance on exhaustive fine-tuning for assessing influence estimation accuracy. 
% Our theoretical analysis and experiments demonstrate that, when filtering large numbers of samples, the cancellation effect is an unreliable indicator.
\end{itemize}
\vspace{-2mm}

\section{Related Works}
\vspace{-2mm}

Our work falls within the broader area of interpretable machine learning, particularly focusing on gradient-method estimations of training data importance~\citep{Koh-2017, Yeh-2018, Yeh-2019, Jia-2019, Garima-2020, Khanna-2018, Sui-2021}. 
% Explainable AI encompasses several established approaches~\citep{Wang-2024}, including self-explanatory models~\citep{Lee-2022} 
% such as decision trees and rule-based systems~\citep{Lee-2022}, surrogate-based model explanations~\citep{Lakkaraju-2019}, and outcome explanations that rely on feature attribution, counterfactual reasoning~\citep{Mothilal-2021}, or input data attribution~\citep{Sundararajan-2017}, as well as internal model inspection~\citep{Goldstein-2013}. 
% Feature attribution methods are typically classified into gradient-based~\citep{Sundararajan-2017}, perturbation-based~\citep{Fong-2019}, and decomposition-based approaches~\citep{Montavon-2019}. Among these, gradient-based techniques help estimate how infinitesimal changes in feature values influence model predictions.
In gradient-based models trained via empirical risk minimization, influence functions offer a practical means to estimate the impact of individual training samples without performing costly leave-one-out retraining. For deep learning, the foundational work \citep{Koh-2017} introduced a Taylor-series approximation combined with the LiSSA optimization algorithm \citep{Agarwal-2017} to make influence estimation tractable. Subsequent efforts, such as Representer Point \citep{Yeh-2018} and Hydra \citep{Chen-2021}, aimed to improve these estimates, though their focus remained primarily on vision tasks. More recent advancements, including DataInf~\citep{Kwon-2024}, Arnoldi iteration~\citep{Schioppa-2022}, and Kronecker-factored curvature approximations \citep{Grosse-2023}, have scaled influence estimation to large generative language models like LLMs. Simpler methods have also emerged, using raw gradients as a proxy for influence ~\citep{Garima-2020, Charpiat-2019}, often enhanced with ensemble-based strategies \citep{Bae-2024, Kim-2023}. Additionally, self-influence -- computed solely within the training set -- has been shown to be effective for identifying influential samples \citep{Bejan-2023, Thakkar-2023}. Influence functions have proven valuable across diverse applications, including classification \citep{Koh-2017, Chhabra-2025, Chhabra-2024b}, generative modeling \citep{Kwon-2024, Schioppa-2022, Grosse-2023}, active learning \citep{Liu-2021b}, in-context learning~\citep{Nguyen-2023, Van-2024, Askari-2025}, and layer quality estimation~\citep{askari2025layerif}.

% The phenomenon of compensation in feature-attributed gradients has been explored in earlier work. For instance, \citep{Liu-2020} demonstrates how regularization techniques during training can reduce such cancellation, while \citep{Kapishnikov-2021} introduces a method to mitigate cancellation in the Integrated Gradients framework~\citep{Sundararajan-2017}.
% In the domain of data attribution, prior studies have leveraged influence functions to identify spurious data artifacts and improve model performance~\citep{Han-2020, Pezeshkpour-2022,Askari-2025}. 
% A notable limitation addressed in \citep{Hammoudeh-2024} is the dependence of many gradient-based methods on the dot product, which may cause influential training examples to be systematically overlooked due to cancellation across model weights. Similarly, \citep{Yen-2022} hypothesizes that high gradient cancellation during training can impair the discriminative power of influence scores. Building on this line of work, our study investigates whether such cancellation adversely affects mislabeled data detection and directly tackles the challenge of dot-product compensation.

Gradient compensation in feature-attributed gradients has been studied previously~\citep{Liu-2020,Kapishnikov-2021}. In data attribution, influence functions have been used to detect spurious training artifacts and improve performance~\citep{Han-2020,Pezeshkpour-2022,Askari-2025}. Prior work~\citep{Hammoudeh-2024,Yen-2022} highlights that dot-product-based methods and high gradient cancellation may reduce the discriminative power of influence scores. We extend this line of work by empirically and theoretically investigating whether such cancellation impairs the detection of mislabeled data. %Owing to space limitations, we discuss additional related work in Appendix \ref{appendix:addn-rel-work}.

It is also interesting to note that prior work has studied the \textit{applicability} of influence functions in LLMs, with varying results. Clearly, while influence functions have been used successfully across several LLM tasks and applications \cite{wanggeneralization, zhang2025correcting, askari2025layerif, xia2024less}, recent work \cite{Li-2025} finds that Hessian-Free, DataInf, and LiSSA perform poorly on tasks such as harmful data detection, class attribution, and backdoor identification, due to iHVP approximation errors, uncertain convergence, and weak correlation between parameter changes and model behavior. Their experiments use default influence-function implementations with mean aggregation across all layers. In contrast, we extend these insights by identifying specific layers and aggregation strategies that achieve performance capable of exceeding the default settings in~\cite{Li-2025}. As part of future work, we believe indetifying optimal layers for influence analysis can significantly improve the efficacy and fidelity of influence functions in LLMs, as observable in our main paper results.

Moreover, given our focus on \emph{where} influence scores are strongest and \emph{how} individual valuations can be effectively combined, we relate our analysis to knowledge-editing methods that similarly perform localization and aggregation. ROME~\citep{Meng-2022} and R-ROME~\citep{Gupta-2024} use \textit{causal tracing} to identify the MLP layer encoding a fact and apply a rank-one edit. MEMIT~\citep{Meng-2023} generalizes this via per-layer causal-effect scores and top-$k$ distributed edits, while EMMET~\citep{Gupta-2024-2} refines aggregation to reduce redundancy. In contrast, we use gradient-based influence to locate layers most affected by beneficial or harmful training samples and, consistent with KE studies, find the strongest discriminative signals in middle layers.

\section{Preliminaries and Notation}
\label{sec:prelim}\vspace{-2mm}

\subsection{Influence functions}
\label{sec:infl}

% Influence function originates from assessing the upweighting of a training sample on parameter estimation in optimization~\citep{Hampel-1974}. Machine learning extends its usage to estimating the impact of training samples on prediction loss~\citep{Koh-2017}, and on any arbitrary utility $f$. We consider how a training sample $\bar{x} \in X$ affects the values of $f$ on the validation set $\bar{x}' \in X'$. This influence is defined as follows. 

Given a training sample $\bar{x} \in X$, model weights $\Theta$, and a validation sample $\bar{x}' \in X'$, the influence is defined as follows for the utility function $f$.
\begin{align}
    I(\bar{x}, X', \Theta) = \left(\frac{1}{|X'|} \sum_{\bar{x}' \in X'}\nabla_\Theta f(\bar{x}', \Theta) \right)^T H^{-1}_\Theta \nabla l(\bar{x}, \Theta)
    \label{eq:infl}
\end{align}
% The value $I$ has a particular meaning regarding directions in the search space of $\Theta$. 
% The term under the parentheses defines the vector in the search space $\Theta$. 
The optimization step $\Delta \Theta_1$ in the direction of the vector given by the term under the parentheses decreases $f$ on $X'$. The second term, $H^{-1}_\Theta \nabla l(\bar{x}, \Theta)$ ($H$ - Hessian matrix) defines the optimization step $\Delta \Theta_2$ that would minimize the loss on the sample $\bar{x}$. 
% It is expressed through second-order derivatives in the Hessian matrix $H$. 
%%Thus, Eq.~\ref{eq:infl} estimates the alignment of these two optimization tasks.
% , the minimization of $f$ on $X'$, and the minimization of $l$ on $\bar{x}$. 
% When $\Delta \Theta_1$ and $\Delta \Theta_2$ are codirectional, minimizing training sample loss would also minimize the $f$, the two optimization goals match. The influence of sample $\bar{x}$ would be high. Conversely, when a step that minimizes loss increases the utility $f$, the influence score has a low value. 

% In the context of the mislabeled training sample filtering task with influence scores, we consider the loss on the validation sample as utility $f$. 

% It is generally expected that 
% The loss optimization on some training samples could be misaligned with the optimization of $f$, and, therefore, would be ``less influential'' (low or negative $I$ values) on $X'$~\citep{Chhabra-2025}. 

The loss optimization on some training samples could be misaligned with the optimization of $f$. Therefore, fitting these samples would increase $f$. We refer to such samples as \textit{detrimental} for the utility under consideration. This definition corresponds to the previously used definition in~\citep{Koh-2017, Kwon-2024, Chhabra-2025}. %Detrimental samples frequently correspond to present noise and mistakes in the training dataset.  
%It is known that deep models demonstrate the ability to fit both the noise and clean data, and, in some cases, noisy labels serve as a regularization technique to improve generalization~\citep{Muller-2019, Xie-2020, Song-2023}.
%
% increase this loss, meaning that the optimization goals $\Theta_{\bar{x}}$ and $\Theta_{\bar{x}'}$ would not correlate. 
% However, deep models demonstrate the ability to fit both the noise and clean data. In some cases, noisy labels serve as a regularization technique to improve generalization~\citep{Muller-2019, Xie-2020, Song-2023}.
%
% The formula~\ref{eq:infl} defines values through the inverse of the Hessian, the second-order derivative of the loss. 
In practice, it is not feasible to compute $I$ on large models due to the time complexity of the Hessian inversion (cubic time in the exact case). Therefore, influence is estimated approximately. 
% This makes the influence calculation infeasible for deep models in terms of memory and clock time. 
% Many authors consider the approximations of $ H^{-1}_\Theta \nabla l(\bar{x}, \Theta)$ for deep architectures without sacrificing precision for specific conditions. 
The TracIn method~\citep{Garima-2020} replaces the \textit{Hessian vector product} $\mathcal{H} = H^{-1}_\Theta \nabla l(\bar{x}, \Theta)$ with the gradient $\nabla l(\bar{x}, \Theta)$.
% , effectively moving from second-order to first-order optimization.
Cosine similarity between the utility term and $\mathcal{H} \approx \nabla l(\bar{x}, \Theta)$ serves as another first-order influence estimation approximation.
% LiSSA~\citep{Agarwal-2017} presents an interactive $\mathcal{H}$ estimation with Taylor expansion. 
DataInf~\citep{Kwon-2024} is a recently proposed second-order method based on swapping the order of matrix inversion and averaging across training samples in $\mathcal{H}$.
% , computing HVP through available gradients $\nabla l(\bar{x}_i, \Theta)$. 
Its error decreases when $|\Theta|$ is small, which makes this method suitable for fine-tuning with LoRA~\citep{Hu-2022, Dettmers-2023}. 
% But before we state all applied methods formally, we also need to consider further decomposition of influence estimation in formula~\ref{eq:infl}.

% For deep neural network models, when the loss is the negative log-likelihood function, 
\looseness-1For cross-entropy loss,
the Hessian is the block diagonal matrix of gradient second moments $|X|^{-1} \sum_{\bar{x} \in X} \nabla l(\bar{x}, \Theta) \nabla l^T(\bar{x}, \Theta)$, one block for one layer $l \in L$ in the neural network. Thus, $I$ is an aggregation of influence values across validation samples and model layers, defined as follows:
% formula~\ref{eq:infl} could be transformed to the following.
% \begin{align*}
%     I(\bar{x}, X', \Theta) = \sum_{l \in L} \left(\frac{1}{|X'|} \sum_{\bar{x}' \in X'}\nabla_{\Theta_l} f(\bar{x}', \Theta_l) \right)^T \mathcal{H}(\bar{x}, \Theta_l) 
% \end{align*}
\vspace{-0.5mm}
{ 
\small
\begin{align}
    I(\bar{x}, X', \Theta) = \frac{1}{|X'|}  \sum_{\bar{x}' \in X', l \in L} I'(\bar{x}, \bar{x}', \Theta_l) \text{ s.t. } I'(\bar{x}, \bar{x}', \Theta_l) = \nabla_{\Theta_l} f^T(\bar{x}', \Theta_l) \mathcal{H}(\bar{x}, \Theta_l)
    \label{eq:infl2}
\end{align}
}%\vspace{-4mm}

  % Vectors $HVP(\bar{x}, \Theta_l)$ are computed with different Hessian vector product approximation methods.   
$I'(\bar{x}, \bar{x}', \Theta_l)$ are tensors collected on different layers and for different influence functions. 
  % The normalization coefficient is not important for training sample ordering according to formula~\ref{eq:infl2}. Noisy training samples are expected to have small accumulated scores.
% This work compares label noise detection and filtering for three commonly used influence methods, 
We employ TracIn, Cosine, and DataInf in our experiments ($\nabla l_{\bar{x}} = \nabla_{\Theta_l} l(\bar{x}, \Theta_l)$), analytically defined as follows:
% , with the following influence matrices $I'$ ($\nabla l_{\bar{x}} = \nabla_{\Theta_l} l(\bar{x}, \Theta_l)$).
\vspace{-3mm}
{
\small
\begin{align}
    \operatorname{TracIn}(\bar{x}, \bar{x}', \Theta_l) = \nabla l^T_{\bar{x}'} \nabla l_{\bar{x}}.
    \label{eq:tracin}
\end{align}\vspace{-4mm}
\begin{align}
    \operatorname{Cosine}(\bar{x}, \bar{x}', \Theta_l) = \frac{\nabla l^T_{\bar{x}'} \nabla l_{\bar{x}}}{|\nabla l^T_{\bar{x}'}| |\nabla l^T_{\bar{x}}|}.
    \label{eq:cosine}
\end{align}\vspace{-2mm}
\begin{align}
    \operatorname{DataInf}(\bar{x}, \bar{x}', \Theta_l) = \frac{\nabla l^T_{\bar{x}'} \nabla l_{\bar{x}} - |X|^{-1} \sum_{\bar{z} \in X} \nabla l^T_{\bar{x}'} \nabla l_{\bar{z}} \nabla l^T_{\bar{z}} \nabla l_{\bar{x}} }{(|X||l| \lambda)^{-1} \sum_{\bar{x} \in X} |\nabla l_{\bar{x}}|^2 }.
    \label{eq:datainf}
\end{align}
\vspace{-2mm}
}
% Note that Cosine value differs from TracIn by a constant for a given $\bar{x}$ and $\bar{x}'$. However, this constant varies across different training-validation pairs, which makes the final influence scores of the two methods different. 

% \textbf{Influence on high-parameter layers}. 
% Aiming at the question of which layers detect better the detrimental samples,
% % in their aggregated influence scores (formula~\ref{eq:infl2})
% we consider two influence functions proposed in the work~\citep{Yen-2022}. 
To address the question of which layers are more effective at detecting detrimental samples, we also consider two influence functions proposed by~\citet{Yen-2022}.
Both of them,  $\operatorname{TracIn}_{we}$ and $\operatorname{TracIn}^{10}_{we}$, operate on word embeddings. The first, $\operatorname{TracIn}_{we}$, considers only the weights 
of the tokens (including special tokens) in both the training and the validation samples: 
% If $\Theta_{we}$ are embedding weights, and $\bar{x} \cap \bar{x}'$ are tokens present in both samples, the scores are defined as follows. 
\begin{align}
    \operatorname{TracIn}_{we}(\bar{x}, \bar{x}', \Theta_{we}) = \sum_{t \in \bar{x} \cap \bar{x}', \theta_t \in \Theta_{we}} \frac{\partial l(\bar{x}', \theta_t)}{\partial \theta_t} \frac{\partial l(\bar{x}, \theta_t)}{\partial \theta_t}.
    \label{eq:tracin-we}
\end{align}\vspace{-3mm}

% We should note that the method works even if training and validation samples do not have shared visible tokens. 
% The influence is then attributed through utility tokens such as begging, end of the sequence, and sentence separation. 

The second, $\operatorname{TracIn}^{10}_{we}$ additionally filters out the weights $t \in \bar{x} \cap \bar{x}', \theta_t \in \Theta_{we}$ keeping only the top-10 that have highest-by-magnitude gradients for training sample $\bar{x}$.

% TODO. The computation of influence scores happens on the neural network submodular level. 

\iffalse
Methods $\operatorname{TracIn}_{we}$ and $\operatorname{TracIn}^{10}_{we}$ were also justified with the complexity of computing TracIn on all word embedding parameters. 
% Because $\operatorname{TracIn}^{10}_{we}$ requires even fewer gradients than $\operatorname{TracIn}_{we}$, the method is tractable for long sentences. 
$\operatorname{TracIn}^{10}_{we}$ is tractable for long sentences because it requires even fewer gradients than $\operatorname{TracIn}_{we}$. 
However, the parameter selection of these methods is still computationally expensive because it requires token-wise analysis of sample pairs. In this work, we collect influence scores on \texttt{WE} using embedding compression and computation batching (App.~\ref{sec:a-exp}).
\fi

\subsection{Cancellation effect}
\label{sec:cancel}

\looseness-1In their study of layer choice effects in influence estimation, \citet{Yen-2022} consider the compensation of gradients as contributed by training samples. They argue that if the sum of gradients w.r.t a particular weight $\theta$ becomes very small,
% due to different directions, attractors in the search space, 
then the optimizer would not change $\theta$ significantly because of them. 
% In other words, the optimizer does not change the weight significantly. 
However, the magnitudes of the summed training gradients could be very large and hence, the $\theta$-contributed influence score could also be large. The authors further state that $\theta$, the weight with a high cancellation effect, should be ignored in influence estimation because it reduces the model's discrimination power. To this end, they propose a metric $C$ to estimate the cancellation effect on a layer $W$ and present the results where smaller $C$ values lead to better model confidence \citep{Yen-2022}:
% In other words, $C$ is a potential indicator of the downstream task performance.
\begin{align}
C(W) = \frac{\sum_{\bar{x}} |\nabla_W l(\bar{x}, W)|}{|\Delta W|}.
\label{eq:cancel}
\end{align}

$C$ can be computed throughout the fine-tuning procedure, considering every checkpoint individually. 
% Therefore, stochastic batching, dropout, and normalization affect final values.
% (for instance, work~\citep{Liu-2020} discusses how regularization techniques reduce the cancellation for feature attributions). 
% \looseness-1
Accumulating the magnitudes of atomic gradient updates per sample requires \textit{size-one batching} during training, which is computationally expensive and infeasible for current models. We thus adopt the original formulation to estimate the cancellation on a single checkpoint in evaluation mode. 

\subsection{Proposed Influence Aggregation Approaches}
\label{sec:appr}

% Equation~\ref{eq:infl2} expresses the scores through the influence matrices and averages them across validation samples and selected layers. This leads to \textit{influence compensation} -- scores of different layers and validation samples could zero each other. Additionally, some of them could dominate others by score magnitude. 
% We investigate alternative aggregation methods with goal to remove these effects.
% Generally, influence scoring could be considered as follows.
As is evident, Eq.~\ref{eq:infl2} averages influence scores across all layers and validation samples. This can cause influence \textit{compensation}, where opposing contributions cancel out, or \textit{dominance}, where large scores overshadow others. In our work, we aim to investigate alternative aggregation strategies that can be employed instead to \textit{mitigate} these effects. Generally, influence attribution can be abstracted as:
\begin{align}
    I(\bar{x}, X', \Theta_L) = \mathcal{A}_{X', L}(I'(\cdot, \bar{x}', \Theta_l))(\bar{x}) ,
    \label{eq:infl-agg}
\end{align}
where $\mathcal{A}$ is an \textit{aggregation operator} (e.g., taking the average across samples/layers). We thus propose our novel aggregation approaches as follows:
% of collected influence matrices on selected layers $L$.

% One validation sample defines its influence order on $X$. Aggregation methods $\mathcal{A}$ define ways to combine them into final scores. Apart from averaging, we consider ranking and voting as alternatives in our experiments. 

%They were identified as the best among attempted according to the approach discussed in section~\ref{sec:best-layers}. 
% Same section proposes the way how to compare different $\mathcal{A}$ before expensive fine-tuning.
% that are most promising by NDR measure. 

\textbf{Ranking}. Every validation sample and layer ranks training samples according to their influence. These ranks are summed, eliminating the domination of individual influences with high magnitude. Additionally, our proposed \texttt{Rank} method ignores incorrectly predicted validation samples by the selected checkpoint:
\begin{align}
    \operatorname{Rank}(I') = 
    % \alpha 
    \sum_{\bar{x}' \in X'', l \in L} \sum_{\bar{y} \in X} \mathbb{I}(I'(\bar{y}, \bar{x}',\Theta_l) < I'(\cdot, \bar{x}',\Theta_l)),
\end{align}
where $\mathbb{I}$ is the indicator function, 
% $\alpha$ - normalization, 
and $X''$ denotes the correctly predicted validation samples. 

\textbf{Positional voting}. To avoid the domination of very low or very high ranks in averaging, we propose a method based on voting. Every validation sample and layer assigns $k$ votes to the least influential training datum, $k-1$ to the next one, etc., until the number of votes reaches zero. We then pick $k$ equal to the number of training samples for filtering as potentially detrimental. As with \texttt{Rank}, the \texttt{Vote} method considers only correctly predicted validation samples by the selected checkpoint:
\begin{align}
    \operatorname{Vote}_k(I') = -\sum_{\bar{x}' \in X'', l \in L}\max(k - \sum_{\bar{y} \in X} \mathbb{I}(I'(\bar{y}, \bar{x}',\Theta_l) < I'(\cdot, \bar{x}',\Theta_l)), 0).    
    \label{eq:vote-agg}
\end{align}

\section{Research questions}
% \label{sec:problem}
\label{sec:rqs}

\looseness-1 We now formalize the research questions (RQs) we seek to study in this paper. To bridge the gaps identified above regarding the role of layers in LLM influence analysis, the impact of layer-wise aggregation strategies in influence functions, and whether influence functions can be evaluation through novel external metrics, without requiring costly fine-tuning, we propose the following RQs:

% \textbf{RQ1}. \emph{What model part is the best in recognizing anomalies of training data according to the observed signals?}

% \textbf{RQ1}. \emph{Which part of the model is most effective at identifying detrimental training samples based on the observed influence?}

% \textbf{RQ1}. 
% \emph{Does gradient cancellation among training samples meaningfully correlate with the quality of layers' influence scores in distinguishing detrimental samples? } 
% Assessing whether the cancellation effect heuristic reliably identifies layers whose influence scores improve downstream performance after filtering detrimental samples is crucial for establishing its validity as an indicator measure. This question requires evaluating scores of different parameter groups, which leads to the following.

\textbf{[RQ1]}.
\looseness-1\emph{Can the gradient cancellation effect serve as a reliable predictor of layer contribution in influence estimation?} In Section~\ref{sec:cancel-exp}, we propose both theoretical and empirical evidence to answer RQ1, showing that high cancellation within a layer does not necessarily imply poor attribution performance.\vspace{1mm}

\textbf{[RQ2]}. 
% \emph{Which part of the model is most effective at identifying detrimental training samples based on the observed influence?}
% \emph{Effectiveness of model layers in detecting detrimental training samples.}
\emph{Which model layers yield the most effective influence scores for detecting detrimental samples?}
% Influence values (section~\ref{sec:infl}) have the potential to recognize anomalies and could be measured across model layers. Answering RQ1 (section~\ref{sec:best-layers}) serves practical purposes of 1) suggesting better ways of training data filtering 2) understanding which model part is most susceptible to anomalies. At the same time, there are many ways how the signals could be combined together (equation~\ref{eq:infl-agg}). It is not possible to test their effect on the model performance exhaustively. This leads to the following question.
% Influence values 
% % (Section~\ref{sec:infl}) 
% have the potential to identify harmful samples and can be computed across different model layers. 
Addressing this question (Section~\ref{sec:layers-exp}) serves two practical goals: (1) improving training data filtering methods, and (2) revealing which parts of the model are most sensitive to such samples. These efforts are motivated by the lack of general consensus on layer choice in influence estimation.\vspace{1mm} %However, there are multiple ways to aggregate these influence signals, thus motivating our efforts.

\textbf{[RQ3]}.
\emph{Can alternative aggregation strategies improve influence estimation performance compared to traditional averaging?}
% Section~\ref{sec:appr} proposes alternatives for how separate scores could be aggregated, the methods \texttt{Rank} and \texttt{Vote}. We study their effect on model accuracy after anomaly filtering in section~\ref{sec:rank-vote}. At the same time, we recognize that some unsupervised learning methods that do not reduce influence dimensionality with aggregation could potentially have an even better effect, but leave this improvement for the future. Evaluating their impact on model performance exhaustively is infeasible. This raises the following question.
Moving beyond influence score averaging across layers, we introduce two novel aggregation methods (Section~\ref{sec:appr}): \texttt{Rank} and \texttt{Vote}, for combining individual influence scores. In Section~\ref{sec:aggr-exp}, we evaluate their impact on model accuracy following sample filtering.\vspace{1mm} 
% While some unsupervised methods that preserve the full dimensionality of influence values may offer further improvements, we leave their exploration for future work. 
% As an exhaustive evaluation of all possible approaches is infeasible, this leads to the following question.
%The evaluation of a new approach requires model retraining of filtered datasets.
%Therefore, we turn to the question of whether one can estimate downstream performance without additional fine-tuning.

\textbf{[RQ4]}. 
% \emph{Is there a measure that could estimate the downstream task performance of the influence function $I'$ and the aggregation method $\mathcal{A}$ without the necessity of additional fine-tuning?}
% \emph{Identifying measures to estimate the downstream performance of $I'$, $\mathcal{A}$ without retraining.}
% \emph{How reliably can noise-distribution measures predict the downstream performance of influence-based data filtering methods?} In Section~\ref{sec:ndr-exp}, we study whether our proposed metrics: Noise Detection Rate (NDR) and Area Under the Curve (AUC), can serve as useful metrics for assessing the performance of influence functions, without requiring expensive fine-tuning.
\emph{How reliably can the detrimental sample distribution measures predict the downstream performance of influence-based data filtering methods?} In Section~\ref{sec:ndr-exp}, we study whether our novel proposed metrics: Noise Detection Rate (NDR) and Area Under the Curve (AUC), can serve as useful metrics for assessing the performance of influence functions, without requiring expensive fine-tuning.\vspace{1mm}

\section{Analysis and Experimental Results}
\label{sec:exp}

% TODO: We consider the problem of detecting the mislabeled or noisy training samples with influence estimation. We constrain our experiments to the context of binary classification, where the noise introduction and fixture are done by flipping the labels. Previous works define the influence measures as an aggregation of scores obtained on separate model modules or layers. We target answering the following research questions.

\textbf{Datasets}. We evaluate detrimental sample filtering across eight GLUE datasets~\citep{Wang-2018}: QNLI, MRPC, QQP, SST-2, CoLA, MNLI, RTE, and STS-B. These benchmarks cover a range of natural language understanding tasks, including sentence similarity, paraphrase detection, sentiment analysis, linguistic acceptability, and natural language inference. Together, they provide a diverse and widely used testbed for assessing model performance and the impact of data quality.

\textbf{Models}. The experiments are conducted on several LLM types: RoBERTa-Large~\citep{Liu-2019-R}, LLaMA-3.2 1B~\citep{Touvron-2023}, Qwen-2.5 1.5B~\citep{Qwen-25}, and Mistral 7B~\citep{Albert-2023}. These models represent various widely used transformer-based language models ranging from millions to several billions of parameters, offering a comprehensive view of influence behavior across model sizes and architectures.

\textbf{Methodology and Protocol}. Similar to prior work \citep{Kwon-2024, Chhabra-2025}, we employ the five-stage pipeline detailed further in Appendix~\ref{sec:a-exp}): (1) Inject synthetic noise into training data and initialize from compressed embeddings. (2) Fine-tune on noisy data, selecting checkpoints with the lowest validation loss~\citep{Garima-2020}, which we verify outperforms final-epoch or accuracy-based selection. (3) Compute influence values across all tunable layers. (4) Partition the model into embeddings (\texttt{WE}), attention groups (four splits), and classification head (\texttt{CL}), aggregating influence per training sample. (5) Remove 30\% least influential samples and retrain, evaluating effectiveness by the best test accuracy. Each configuration is repeated over 10 seeds for robustness.

Furthermore, to robustly compare configurations, we compare them pairwise across multiple runs. Configuration \texttt{A} outperforms \texttt{B} if it performs better for the same dataset, checkpoint, and random seed. We also report the counts of configuration \texttt{A} being better more times than some predefined ranking threshold (e.g. 75\%) compared to other configurations.
This results in a score matrix showing which configurations consistently outperform others. From this matrix, we identify Pareto fronts and assess which influence methods perform better. 
% : a configuration at rank ii is only outperformed by configurations with lower ranks. 
%This approach helps distinguish strong configurations and identify trade-offs between them. 
% The “confident win” definition could be further refined using statistical significance tests.

% \subsection{RQ1. Best layers for the noise detection}
% \label{sec:best-layers}

\subsection{\textbf{[RQ1]} Verification of Assumptions about Cancellation Effect}
\label{sec:cancel-exp}\vspace{-2mm}

\begin{wraptable}[16]{r}{0.38\textwidth}
\caption{Training gradient cancellations for various LLMs (1 $\rightarrow$ no cancellation, $\infty \rightarrow$ extreme cancellation). }\vspace{-3mm}
\label{tab:r-cancellation}
\centering
\resizebox{\linewidth}{!}{ 
\begin{tabular}{ccccccc}
\hline
 & Layer   & Mean {\scriptsize $\pm$ Std}               &   Median &   Min & Max & $\rho$ \\
 % &     $ln|\Theta|$ \\
\hline
% \multicolumn{5}{c}{Roberta-Large} \\
% \hline
 \multirow{6}{*}{\parbox{0.1cm}{\centering \rotatebox[origin=c]{90}{Roberta-Large}}} &
 WE      & 2.2 {\scriptsize $\pm$ 0.3}  &      1   &   1   & $\infty$ 
 & -0.3 \\
 % &     7 \\
 % 1.4e+07 \\
 &
 00-05   & 9.4 {\scriptsize $\pm$ 3.9}  &     11.8 &   1   & $10^6$ 
 & 0.1 \\
 % 1.4e+06  \\
 % & 5 \\
 % 9.8e+04       \\
 &
 06-11   & 10.5 {\scriptsize $\pm$ 4.6} &     14.3 &   1.7 & $10^6$ 
 & 0.1 \\ 
 % 1.4e+06  \\
 % & 5 \\
 % 9.8e+04       \\
 &
 12-17   & 9.4 {\scriptsize $\pm$ 5.1}  &     12.5 &   1.6 & $\infty$ 
 & 0.1 \\
 % & 5 \\ 
 % 9.8e+04       \\
 &
 18-23   & 8.5 {\scriptsize $\pm$ 4.4}  &     11.1 &   1.5 & $10^6$ 
 & 0.2 \\
 % 2.9e+06  \\
 % & 5 \\ 
 % 9.8e+04       \\
 &
 CL      & 8.5 {\scriptsize $\pm$ 6.1}  &     11.1 &   1.6 & $\infty$ 
 & 0.1 \\
 % & 6 \\ 
 % 1.1e+06 \\
\hline
% \hline
% \multicolumn{6}{|c|}{Llama-3.2 1B} \\
% \hline
\multirow{6}{*}{\parbox{0.1cm}{\centering \rotatebox[origin=c]{90}{Llama-3.2 1B}}} &
 WE      & 2.9 {\scriptsize $\pm$ 0.3} &      1   &   1   & $\infty$ 
 &  0.3 \\
 % &      2.9e+07 \\
 &
 00-03   & 8.4 {\scriptsize $\pm$ 2.9} &     11.8 &   1.7 & $\infty$ 
 & -0.1 \\
 % & 1.1e+05       \\
 &
 04-07   & 5.8 {\scriptsize $\pm$ 2.3} &      7.7 &   1.2 & 
 $10^6$ 
 & -0.2 \\
 % 1.4e+06  \\ 
 % & 1.1e+05       \\
 &
 08-11   & 4.4 {\scriptsize $\pm$ 1.6} &      5.8 &   1   & 
 $10^5$ 
 & -0.1 \\
 % 4.7e+05  \\ 
 % & 1.1e+05       \\
 &
 12-15   & 4.0 {\scriptsize $\pm$ 1.7} &      5.3 &   1   & 
 $10^6$ 
 & -0.1 \\
 % 6.3e+05  \\ 
 % & 1.1e+05       \\
 &
 CL      & 3.1 {\scriptsize $\pm$ 2.3} &      2.5 &   1   & 
 $10^4$ 
 & -0.1 \\
 % 5.8e+03   \\
 % &   4100       \\
\hline
% \multicolumn{5}{c}{Mistral 7B} \\
% \hline
\multirow{6}{*}{\parbox{0.1cm}{\centering \rotatebox[origin=c]{90}{Mistral 7B}}} &
 WE      & 3.5 {\scriptsize $\pm$ 0.3}   &      1.1 &   1   & $\infty$ 
 & 0.1 \\
 % & 7 \\
 % 4.3e+07 \\
 &
 00-07   & 17.7 {\scriptsize $\pm$ 3.5}  &     17   &   1.6 & $\infty$ 
 & 0.0 \\
 % & 5 \\
 % 4.3e+05       \\
 &
 08-15   & 16.4 {\scriptsize $\pm$ 6.4}  &     18.6 &   2.2 & $\infty$ 
 & 0.1 \\
 % & 5 \\ 
 % 4.3e+05       \\
 &
 16-23   & 15.6 {\scriptsize $\pm$ 8.6}  &     16   &   1.9 & $\infty$ 
 & 0.0 \\
 % & 5 \\ 
 % 4.3e+05       \\
 &
 24-31   & 15.7 {\scriptsize $\pm$ 10.2} &     15.5 &   1.8 & $\infty$ 
 & 0.0 \\
 % & 5 \\ 
 % 4.3e+05       \\
 &
 CL      & 20.5 {\scriptsize $\pm$ 19.5} &     11.7 &   3.8 & $10^5$ 
 & 0.1 \\ 
 % 9.1e+04  \\
 % &   5 \\ 
 % 8200       \\
\hline

\end{tabular}}
\end{wraptable}
% \vspace{-2em}

% How adequately does this equation describe the gradient compensation happening during training?
% Indeed, 
% When cancellation happens, $\Delta W$ is small for commonly used optimizers, and $C(W)$ is high. However, metric $C$ is not linear w.r.t. $W$ subsets. 
% Let's consider the embedding \texttt{WE} and the classification head \texttt{CL}. 
% % This corresponds to the embedding layer and the classification head without dense projection. 
% % For instance, Roberta-Large 
% While it is true that many gradients flow through \texttt{CL} during backpropagation, there exist weights in \texttt{WE} (weights of frequent tokens), through which this flow is also high. Moreover, as with bias weight, these gradients tend to cancel out.
% % , resulting in a small $\Delta \theta$ for a token. However, 
% Because the equation~\ref{eq:cancel} takes the norm of all changes, this high cancellation of one weight is not noticeable. The additivity property does not hold for this metric. 

The \textit{cancellation effect} arises when opposing gradient contributions cause small net parameter updates ($\Delta W$) (Eq.~\ref{eq:cancel}), while the cancellation metric $C(W)$ remains high. Intuitively, the key issue arising in Eq.~\ref{eq:cancel} is that $C$ is not additive across parameter subsets. For example, although many gradients pass through the classification head (\texttt{CL}), frequent-token embeddings in \texttt{WE} also accumulate large but opposing updates. Since $C$ aggregates changes via norms, such cancellations are obscured, limiting its reliability as an indicator.
% \textbf{Empirical study}.
% We collect the cancellation effect metric~\ref{eq:cancel} with L1 norm for GLUE datasets and 3 considered models. After the first fine-tuning on a noisy dataset, we preserve a checkpoint with the lowest loss on the validation set, which is also used for influence computation. Then, the cancellation metric routine passes the training set through this checkpoint in evaluation mode to sum up gradients and their absolute values. In this way, it estimates cancellation for one epoch, ignoring dropout, activation normalization, and the stochasticity of sample batching. This approach differs from the original work, where cancellations are collected during training from the start to the checkpoint of interest. We preserve computed cancellation values for each tunable module and word embeddings. 

Table~\ref{tab:r-cancellation} presents the cancellation effect measured across LLMs, datasets, and seeds. The embeddings \texttt{WE} indeed have the smallest values. 
% The median values of $C$ correspond to weights inside corresponding layer groups 
The median value of $C$ reflects the \emph{cancellation of an individual parameter}, predominantly observed within a corresponding layer group.
% \td{Tried to rewrite this sentence but not sure what you are trying to say, so please rewrite plainly and correctly so that reviewers can better understand}. 
% Note that 
Value $C=1$ denotes the smallest cancellation when all the gradients are codirectional (e.g. majority of \texttt{WE} parameters). 
However, there are \texttt{WE} weights with infinite cancellation scores that have 
two or more collected gradients from training samples.
% a non-zero number of updates from samples 
% \td{Also confusing -- non-zero number of updates from samples should be rewritten to make it better understandable.}. 
% The cancellation happens on the level of separate parameters and is not reflected in the metric value for the whole layer. 
Clearly, the results confirm that layer group cancellation $C$ ``smoothes'' the gradient compensation over a large number of parameters and hides those weights where the individual parameter gradient cancellation is high.
\citet{Yen-2022} showed that influence estimates from layers with lower $C$ better capture expected prediction changes for ground-truth classes when top-$k$ opponents are filtered out for a given sample. In contrast, we extend this analysis to the validation set as a whole, filtering opponents across many samples and measuring downstream performance via task accuracy. 
% \td{You cannot refer to Table 3 right now since we are still in this section, please remove this mention and correct}, 
Results demonstrate only weak or no correlation between $C$ and performance (Table~\ref{tab:r-cancellation}, column $\rho$: Spearman correlation). 
% Notably, middle attention layers exhibit strong gradient compensation (low $C$; see Table~\ref{tab:r-cancellation}) yet achieve the best anomaly filtering ability (Section~\ref{sec:layers-exp}) \td{Again, don't jump ahead. Talk about things as they appear, and remove the mention of Section 5.2 here. In general you need to better describe Table 1 here and don't jump ahead}. 
% This discrepancy demonstrates that $C$ is not a reliable indicator of downstream utility.

% \textbf{Counter-example}. The cancellation effect hypothesis does not hold in the following setting. 

\looseness-1 Next, we provide theoretical evidence that reinforces the aforementioned issues with the cancellation effect and demonstrate how it is an unreliable estimator of influence performance (\textit{see Appendix~\ref{sec:a-c-proof} for the proof}).
% on how the cancellation can positively affect the influence attributions.

% \textbf{Theorem 1}. Given training set $X=\{\bar{x}_1, \bar{x}_2\}$, where $\bar{x}_1$ is noise and $\bar{x}_2$ is clean; two weights $\theta$ and $\omega$, $C(\theta) \ll C(\omega)$, 
% % $W_1=\{\theta\}$, $W_2=\{\theta, \omega\}$, 
% TracIn functions $I_{\theta, \omega}$, $I_{\theta}$, 
% there exists validation sample $\bar{x}_3$, that increase distance between noise and clean samples with $I_{\theta, \omega}$,  $\Delta I_{\theta,\omega} > \Delta I_{\theta}$, improving discrimination power of influence method.

% \textbf{Theorem 1}. Consider 
% (1) an arbitrary training set $X$,
% (2) samples $\bar{x}_1, \bar{x}_2\ \in X$, where $\bar{x}_1$ is noisy and $\bar{x}_2$ is clean,
% (3) two weights $\theta$ and $\omega$ with cancellations $C(\theta) \ll C(\omega)$, $C(\omega) \rightarrow \infty$ for $\bar{x}_1, \bar{x}_2$, and
% (4) TracIn influence functions $I_{\theta,\omega}$ and $I_{\theta}$.
% Then there exists a validation sample $\bar{x}_3$ such that $\Delta I_{\theta,\omega} > \Delta I_{\theta} > 0$
% which increases the separation between noise and clean samples and thereby improves the discrimination power of the influence method.

% \begin{theorem}[\textbf{Invalidating \citep{Yen-2022}'s Findings}]
\begin{theorem}[\textbf{Cancellation Can Improve Influence Estimation.}]
Let $X$ be a training set. Consider:  
\begin{enumerate}[nosep]
    \item two samples $\bar{x}_1, \bar{x}_2 \in X$, where $\bar{x}_1$ is noisy and $\bar{x}_2$ is clean;
    \item two parameter vectors $\theta$ and $\omega$ such that their cancellation scores satisfy 
    $C(\theta) \ll C(\omega)$ with $C(\omega)\to\infty$ for $\{\bar{x}_1,\bar{x}_2\}$;
    \item the TracIn influence scores $I_\theta$ (based on $\theta$ alone) and $I_{\theta,\omega}$ (based jointly on $\theta,\omega$);
    \item influence score distance between $\bar{x}_1$ and $\bar{x}_2$ w.r.t. weights $\Theta$ and validation samples $X'$: $\Delta_\Theta I(\bar{x}_1,\bar{x}_2 \mid X') = |I(\bar{x}_1, X', \Theta) - I(\bar{x}_2, X', \Theta)|$.
\end{enumerate}
Then there exists a validation point $\bar{x}_3$ such that:\vspace{-1mm} 
% \td{Need to define $\Delta I$ somewhere before here}
\[
\Delta I_{\theta,\omega}(\bar{x}_1,\bar{x}_2\mid \bar{x}_3)
   > \Delta I_\theta(\bar{x}_1,\bar{x}_2\mid \bar{x}_3),
\]
\looseness-1 i.e.\ the separation between noisy and clean samples is strictly larger under $I_{\theta,\omega}$, showing that contrary to the claim of \cite{Yen-2022}, the inclusion of high cancellation weights can \textit{improve} influence.
\end{theorem}\vspace{-2mm}

%\textbf{Proof}. See Appendix~\ref{sec:a-c-proof}. \qed

\textbf{Remarks on RQ1 Findings.} Our empirical results demonstrate an absent-to-weak correlation between the cancellation effect and the accuracy of models retrained on datasets filtered by corresponding layers. Further, our theoretical analysis provides a counterexample, demonstrating that weights with high cancellation \textit{can}, in fact, improve influence estimation.

\subsection{\textbf{[RQ2]} Layer effectiveness in Influence Function Analysis}
\label{sec:layers-exp}\vspace{-2mm}

Next, owing to the limitations of the cancellation effect as previously discussed, we now seek to analyze which layers are better suited for influence estimation in language models.
% In Table~\ref{tab:first-vs-last}, 
We compare layers based on how well they can detect detrimental samples in accordance with the framework outlined at the beginning of the section. 
% In the table, the model weight groups are presented as rows ordered starting from $\operatorname{TracIn}^{10}_{we}$, $\operatorname{TracIn}_{we}$ (i.e., selection from \texttt{WE}), followed by full \texttt{WE} layer, attention layer groups, and the \texttt{CL} layer. 
% Attention groups are shown in the same order as they appear in the model, but their size varies with the model (Appendix~\ref{sec:a-exp}). 
% % of noise of different layer groups for three models and three influence computation approaches. 
% % The ranking threshold is 0.5. 
% Columns of the table contain the layer's performance ranking and win ratio for each model and influence function. 
% % Configurations are compared on 80 sampled best test accuracies (8 datasets and 10 seeds). 
In general, our results are only aligned with \citep{Yen-2022} regarding the classification head \texttt{CL} performance. We observe that \texttt{CL} is highly susceptible to influential noise across models and methods. Dataset filtering with scores from \texttt{CL} improves downstream accuracy less than other layers (Figure~\ref{fig:m-acc}). 

At the same time, \texttt{WE} is not the best choice of layer for every case scenario. While \texttt{TracIn} applied on the smaller model RoBERTa-Large demonstrates best accuracies when \texttt{WE} is used, in general, the first or second attention layer groups are better options-- both in terms of performance and computational complexity, when larger models are used. Methods $\operatorname{TracIn}^{10}_{we}$ and $\operatorname{TracIn}_{we}$ improves upon \texttt{TracIn} with \texttt{WE}, but requires more clock time in practice (Section~\ref{sec:infl}). For other influence functions, we also observe the best performance on the first and second attention layer groups. Appendix~\ref{sec:a-layer-ranking} contains details on the layer ranking per model.

Figure~\ref{fig:m-acc} demonstrates the best test set accuracy variation (mean and Q1-Q3 quartile range) after filtering with influence scores for Mistral 7B. Configurations have statistically significant differences. The best layers outperform the worst ones (generally, \texttt{CL} for most tasks and methods) by 10-15\%. Similar observations for other models are in provided in Appendix~\ref{sec:a-acc-distr} due to space constraints.
% \vspace{-4.5mm}
% The results presented in this section are obtained by applying influence methods with standard mean aggregation for scoring. 
\begin{figure}[h]
    \centering
    \includegraphics[width=\textwidth]{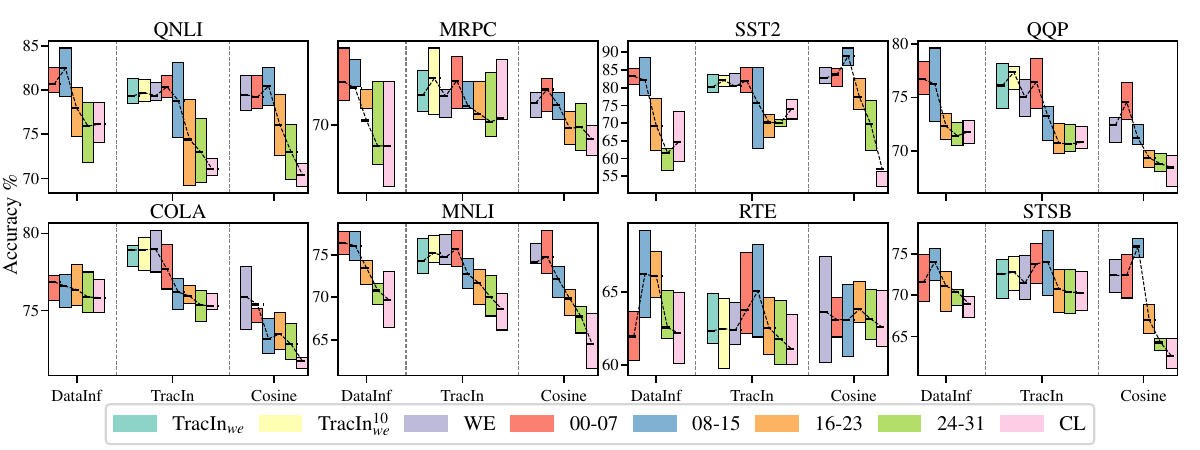}\vspace{-3mm}
    \caption{Best test accuracy of Mistral 7B after 30\% filtering (averaged over 10 runs). Early and middle attention layers yield the strongest influence method performance on most tasks. Differences are statistically significant (\textit{Friedman p-value: 1e-142}).}
    \label{fig:m-acc}
\end{figure}

% While table~\ref{tab:first-vs-last} compares layers for each influence function separately, the following results present a comparison across them.

\textbf{RoBERTa-Large}. 
% This is the smallest model. 
% On this shallow model, we observe a spike in filtering performance (for DataInf, Cosine) on later attention layer groups. DataInf on group 4 (layers 18-23) has the top-1 performance with 0.70 win rate over other configurations. Cosine on group 3, layers 12-17 (0.56 win rate) and TracIn on \texttt{WE} (0.51) are runner-ups. $\operatorname{TracIn}_{we}$ and $\operatorname{TracIn}^{10}_{we}$  are very similar (0.47) but dominated by TracIn on \texttt{WE}. 
% COLA accuracy trend on figure~\ref{fig:r-acc-trend} in App.~\ref{sec:a-acc-distr} demonstrates best that the ideal filtering \texttt{Full} sometimes is outperformed by influence methods that leave some percent of noise. 
For this comparatively smaller model, filtering performance peaks in later attention layers. DataInf achieves the strongest result on group 4 (layers 18–23) with a 0.70 win rate, outperforming other configurations. Cosine on group 3 (layers 12–17, 0.56 win rate) and TracIn on the embedding layer (0.51) are the closest competitors. The two variants $\operatorname{TracIn}{we}$ and $\operatorname{TracIn}^{10}{we}$ yield nearly identical scores (0.47) but are consistently dominated by TracIn on \texttt{WE}. Interestingly, COLA accuracy trends (see Figure~\ref{fig:r-acc-trend}, Appendix~\ref{sec:a-acc-distr}) reveal that the oracle setting \texttt{Full} is occasionally surpassed by influence-based filtering that retains a fraction of noisy samples.

% In some cases, filtering methods outperform the \texttt{Full} baseline (ideal noise removal). For instance,  
% figure~\ref{fig:r-acc-trend}. This is an exceptional case when the left noise is beneficial for fine-tuning.

\looseness-1\textbf{Llama-3.2 1B}. 
% In contrast to other models, influence functions demonstrate the worst noise filtering ability on Llama-3.2. 
% % filtering with influence functions on this model is the worst compared to others. 
% Configurations do not outperform the uniform noise removal, the baseline \texttt{Random} even when discarded samples have a high percentage of mislabeled. DataInf on group 2, layers 04-07 (0.64 win rate), group 1, layers 00-03 (0.58) is ranked best together with \texttt{Random} (0.64). Figure~\ref{fig:l-acc-trend} in App.~\ref{sec:a-acc-distr} shows accuracy trends for the best methods. \texttt{Random} surpasses others on QQP, COLA, and STSB in later epochs.
% At the same time, noise distribution in the influence range (figure~\ref{fig:l-noise-distr}, App.~\ref{sec:a2-noise-distr}) is not uniform. Noise concentrates on both sides of the influence range. A high percentage remains influential after filtering. 
% % This could suggest that, apart from there is some quality of left noise that swings the model learning trajectory to the suboptimal path. In comparison between influence methods, layers 04-07 perform much better than the others 
For Llama, influence functions perform worst among all settings: none of the configurations surpass the uniform removal baseline (\texttt{Random}), even when a large fraction of mislabeled samples is discarded. The strongest outcomes are obtained with DataInf on group 2 (layers 04–07, 0.64 win rate) and group 1 (layers 00–03, 0.58), comparable to \texttt{Random} (0.64). Accuracy trajectories (Figure~\ref{fig:l-acc-trend}, Appendix~\ref{sec:a-acc-distr}) further show that \texttt{Random} outperforms influence methods on QQP, COLA, and STSB in later epochs. Despite this, noise distribution analysis (Figure~\ref{fig:l-noise-distr}, Appendix~\ref{sec:a2-noise-distr}) reveals non-uniform patterns, with mislabeled samples concentrating at both extremes of the influence range, describing a substantial fraction of detrimental samples as influential for the task goal. We provide additional discussion on Llama results in Appendix~\ref{sec:a-llama}.
%Appendix~\ref{sec:a-llama} provides additional discussion and shows that focusing on LoRA B-value projection modules in layers 8–9 with Vote aggregation yields statistically significant improvements over the Random baseline.

\textbf{Qwen-2.5 1.5B}. 
% Cosine on group 2 (layers 07-13) and DataInf on group 1 (layers 00-06) and group 2 demonstrate the best performance and are ranked top-1. TracIn is the runner-up (rank 2) with its best outcomes on the second group. However, method $\operatorname{TracIn}_{we}$ gives comparable results to TracIn on group 2. Figure~\ref{fig:q-acc}, App.~\ref{sec:a-acc-distr} demonstrate that middle layers 07-13 frequently outperform others across GLUE tasks and methods. 
The strongest results are achieved by Cosine on group 2 (layers 07–13) and by DataInf on both groups 1 (layers 00–06) and 2, all ranked top-1. TracIn follows as the runner-up, with its best performance on group 2, while $\operatorname{TracIn}_{we}$ produces comparable outcomes on the same group. Accuracy trends (Figure~\ref{fig:q-acc}, Appendix~\ref{sec:a-acc-distr}) further show that middle layers (07–13) consistently outperform other layers across GLUE tasks and methods.

\looseness-1\textbf{Mistral 7B}. 
% % In contrast, many methods outperform \texttt{Random} on this model. 
% The top-1 ranked configurations include DataInf and TracIn methods.  DataInf demonstrates best outcomes on layers 08-15 (0.71 win rate), 00-07 (0.69), and 16-23 (0.51). TracIn has its best on attention layers 00-07, which is comparable to $\operatorname{TracIn}^{10}_{we}$ that improves TracIn on \texttt{WE}.
% % variation  (0.66), also belongs to the Pareto front. The best layers 00-07 of Cosine are ranked second. 
% Cosine on 00-07 and 08-15 is the runner-up.
% Figure~\ref{fig:m-acc-trend} demonstrates that different influence methods could be better on average for specific tasks. For instance, Cosine 08-15 has higher trends on SST2 and STSB, DataInf 08-15 on QNLI and RTE, TracIn on COLA. 
Top-1 configurations constitute DataInf and TracIn. DataInf achieves the strongest results on layers 08–15 (0.71 win rate), 00–07 (0.69), and 16–23 (0.51). TracIn performs best on layers 00–07, closely matched by $\operatorname{TracIn}^{10}_{we}$, which improves over TracIn on \texttt{WE}. Cosine on layers 00–07 and 08–15 ranks as runner-up. Accuracy trends (Figure~\ref{fig:m-acc-trend}, Appendix~\ref{sec:a-acc-distr}) reveal task-specific preferences: Cosine (08–15) excels on SST2 \& STSB, DataInf (08–15) on QNLI \& RTE, and TracIn on COLA.

% the accuracy trend with training epochs. The accuracy lies between the baseline asymptotes (lower and upper bounds, \texttt{Random} and \texttt{Full}). However, each dataset has its leader in terms of average accuracy. For instance,  

% For deeper models, the best layer position shifts closer to the middle and start of the network. For example, in figure~\ref{fig:m-acc}, the results for Mistral demonstrate the best performance of DataInf 08-15 layers. Cosine method best layers are 00-07, while TracIn still performs the best on \texttt{WE}. 

% \begin{figure}[h!]
%     \centering
%     \includegraphics[width=\textwidth]{plots/mistral/fig-accuracy--sm.pdf}    
%     \caption{Mistral 7B accuracy on test set after 30\% filtering with the best by approach influence methods ignoring aggregation across all layers.  }
%     \label{fig:m-acc-trend}
% \end{figure}

% In summary, our observations indicate the following regarding the first research question. With mean score aggregation, the classification head \texttt{CL} demonstrates poor noise discrimination performance across the considered influence methods and models. The TracIn method has decreasing trends from the first layers to the last. Therefore, TracIn \texttt{WE} variations are better for this method. However, DataInf, Cosine frequently outperform TracIn on early-middle attention layers of deep models. 

\textbf{Remarks on RQ2 Findings.} Across models and methods, we observe a consistent pattern in layer-level performance. With mean-score aggregation, the classification head (\texttt{CL}) shows limited ability to discriminate noise across all influence methods. 
% TracIn exhibits a monotonic decline from early to late layers, making its word embedding (\texttt{WE}) variations more effective overall. 
% In contrast, 
DataInf and Cosine frequently outperform TracIn, particularly on early–middle attention layers of deeper models, suggesting that these layers encode representations that are most informative for noise filtering. 
Comprehensive rankings across all configurations are further provided in Appendix~\ref{sec:a1-test-set-accs} (Tables~\ref{tab:r-test-acc}, \ref{tab:l-test-acc}, \ref{tab:m-test-acc}). 
Additionally, we verify and reconfirm our findings on autoregressive datasets in Appendix~\ref{sec:a-autoreg-ds}.
% \vspace{-3mm}

% The reader should consider App.~\ref{sec:a1-test-set-accs}, Tables~\ref{tab:r-test-acc},\ref{tab:l-test-acc},\ref{tab:m-test-acc} for detailed ranking across configurations, including different aggregation methods.

\subsection{\textbf{[RQ3]} Exploring better strategies to combine influence scores}
\label{sec:aggr-exp}\vspace{-2mm}

% Correlation between training sample scores of different methods and layers (Figure~\ref{fig:m-infl-corr}) discerns three layer groups (at the beginning, in the middle, and at the end of a network). A layer of one group has similar sample scores. For instance, TracIn variations on \texttt{WE} qualitatively recognize the same noise. The last layers also have a strong agreement in every method. The correlation between groups is very weak. 
% % Middle layers happen to differ from both groups. 
% For deeper models, DataInf has similar scores to TracIn on the last layers, while the similarity with Cosine is weak.

% \begin{wrapfigure}{r}{0.55\textwidth}
%     \centering\vspace{-2mm}
%     \includegraphics[width=\linewidth]{plots/mistral/score-corr-all.pdf}\vspace{-3mm}
%     \caption{Influence correlation between layers and methods (Mistral-7B). 
%     % The last layers demonstrate better agreement in scoring, especially on deeper models. First layers also demonstrate weak correlation, and capture qualitatively different noise than the last layers. 
%     }
%     \label{fig:m-infl-corr}
% \end{wrapfigure}

Next, we aim to study whether our proposed aggregation strategies (Section \ref{sec:appr}) can result in better influence function performance. We first analyze the correlation of training-sample influence scores across layers and methods (shown in Appendix~\ref{sec:a-ndr-across-layers} 
% Figure~\ref{fig:m-infl-corr}
Figure~\ref{fig:corr-all}), revealing three layer groups: \textit{early, middle}, and \textit{late}. Within a group, layers assign similar scores: for example, TracIn variations on \texttt{WE} consistently identify the same noisy samples, and late layers show strong agreement across all methods. Correlations between groups are weak, indicating that layers often \textit{compensate} for one another. Across methods, DataInf closely aligns with TracIn on later layers, while Cosine shows low agreement, providing a complementary view of importance attribution. These patterns motivate the underlying ideas behind our proposed aggregation strategies, \texttt{Rank} and \texttt{Vote}, that reconcile layer- and method-level discrepancies.
% From different alternative methods of influence score aggreation, we found \texttt{Rank} and \texttt{Vote} to be most promising. As with layer analysis, for each eggregation we collected accuracy change after retraining on dataset filtered by different strategies.
% % As alternative methods to mean aggregation, we consider \texttt{Rank} and \texttt{Vote} described in section~\ref{sec:infl}. We measure the improvement of noise filtering in terms of the best test set accuracy. 
% Figure~\ref{fig:rank-vote-diffs-mistral} reports the accuracy change from the level of mean aggregation. 
% Beyond mean aggregation, we evaluate alternative strategies. Among them, \texttt{Rank} and \texttt{Vote} emerge as the most promising. 
% % Similar to our layer analysis, we measure the effect of each aggregation by retraining on datasets filtered under different strategies and recording the resulting test accuracy. 
Figure~\ref{fig:rank-vote-diffs-mistral} illustrates the performance differences relative to mean aggregation for Mistral 7B (results for other models show similar trends and are provided in Appendix~\ref{sec:a3-agg-methods}).
% similarly, results for varying number of votes are provided in Appendix~\ref{sec:a-vote-k}).

\begin{figure}[t]
    \centering
    \includegraphics[width=\linewidth]{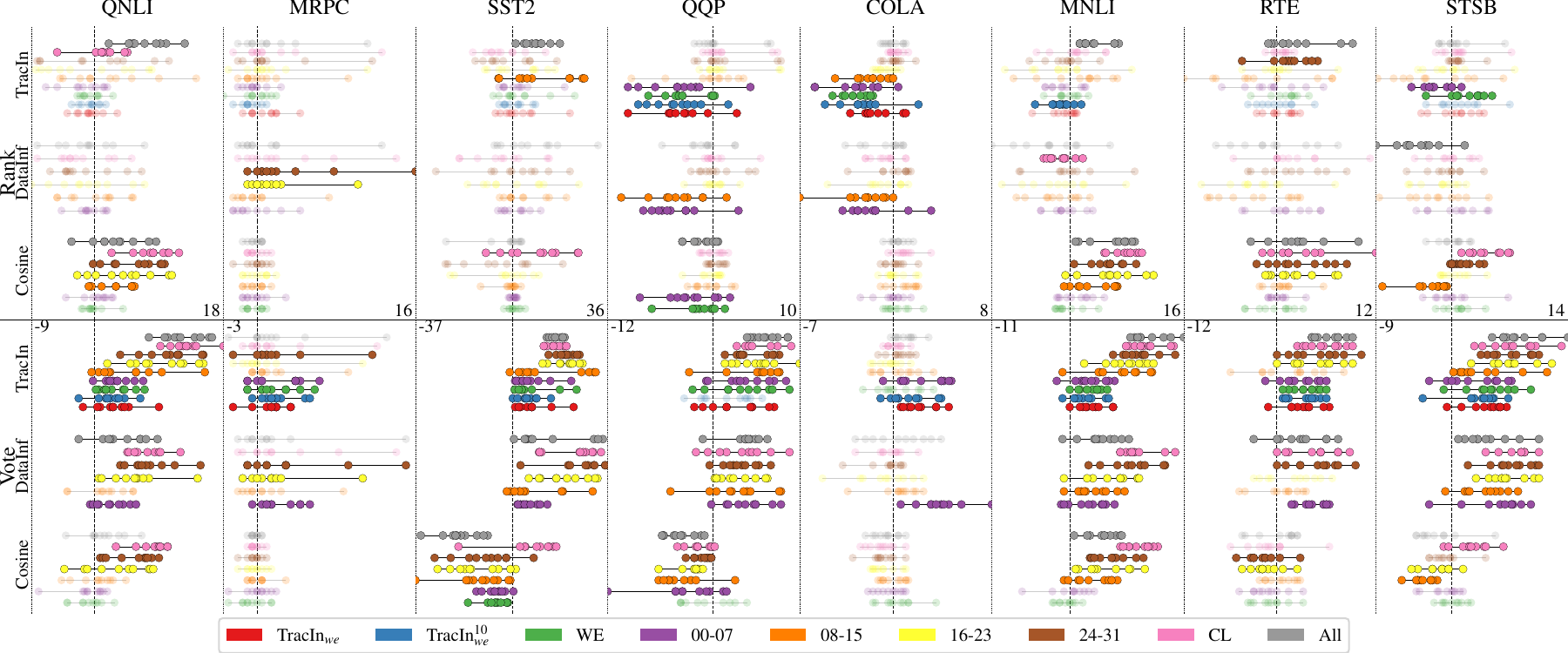}
    % \vspace{-3mm}
    \caption{Mistral 7B: Accuracy improvements (\% change) of \texttt{Rank} and \texttt{Vote} relative to mean aggregation, 10 runs per layer group. Voting consistently improves accuracy for DataInf and TracIn across most datasets and layers, while degrading performance for Cosine. Statistically significant differences are shown in opaque colors (Wilcoxon test, $p < 0.1$ and mean change $>1\%$).
    }
    % \vspace{-5mm}
    \label{fig:rank-vote-diffs-mistral}
\end{figure}

% Voting greatly improves the filtering on the Mistral model for TracIn and DataInf without statistically significant loss of accuracy on any dataset. COLA happens to be the ``toughest'' dataset in this regard. For TracIn, similar trends are observed for Llama and Roberta  (App.~\ref{sec:a3-agg-methods}). However, while the accuracy is increased on many datasets and layers, we observe drops on COLA for these models. \texttt{Rank} method is a bit better than \texttt{Vote} for DataInf method on Llama model, but is inferior on other configurations. 
% Voting substantially improves filtering for TracIn and DataInf on Mistral, with no statistically significant losses across datasets, although COLA remains the most challenging benchmark. Similar gains are observed for TracIn on Llama and RoBERTa (App.~\ref{sec:a3-agg-methods}), again with occasional drops on COLA. For DataInf on Llama, \texttt{Rank} slightly outperforms \texttt{Vote}, but in most other configurations \texttt{Vote} is superior.

Voting substantially improves filtering for TracIn and DataInf on Mistral, with no statistically significant losses across datasets, though CoLA remains the most challenging benchmark. Comparable gains are observed for TracIn on Llama and RoBERTa (Appendix~\ref{sec:a3-agg-methods}), with occasional drops on CoLA, QQP. Moreover, performance improvement is observed for NLI tasks, SST-2, and STS-B. For DataInf on Llama, \texttt{Rank} slightly outperforms \texttt{Vote}, whereas in most other settings \texttt{Vote} remains superior.

Utilizing influence functions with \texttt{Vote} on Mistral leads to significant improvements and 
% brings many outliers to the top 
promotes outlier configurations (layer–influence method combinations) to higher positions in the global ranking
% \td{change this wording "brings many outliers to the top" with something more descriptive and understandable -- rewrite please} 
(Appendix~\ref{sec:a1-test-set-accs}, Table~\ref{tab:m-test-acc}). 
% For instance, TracIn \texttt{CL} and TracIn$^{10}_{we}$ share the best rank and become non-dominant. For Roberta, TracIn \texttt{CL} rank changes \texttt{Mean} $10 \rightarrow 2$ \texttt{Vote}.
% % for Llama, it stays equal 9; 
% For Mistral, TracIn \texttt{CL} layer rank improves from $12$ for \texttt{Mean} aggregation to $1$ for \texttt{Vote}. 
For instance, TracIn \texttt{CL} and TracIn$^{10}_{we}$ now share the top rank and thus become non-dominant. On RoBERTa, the rank of TracIn \texttt{CL} improves from 10 (\texttt{Mean}) to 2 (\texttt{Vote}). On Mistral, the TracIn \texttt{CL} layer rank rises from 12 (\texttt{Mean}) to 1 (\texttt{Vote}).
% \td{This last line is also not clear please rewrite clearly (can use an LLM as before if needed)}
% In practice, the method could improve the scoring on modules that were initially poor but computationally feasible, making these modules better than dense ones. 

\looseness-1 The positional voting scheme introduces the hyperparameter $k$ in Equation~\ref{eq:vote-agg}, representing the number of votes each validation sample assigns to training samples. We also study how aggregation performance varies as a function of $k$. Figure~\ref{fig:layers-ndr-vote_k-value-b} shows the percentage of filtered detrimental samples across different values of $k$ and model layers for Qwen-2.5 1.5B with LoRA Value B module, which provides the strongest overall performance. For both DataInf and TracIn, performance is maximized for $k\in[10,50]$. In contrast, the Cosine-based method attains a local minimum around $k \approx 50$ and reaches its maximum at larger $k$. Similar trends are observed for other LoRA modules (see Appendix~\ref{sec:a-vote-k} for more details).

\begin{figure}[t]
    \centering
    \includegraphics[width=\linewidth]{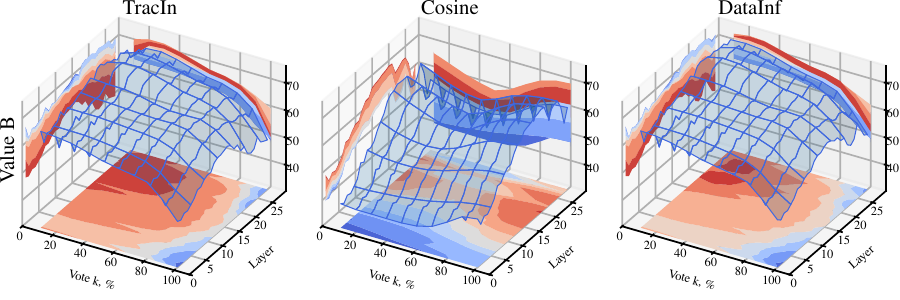}
    \caption{Filtered detrimental samples \% with LoRA Value B module for Qwen-2.5 1.5B as a function of the number of votes $k$ (positional voting, Equation~\ref{eq:vote-agg}) and the selected model attention layer. The heatmaps on the axes show where performance is maximized.
    }
    \label{fig:layers-ndr-vote_k-value-b}\vspace{-2mm}
\end{figure}

\textbf{Remarks on RQ3 Findings.} 
% Observed performance improvements conclude that better aggregation methods, like \texttt{Vote}, exist even across many contexts. For instance, DataInf 00-07 winning rate is increased to 0.84. 
% % However, the method is not generally better (at least across models and tasks for the specific influence function). 
% % The design of a better influence estimation is a persistent problem of a.
% Future studies could potentially target methods $\mathcal{A}$ based on unsupervised learning from influence signals or multi-objective approaches. 
Our results demonstrate that alternative aggregation strategies can significantly improve the effectiveness of influence estimation. In particular, \texttt{Vote} consistently outperforms mean aggregation across multiple datasets and models, with notable gains such as an increase in the win rate of DataInf (layers 00–07) to 0.84. These findings establish that aggregation is a critical factor in influence analysis. While \texttt{Vote} offers a simple and effective alternative, future work could explore efficient unsupervised or multi-objective aggregation methods $\mathcal{A}$ to exploit the influence signals and further enhance filtering performance.

% TODO: DataInf, LiSSA, TracIn - general statement through aggregation formula 

% TODO: find similar works 

% Previous work \citep{} argues that the first layers are better than the last in the noise detection. It argues that the last layers suffer more from the cancellation effect. During backpropagation, the gradients are aggregated in weight tensors through different paths, and the direction of updates could compensate for each other. Authors propose the estimation of such a cancellation effect through the following ratio:

% We examine proxy measures for estimating influence performance without retraining. 
% % Beyond efficiency, such measures can also enhance interpretability by revealing where noisy samples exert the most influence. 
% We first revisit the cancellation effect, followed by noise detection rate and noise distribution analysis.

% \subsubsection{Distribution of Noisy Samples Across Influence Scores}
% \label{sec:ndr-exp}

% \vspace{-3mm}
\subsection{\textbf{[RQ4]} Identifying measures to estimate downstream performance}
\label{sec:ndr-exp}\vspace{-1mm}

% \textbf{Noise detection rate}. 
% In the benchmarking setting, where anomalies are added to the training data in a controllable manner, we can figure out which samples, altered with noise, are filtered out as least influential according to the different methods. \textit{Noise detection rate} NDR measures the percentage of mislabeled samples in the first k samples ordered with influence functions. Thus, NDR at 30\% estimates the quantity of filtered noise before final fine-tuning (Stage 5, section~\ref{sec:exp}). 
%We evaluate how effectively influence functions identify noisy samples in a controlled benchmarking setup with injected anomalies. 
We now investigate whether we can utilize an external metric as proxy to evaluate influence function methods, thereby foregoing expensive retraining/fine-tuning based evaluation. To this end, we propose two metrics. The \textit{Noise Detection Rate} (NDR), quantifies the proportion of mislabeled samples ranked among the top-$k$\% least influential, i.e., \textit{the fraction of noise filtered prior to final fine-tuning.}
% NDR curve measures the percentage of mislabeled samples in the bottom k samples in the order of influence scoring. To compute the curve, we must know beforehand which sample has a flipped label. Therefore, this metric is available only during benchmarking. However, the NDR curve and its values at a specific threshold (30 percent, we filter out this fraction eventually) could have utility in estimating the goodness of new influence methods before attempting them on target tasks. 
% The noise could have a different distribution in the influence range. Ideally, we expect the method to capture anomalies as the least influential signals of the model. However, it is not always the case. NDR curve visually represents the noise position. 
% The AUC, in contrast to NDR at threshold, estimates the noise skew towards the beginning or influence range. 
In contrast, the Area Under the Curve (AUC) metric captures the overall skew of noise across the influence ranking, indicating whether mislabeled samples concentrate at the top of the list or are more uniformly spread.
%We hypothesize that NDR and AUC could serve as reliable proxy metrics. Even though the measures arguably represent only the quantity and not the quality of filtered noise, removing more anomalies is expected to lead to better accuracy. We validate the hypothesis by measuring the metric correlations to the actual performance observed after fine-tuning on filtered datasets. 
We validate the hypothesis that NDR and AUC can serve as reliable proxy metrics for influence estimation by measuring the metric correlations to the actual ground-truth performance observed after fine-tuning on filtered datasets. 
% NDR at threshold 30 estimates how many noisy samples are discarded before fine-tuning. We test its correlation to the downstream task performance. 
% Table~\ref{tab:spearman} presents the Spearman correlation $\rho$ for NDR and AUC for different influence functions, aggregation methods, and models (i.e., across selected layers, GLUE tasks, and random seeds). We observe weak-to-strong statistically significant positive correlations of NDR and AUC. NDR (threshold 30\%) has a higher correlation in most cases which makes it preferable over AUC. A weak positive correlation between 0.5 and 0.7 does not allow us to conclude that the metrics could be a reliable indicator in all settings. However, we observe a strong correlation 0.8-0.9 for \texttt{Vote} aggregation configurations. Indicators incorporating noise quality estimation or highly influential noise count could be more reliable. 
% This is left for future research. 
Table~\ref{tab:spearman} reports the correlation of NDR and AUC with the actual ground-truth downstream performance.
% , computed across influence functions, aggregation methods, models, selected layers, GLUE tasks, and random seeds. 
Overall, we observe moderate-to-strong statistically significant positive correlations. NDR at 30\% generally exhibits higher values than AUC. However, values in the 0.5–0.7 range indicate that these metrics are not consistently reliable across all configurations. \texttt{Vote} aggregation has stronger correlations of 0.8–0.9, suggesting that combining influence scores appropriately can enhance \textit{proxy predictiveness}. 

% Future work could explore indicators that also account for noise quality or the influence of highly impactful anomalies.
\begin{wraptable}[16]{r}{0.5\textwidth}
\caption{Spearman correlation ($\rho$) of training gradient cancellation ($C$), noise detection rate (NDR), and NDR-AUC with final task accuracy after filtering, across influence aggregation methods. High correlations are shown in bold and * denotes non-statistically significant results.}\vspace{-4mm}
\label{tab:spearman}
\centering
\resizebox{\linewidth}{!}{ 
\begin{tabular}{llccc|ccc|ccc}
\hline 
&        \multirow{2}{*}{Infl.func}  & \multicolumn{3}{c|}{Roberta-Large} & \multicolumn{3}{c|}{Llama-3.2 1B}
& \multicolumn{3}{c}{Mistral 7B} \\
% \cline{3-11}
  &      &  $C$  &  NDR  &  AUC &  $C$  &  NDR  &  AUC &  $C$  &  NDR  &  AUC  \\
\hline
 \multirow{5}{*}{\parbox{0.1cm}{\centering \rotatebox[origin=c]{90}{Mean}}} &      DataInf       &  0.2  &     0.7      &      0.5      & 0.0*  &     0.6      &      0.5      &  0.1 &     0.5      &      0.5      \\
   &   TracIn$_{we}$    & -0.3  &     0.6      &      0.4      & 0.2*  &     0.6      &      0.5      & 0.1*  &     0.4      &      0.3      \\
   & TracIn$^{10}_{we}$ & -0.3  &     0.6      &      0.4      & 0.0*  &     0.5      &      0.5      & 0.0*  &     0.5      &      0.3      \\
   &       TracIn       & 0.0*  &     0.4      &      0.3      & 0.1*  &     0.6      &      0.5      & 0.0* &     0.6      &      0.5      \\
   &       Cosine       &  0.2  &     0.7      &      0.6      & -0.0* &     0.5      &      0.5      & -0.1  &     0.6      &      0.5      \\
\hline
 \multirow{5}{*}{\parbox{0.1cm}{\centering \rotatebox[origin=c]{90}{Rank}}} &      DataInf       &  0.3  &     0.7      &      0.7      & -0.1* &     0.6      &      0.6      & 0.0*  &     0.6      &      0.7      \\
   &   TracIn$_{we}$    & -0.3  &     0.6      &      0.5      & -0.1* &     0.2      &     0.2*      & 0.0* &     0.3      &      0.2      \\
   & TracIn$^{10}_{we}$ & -0.3  &     0.5      &      0.5      & -0.2* &     0.1*     &     0.0*      & -0.1* &     0.2*      &     0.1*      \\
   &       TracIn       &  0.2  &     0.5      &      0.5      & -0.1* &     0.5      &      0.5      & -0.1* &     0.4      &      0.5      \\
   &       Cosine       &  0.2  &     0.7      &      0.7      & -0.1* &     0.6      &      0.6      & 0.1*  &     0.6      &      0.6      \\
\hline
 \multirow{5}{*}{\parbox{0.1cm}{\centering \rotatebox[origin=c]{90}{Vote}}}  &      DataInf       &  0.3  &     \textbf{0.8}      &      \textbf{0.8}      & -0.1  &     0.6      &      0.6      & 0.1  &     \textbf{0.9}      &      \textbf{0.8}      \\
   &   TracIn$_{we}$    & -0.4  &     \textbf{0.8}      &      \textbf{0.8}      & -0.1* &     0.5      &      0.5      & -0.1* &     \textbf{0.8}      &      \textbf{0.8}      \\
   & TracIn$^{10}_{we}$ & -0.4  &     0.7      &      \textbf{0.8}      & -0.1* &     0.5      &      0.5      & 0.0* &     \textbf{0.8}      &      \textbf{0.8}      \\
   &       TracIn       &  0.2  &     \textbf{0.8}      &      0.7      & -0.1  &     0.6      &      0.6      & 0.1  &     \textbf{0.8}      &      \textbf{0.8}      \\
   &       Cosine       &  0.2  &     0.7      &      0.6      & -0.1  &     0.2      &      0.2      & 0.0*  &     0.4      &      0.4      \\
\hline
\end{tabular}}
\end{wraptable}
Figure~\ref{fig:ndr-with-layer} further illustrates the proportion of noise filtered by each layer of Mistral 7B. Across models and influence functions, the Value B LoRA modules consistently capture the largest fraction of noisy samples, with peak filtering observed in the middle attention layers for larger models and later layers for smaller models (see Appendix~\ref{sec:a-ndr-across-layers} for additional results).
% While NDR accounts for filtering only in the first samples of the order, AUC also reflects the presence of influence noise. We observe a weak positive correlation of 0.5-0.6 across influence methods with Mean aggregation. Values for NDR are slightly better than for AUC, i.e. NDR curve shape is less important than captured noise at start. However, considering the nature of correlated metrics and weak correlation, we conclude that the high quantity of filtered noise does not always lead to good performance. There are several reasons for this. First, it is known that noise could potentially improve the learning trends. Second, the learning could be swayed by several influential noisy samples, not detected in the majority filtering. 
Additional \emph{sample-influence ranking} variation analysis in Appendix~\ref{sec:a-rank-var} indicates that average noise influence attains its minimum at layers with the highest noise-detection efficiency.\vspace{2mm}

% DataInf, TracIn methods with Vote aggregation demonstrate a stronger correlation of 0.8-0.9 between NDR and performance on Roberta-Large and Mistral models. As a result, this allows us to consider NDR measures of different layers and modules as a preliminary indicator of goodness. 
% However, we cannot generally conclude that an arbitrary influence method with high NDR would have high downstream performance due to weak correlation. 

% Even though NDR correlation is weak for some settings, we still consider the discarded noise quantity to indicate the desired configuration, as there are infinitely many variations of $(I, \mathcal{A})$ and selected layers. Methods \texttt{Rank} and \texttt{Vote} (section~\ref{sec:rank-vote}) were selected because of high NDR. 

% Figure~\ref{fig:ndr-with-layer} demonstrates how much noise is filtered by every layer of Mistral 7B. We observe that Value B LoRA modules are the best in capturing noise across all models and influence functions. App.~\ref{sec:a-ndr-across-layers} provides more NDR trends. Value B achieves the highest noise filtering percentage on the attention layers close to the middle of a model. 
\begin{figure}[t]
    \centering
    \includegraphics[width=1.0\textwidth]{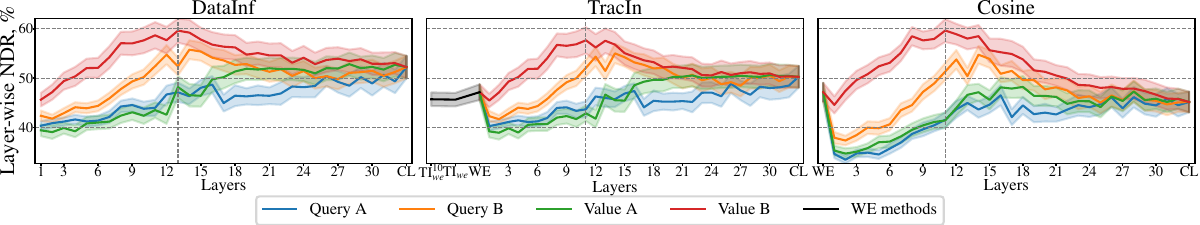}\vspace{-3mm}
    \caption{Layer-wise NDR (\%) across layers of Mistral 7B. }\vspace{-4mm}
    \label{fig:ndr-with-layer}
\end{figure}

% \td{dvitel: I plan to add HERE table comparing query vs value, LoRA A vs LoRA B NDRs, targeting to answer the question which modules better isolate noise as least influential}

\looseness-1\textbf{Remarks on RQ4 Findings.} 
% To summarize the answer to RQ2, NDR could indicate method performance in limited contexts. In many cases, it has only a weak positive correlation. Measures that incorporate noise quality could potentially be better indicators.
The cancellation effect shows little correlation with downstream performance and fails to reliably indicate influential layers or weights. In contrast, NDR serves to be a useful proxy, particularly under \texttt{Vote} aggregation, though correlations are generally moderate-to-high across layers, tasks, and models. The proposed AUC metric exhibits similar trends. Together, these results suggest that proxy metrics incorporating both the quantity and quality of filtered anomalies could offer more reliable indicators for influence function evaluation, highlighting a promising direction for future research.

\section{Conclusion}
 
% This work targets the search for the best model layers to identify anomalies based on influence functions. The results align with previous works; the classification head has suboptimal performance. However, embeddings are not the best choice either. LoRA Value B modules of middle attention layers capture the highest anomaly percentage in the least influential range. The same layers demonstrate the best downstream task performance in many contexts. 

% We validate the proposed indicators, the noise detection rate, and cancellation estimation, in terms of how well they align with observed downstream task accuracy. While cancellation does not indicate a good layer selection, the higher amount of filtered noise has a weak-to-strong positive correlation with performance in different contexts. 
% Finding better indicators is important for future work, as it would allow the design of better influence estimators without the need for exhaustive fine-tuning. 

% To improve estimation on the classification head, we propose and investigate alternative methods of individual influence score aggregations. Voting gives an accuracy boost across many settings for deeper models. Proposed approaches also demonstrate that it is possible to alleviate the compensation of influence values.

Our paper conducts multi-faceted analysis to better understand how influence functions can be effectively leveraged for data-centric learning in large language models. Firstly, our work challenges prior assumptions in past work about the most informative layers for influence estimation in LLMs and demonstrates that middle attention layers often provide more reliable signals than the \textit{first} embedding layers. Secondly, we introduce improved layer-wise influence score aggregation strategies that move beyond the standard averaging approach, and attain significant performance improvements in influence analysis. Finally, we propose external evaluation metrics, such as the Noise Detection Rate (NDR), to offer a more robust framework for evaluating and interpreting influence scores without undertaking costly retraining for influence function evaluation. These findings advance the understanding of training data influence in large language models and provide practical tools for more accurate and efficient model auditing.

% Our paper investigates how influence functions can be better leveraged data-centric learning in large language models. Owing to limited past work in identifying which layers are better suited for influence function analysis, we conduct analysis across several LLMs and datasets. We first contrast with prior work by \cite{Yen-2022} and We show that (contrary to common practice in selecting layers for influence function analysis) neither embeddings nor the classification head is an optimal choice. Moreover, LoRA Value B modules in middle attention layers consistently capture the highest proportion of anomalies and yield superior downstream influence estimation performance. Among aggregation strategies, ranking and voting alleviate influence score compensation and significantly improve accuracy over averaging, especially for \texttt{CL}. We validate that the gradient cancellation effect~\citep{Yen-2022} does serve as a reliable predictor of retraining performance when a group of least influential samples is removed, whereas noise-detection measures exhibit moderate-to-strong correlations with the accuracy, making them useful for preliminary evaluations of future attribution methods. 
% % Future research could focus on refining influence aggregation such proxy measures to reduce the reliance on costly retraining when developing new influence estimators.

%\clearpage

\section{Reproducibility Statement}
We have open-sourced our code and implementation for our experiments: \href{https://github.com/dvitel/nn-infl}{https://github.com/dvitel/nn-infl}. To support reproducibility, multiple runs were executed with fixed random seeds specified in the accompanying shell scripts, and determinism was enabled in the underlying \textit{pytorch} framework. All experiments were conducted on an Ubuntu server with 8x NVIDIA A100s (40GB VRAM/GPU). 

%\section{Ethics Statement}

\bibliography{refs}
\bibliographystyle{unsrtnat}

\newpage

\newpage

%%%%%%%%%%%%%%%%%%%%%%%%%%%%%%%%%%%%%%%%%%%%%%%%%%%%%%%%%%%%

\appendix

\setlength{\textfloatsep}{5pt}  % vertical space between floats and text
\setlength{\floatsep}{5pt}      % vertical space between floats
\setlength{\intextsep}{5pt}     % vertical space for in-text floats
\setlength{\belowcaptionskip}{2pt} % space below captions
\setlength{\abovecaptionskip}{2pt} % space above captions

\section*{Appendix}

\section{Experimental setup details}
\label{sec:a-exp}
% TODO: + add this: Configurations are compared on 80 sampled best test accuracies (8 datasets and 10 seeds). 

This section describes the five-stage experimental pipeline from Section~\ref{sec:exp} in detail.

% \textbf{Stage 1}. The anomaly detection problem is defined on GLUE datasets with mislabeled training samples. Preprocessing selects a maximum of 4500 training samples, 500 validation samples for influence computation, and 500 test-set samples for downstream performance reporting from the dataset under consideration. Ten random seeds define different selections. On every binary classification task, we flip 20\% training labels uniformly selected with a random seed. MNLI task is converted to binary by selecting the samples that belong to the 0 (entailment) and 1 (neutral) classes. For STSB, 0-5 scores are converted to 0 (different) and 1 (similar) labels by threshold 3. Then, noise is added similarly to other tasks. 

\textbf{Stage 1.} We define the anomaly detection task on GLUE datasets with injected label noise. For influence computation, we select up to 4500 training samples, 500 validation samples, and 500 test samples for reporting downstream performance, with 10 random seeds controlling different selections. For binary classification tasks, 20\% of training labels are flipped uniformly at random. MNLI is binarized by keeping only entailment (0) and neutral (1) classes, and STSB scores 0–5 are converted to 0/1 using threshold 3; noise is added similarly.

Checkpoints preparation adds LoRA modules~\citep{Dettmers-2023} to query and value projections in every attention layer, compresses embeddings to the tokens participating in the GLUE dataset splits. The embedding compression greatly speeds up influence computations based on gradients. 
% For \texttt{WE}, we report the number of parameters of compressed embeddings in the result tables.

% We consider the following models. 
% RoBERTa-Large contains 355M parameters, 1.5M trainable parameters, 24 attention layers, and a classification head with a dense projection module. Instruction-tuned Llama-3.2 has 1B parameters, 16 attention layers with attached LoRA modules, and a tunable classification head (0.5M trainable parameters in total). This model is multilingual. Therefore, embedding compression significantly reduces the number of parameters for influence estimation. Qwen-2.5 version with 1.5B parameters has 28 attention layers, Mistral 7B -- 32 attention layers. 
% % A total number of tunable LoRA and classification head parameters is around 1.7 M. 
% We leave word embedding layers \texttt{WE} of the models frozen during training, but compute gradients on them when estimating the influence with different methods from section~\ref{sec:infl}. 
% We use the following learning rates: Roberta-Large 3e-4, Llama 1e-4, Qwen 3e-4 (1e-3 for RTE), Mistral 5e-5. All models are tuned with 16-bit float weights. 

RoBERTa-Large has 355M parameters (1.5M trainable), 24 attention layers, and a classification head with dense projection. Llama-3.2 1B is multilingual, with 16 attention layers, attached LoRA modules, and a tunable classification head (0.5M trainable parameters); embedding compression reduces parameters for influence estimation. Qwen-2.5 1.5B has 28 attention layers, and Mistral 7B has 32. Word embeddings (\texttt{WE}) remain frozen during training, but gradients are computed for influence estimation. Learning rates are: Roberta 3e-4, Llama 1e-4, Qwen 3e-4 (1e-3 for RTE), and Mistral 5e-5; all models use 16-bit float weights.

% \textbf{Stage 2}. Fine-tuning on noisy datasets starts with the corresponding initial checkpoint and proceeds for 10 epochs. During this process, we track validation set accuracy and loss metrics to detect the checkpoint useful for influence estimation. Work~\citep{Garima-2020} provides guidelines on which checkpoints are better for TracIn. Considering them and conducting some preliminary runs, we pick the checkpoint with the best validation set loss in all reported experimental runs. 
\textbf{Stage 2}. Models are fine-tuned on noisy datasets for 10 epochs from the initial checkpoint. Validation loss and accuracy are tracked to select the checkpoint for influence estimation. Following~\citep{Garima-2020} and preliminary experiments, we use the checkpoint with the lowest validation loss in the following stages.

\textbf{Stage 3}. We split the neural network layers into groups: embeddings \texttt{WE}, classification head \texttt{CL}, and 4 groups with an equal number of internal attention layers. The groups in the results are numbered from 1 to 4. For instance, group 1 of RoBERTa-Large includes the first to sixth attention layers, also denoted 00-05. Group 2 has 06-11, group 3 -- layers 12-17, and group 4 -- 18-23. Similar grouping is done for other models. 

\textbf{Stage 4}. The scoring routine computes attributed values from the collected influence tensors of layer groups for a particular aggregation method under consideration (default is averaging). This stage also estimates other indicator metrics such as noise detection rate (based on knowledge of the flipped label), noise histograms, and cancellation effect.

% \textbf{Stage 5}. Obtained scores define the priority of training samples. 30 percent of samples with the lowest scores are discarded from the datasets. Filtered datasets are used in fine-tuning with the same hyperparameters of Stage 2, starting from the initial checkpoint prepared on Stage 1. The performance of the influence method filtering is estimated with an accuracy metric on the separate test set, which was not used for training and for influence computation. 
\textbf{Stage 5}. Training samples are ordered by their influence scores, and the lowest 30\% are removed. The filtered datasets are fine-tuned using the same hyperparameters as Stage 2, starting from the Stage 1 checkpoint. Method performance is evaluated via test-set accuracy on samples unseen during training or influence computation.

% \textbf{Computing cancellation}. Cancellation effect (section \ref{sec:cancellation}) is estimated with L1 norm on Stage 2 checkpoint. cancellation metric routine passes the training set through this checkpoint in evaluation mode to sum up gradients and their absolute values. In this way, it estimates cancellation for one epoch, ignoring dropout, activation normalization, and the stochasticity of sample batching. This approach differs from the original work, where cancellations are collected during training from the start to the checkpoint of interest. We preserve computed cancellation values for each tunable module and word embeddings. 

We apply the following techniques to optimize influence computation.

% \textbf{Compression}. Instead of using all embedding weights, the layer is replaced with one that contains only weights of tokens from the task dataset. Creating sub-views of embeddings could be more appropriate for a production setting. However, this technique is useful in benchmarking where many tokens are not utilized (especially in multi-lingual models). 
\textbf{Compression}. To reduce computation, we replace the full embedding layer with weights corresponding only to tokens present in the task dataset. While sub-views may be preferable in production, this approach is effective in benchmarking, especially for multilingual models with many unused tokens.

% \textbf{Batching the influence}.
\textbf{Batching}. Equations~\ref{eq:tracin},\ref{eq:cosine},\ref{eq:datainf} depend on dot-product $\nabla l^T_{\bar{x}'} \nabla l_{\bar{x}}$. 
% First, we estimate the number of gradients required by the methods. 
To maintain the gradients for $n$ training samples, $k$ validation samples, and $m$ parameters, \texttt{TracIn} and \texttt{Cosine} requires $O(nkm)$, \texttt{DataInf} -- $O(nkm + n^2m)$ of memory. In one iteration, we pick $n_1$ training and $k_1$ validation samples s.t. $O(n_1k_1m)$ gradients fit the available GPU memory, and compute $O(n_1k_1)$ influence scores between sample pairs. Staying in $O(n_1k_1m)$ limit requires the iteration through all pair of $\lceil n/n_1 \rceil$ training and $\lceil k/k_1 \rceil$ validation batches.  

Inner loop has to recompute the gradients because of the memory limit. For A100 80GB GPU, 4500 training and 500 validation samples, influence values on Mistral embeddings are computed in 20 batches in our experiments. It is still very expensive for \texttt{DataInf} because it also depends on $\nabla l^T_{\bar{z}} \nabla l_{\bar{x}}$.

\clearpage

\section{Best test set accuracy distributions and trends}
\label{sec:a-acc-distr}

This section presents the best test set accuracy distribution after filtering for RoBERTa (Figure~\ref{fig:r-acc}), Llama (Figure~\ref{fig:l-acc}), and Qwen (Figure~\ref{fig:q-acc}). TracIn indeed has decreasing performance trend from early \texttt{WE} to later \texttt{CL} layers across models. However, DataInf and Cosine frequently outperform TracIn on early or middle attention layers for bigger models, and later attention layers for the smaller Roberta. 

\begin{figure}[h]
    \centering
    \includegraphics[width=\textwidth]{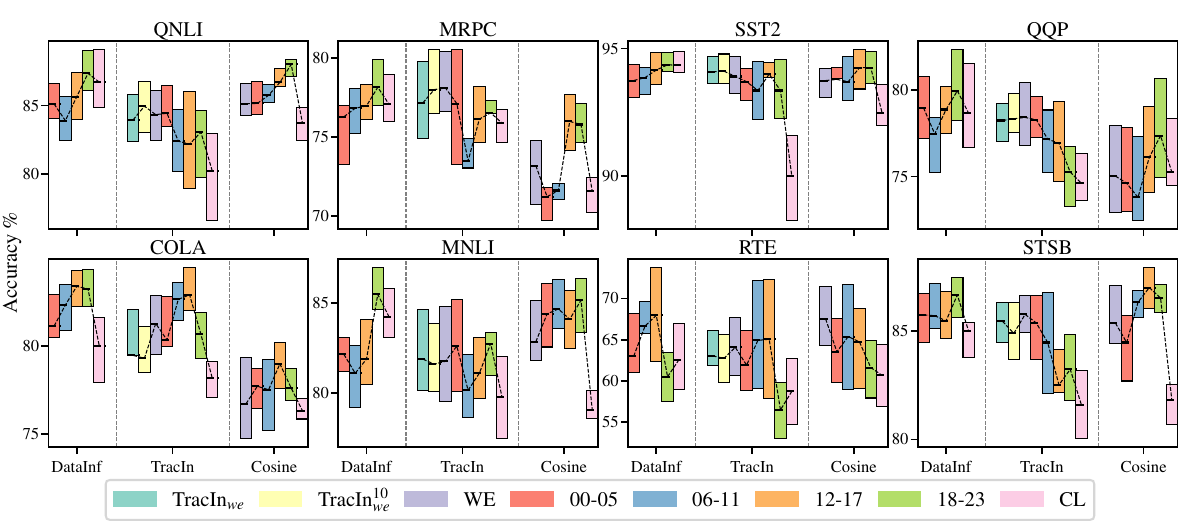}
    \caption{Roberta-Large best test set accuracy after 30\% filtering over 10 runs.  }
    \label{fig:r-acc}
\end{figure}

\begin{figure}[h!]
    \centering
    \includegraphics[width=\textwidth]{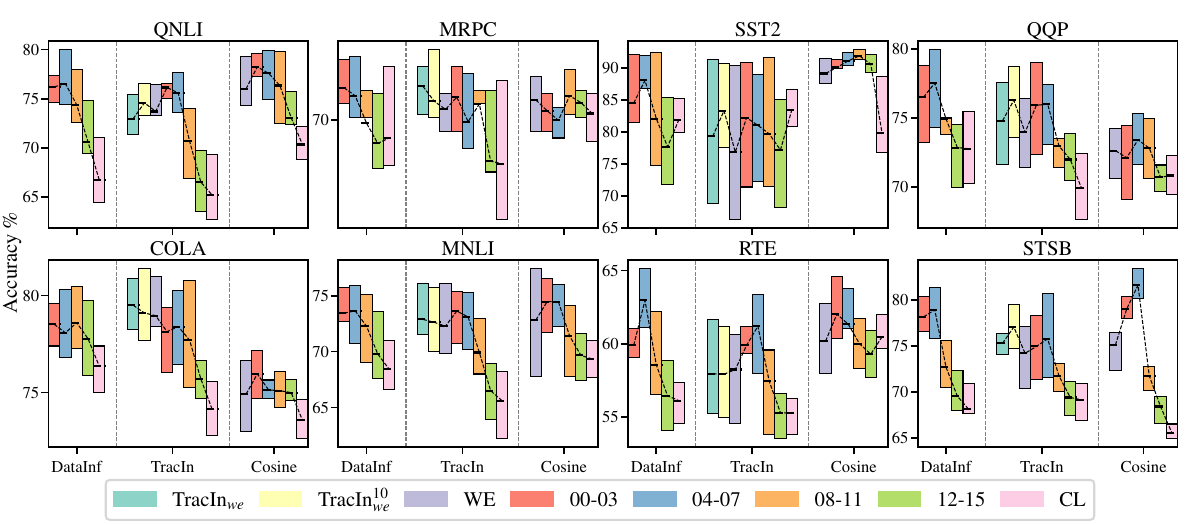}
    \caption{Llama-3.2 1B best test set accuracy after 30\% filtering over 10 runs.  }
    \label{fig:l-acc}
\end{figure}

\begin{figure}[h!]
    \centering
    \includegraphics[width=\textwidth]{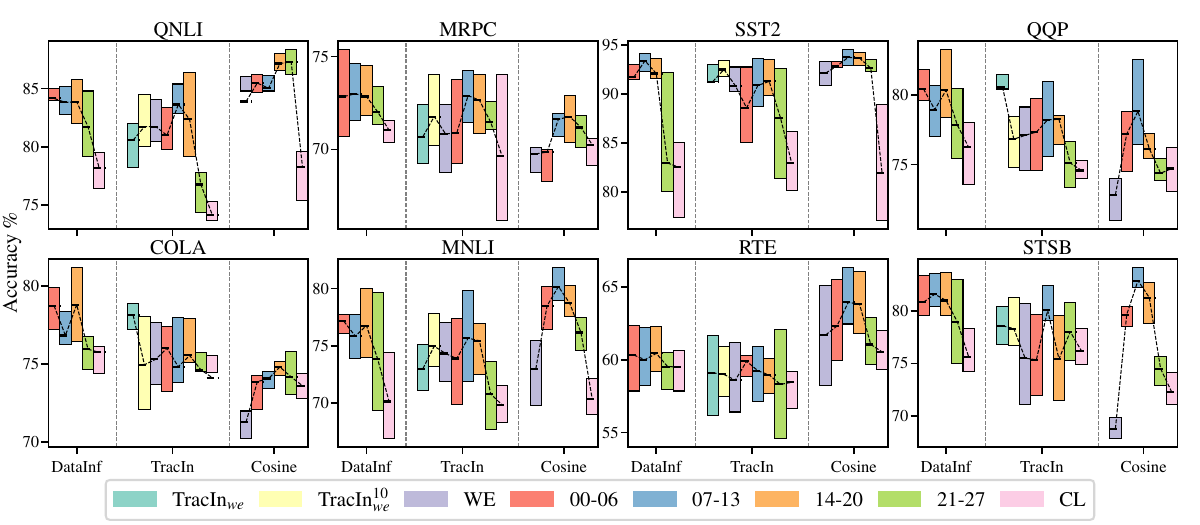}
    \caption{Qwen-2.5 1.5B best test set accuracy after 30\% filtering over 10 runs.  }
    \label{fig:q-acc}
\end{figure}

The following trends are presented for the best configurations, influence method, and layer group for Roberta (Figure~\ref{fig:r-acc-trend}), Llama (Figure~\ref{fig:l-acc-trend}), Qwen (Figure~\ref{fig:q-acc-trend}),  and Mistral (Figure~\ref{fig:l-acc-trend}). Accuracy trends are measured on the test set, not used for training or influence computation. 

\begin{figure}[h]
    \centering
    \includegraphics[width=\textwidth]{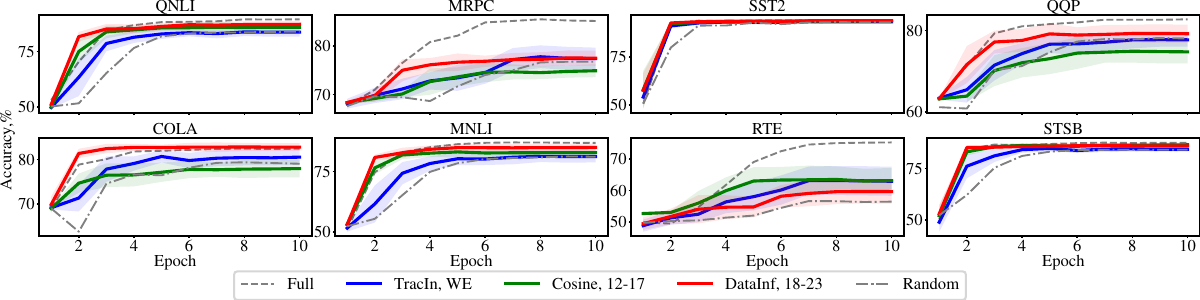}
    \caption{Roberta-Large accuracy trend on test set after 30\% filtering.  }
    \label{fig:r-acc-trend}
\end{figure}

\begin{figure}[h!]
    \centering
    \includegraphics[width=\textwidth]{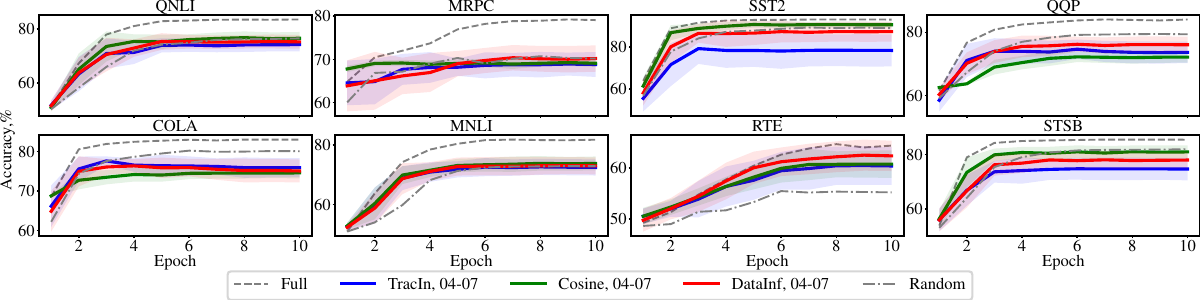}    
    \caption{Llama-3.2 1B accuracy trend on test set after 30\% filtering.  }
    \label{fig:l-acc-trend}
\end{figure}

\begin{figure}[h!]
    \centering
    \includegraphics[width=\textwidth]{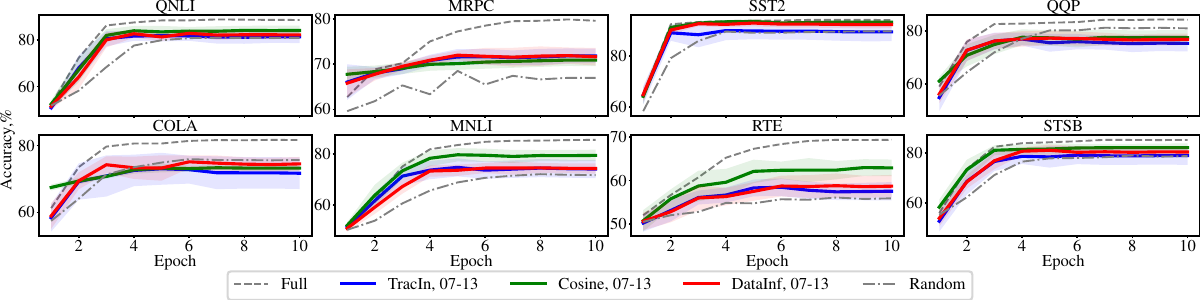}    
    \caption{Qwen-2.5 1.5B accuracy trend on test set after 30\% filtering.  }
    \label{fig:q-acc-trend}
\end{figure}

\begin{figure}[h!]
    \centering
    \includegraphics[width=\textwidth]{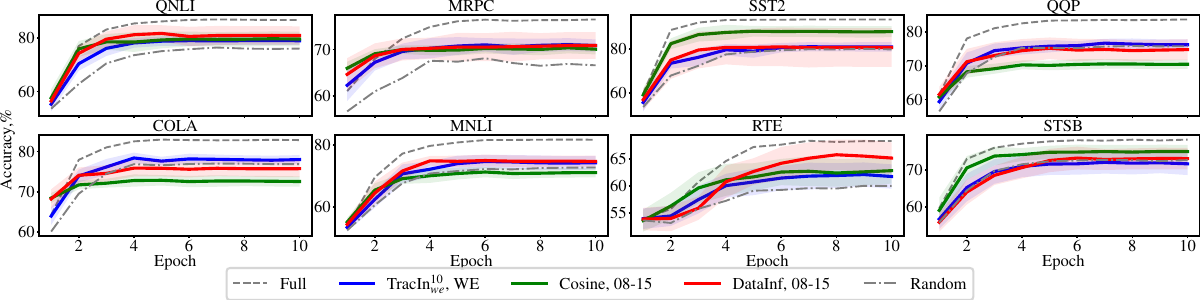}    
    \caption{Mistral 7B accuracy trend on test set after 30\% filtering.  }
    \label{fig:m-acc-trend}
\end{figure}

\clearpage

\section{Layer ranking per model}
\label{sec:a-layer-ranking}

Table~\ref{tab:first-vs-last} presents the results of layer performance ranking based on the measured best test set accuracy after sample filtering, where 30\% of the least influential samples according to each layer are discarded and subsequent retraining is performed. 
In the table, the model weight groups are presented as rows ordered starting from $\operatorname{TracIn}^{10}_{we}$, $\operatorname{TracIn}_{we}$ (i.e., selection from \texttt{WE}), followed by the full \texttt{WE} layer, attention layer groups, and the \texttt{CL} layer. 
Attention groups are shown in the same order as they appear in the model, but their size varies with the model.
% of noise of different layer groups for three models and three influence computation approaches. 
% The ranking threshold is 0.5. 
% Columns of the table contain the layer's performance ranking and win ratio for each model and influence function. 
% Configurations are compared on 80 sampled best test accuracies (8 datasets and 10 seeds). 

\begin{table}[ht]
\caption{Layer ranks (embed \texttt{WE}, attn 1–4, head \texttt{CL}) and win rates (in parentheses) across models and influence functions (threshold=0.5). `–' denotes N/A. More specifically, DataInf is infeasible to run on \texttt{WE} and TracIn specific \texttt{WE} variations $\operatorname{TI}$. Most performant layers are shown in bold. 
% \td{You need to briefly mention what () denotes -- std dev? - added (in parentheses)}
}
% \vspace{-2mm}
\label{tab:first-vs-last}
\centering
\resizebox{\linewidth}{!}{ 
\begin{tabular}{cccc|ccc|ccc|ccc}
\hline
  \multirow{2}{*}{Layers}  &        \multicolumn{3}{c|}{Roberta-Large}        & 
  \multicolumn{3}{c|}{Llama-3.2 1B} &
  \multicolumn{3}{c|}{Qwen-2.5 1.5B} &
  \multicolumn{3}{c}{Mistral 7B} \\
  % \cline{2-13}
& DataInf & TracIn & Cosine & DataInf & TracIn & Cosine & DataInf & TracIn & Cosine & DataInf & TracIn & Cosine \\
\hline 
TI$^{10}_{we}$ 
&    -    & 1 (.49) &  -      
&    -    & \textbf{1 (.56)} &  -
&    -    & 2 (.49) &  - 
&    -    & 1 (.58) &  -    \\
TI$_{we}$ 
&    -    & 2 (.48) &   -
&    -    & 2 (.52) &   -
&    -    & 1 (.51) &   -  
&    -    & 2 (.50) &   -   \\
\texttt{WE}
&    -    & \textbf{1 (.54)} & 2 (.41) 
&    -    & 3 (.44) & 3 (.42)  
&    -    & 3 (.40) & 3 (.19)  
&    -    & 2 (.50) & 1 (.58) \\
1
& 4 (.29) & 3 (.50) & 3 (.35) 
& 1 (.53) & 2 (.54) & \textbf{1 (.53)} 
& \textbf{1 (.49)} & 2 (.42) & 2 (.45)
& 2 (.53) & \textbf{1 (.61)} & \textbf{1 (.60)} \\
2
& 4 (.30) & 5 (.40) & 2 (.40) 
& \textbf{1 (.58)} & 1 (.54) & \textbf{1 (.53)} 
& 1 (.48) & \textbf{1 (.57)} & 1 (.60)
& \textbf{1 (.60)} & 2 (.48) & 2 (.55) \\
3 
& 2 (.40) & 4 (.44) & \textbf{1 (.56)}
& 2 (.42) & 4 (.38) & 2 (.45) 
& \textbf{1 (.49)} & 2 (.48) & \textbf{1 (.61)}
& 2 (.39) & 3 (.29) & 3 (.37) \\
4
& \textbf{1 (.55)} & 6 (.36) & 1 (.54) 
& 3 (.23) & 5 (.22) & 4 (.28) 
& 2 (.31) & 4 (.30) & 2 (.40) 
& 3 (.20) & 4 (.20) & 4 (.23) \\
\texttt{CL} 
& 3 (.36) & 7 (.17) & 4 (.15) 
& 4 (.18) & 6 (.17) & 5 (.15) 
& 3 (.16) & 5 (.22) & 3 (.18) 
& 3 (.20) & 4 (.20) & 5 (.09) \\
\hline
\end{tabular}
}
\end{table}

Here, the ranking is given across layers. In many cases, first or second attention groups happen to be most performant, which makes them a good selection for detrimental sample detection. 

\section{Influence score correlations between layers and methods}

Figure~\ref{fig:corr-all} presents computed correlations between training-sample attributed scores of different layers and influence methods. TracIn splits the model into three groups of \emph{early, middle, and later} layers based on their score agreement, i.e., the strength of correlation. DataInf has high correlation with TracInf on later layers such as \texttt{CL}. This layer split suggests that different parameters could capture different detrimental samples, and a good aggregation of scores is viable.  

\begin{figure}[ht]
    \centering
    \includegraphics[width=\textwidth]{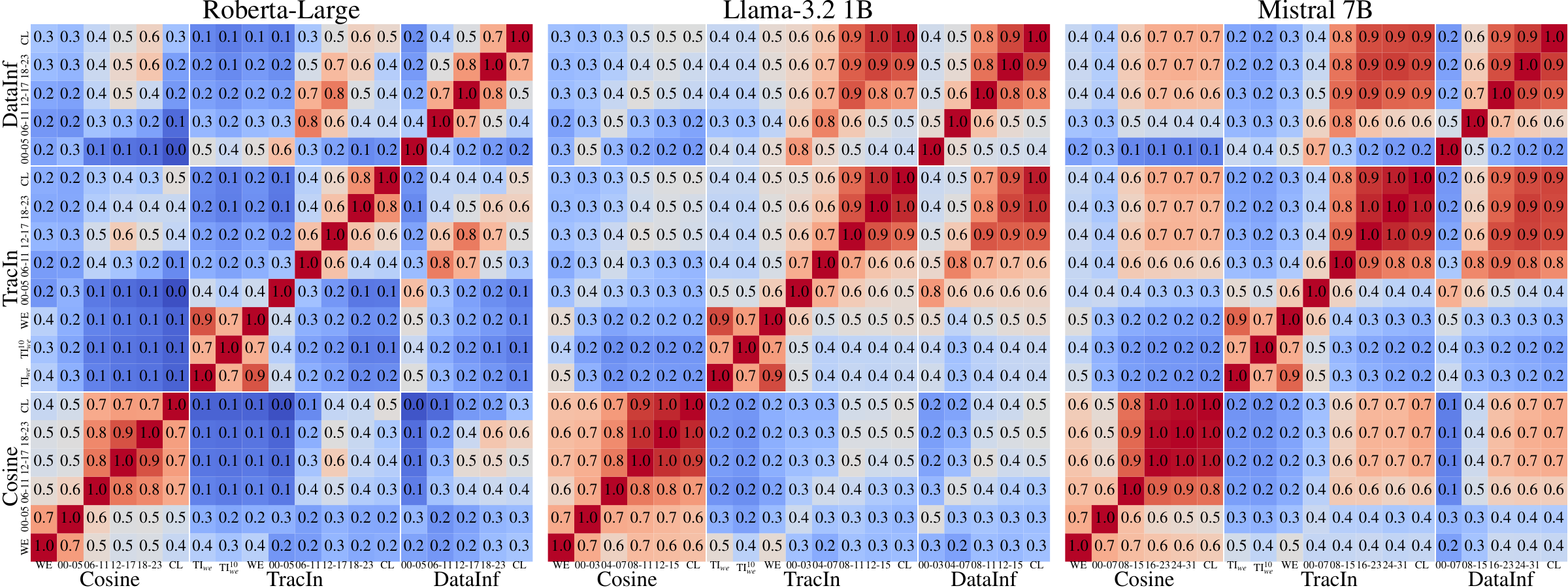}    
    \caption{Influence scores correlation for 3 models, methods, and layers.}
    \label{fig:corr-all}
\end{figure}

\clearpage

\section{Detected noise in 30\% least influential samples}
\label{sec:a-ndr-across-layers}

The following results present the noise detection across layers for Roberta-Large (Figure~\ref{fig:r-ndr-across-layers}), Llama-3.2 1B (Figure~\ref{fig:l-ndr-across-layers}), Qwen-2.5 1.5B (Figure~\ref{fig:q-ndr-across-layers}), Mistral 7B (Figure~\ref{fig:m-ndr-across-layers}). The consequent conclusions are (1) LoRA Value B is the best in capturing noise across models; (2) layers with high NDR also have high downstream performance; (3) spike of NDR moves from later attention layers for shallow models to the early-middle layers for deeper models.

\begin{figure}[h!]
    \centering
    \includegraphics[width=\textwidth]{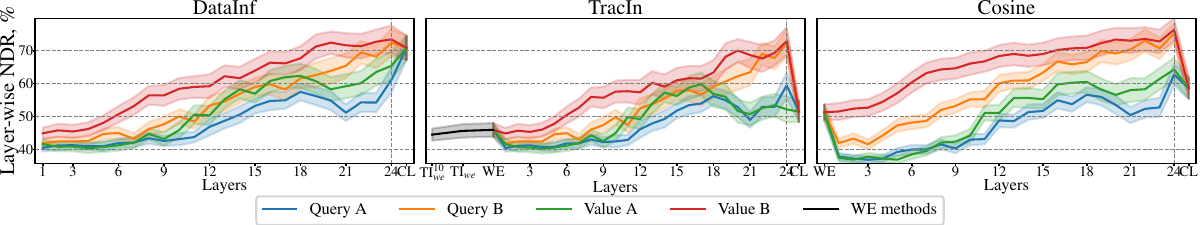}    
    \caption{Roberta-Large NDR for 30\% threshold across modules and layers. 
    % \td{Need to tell people which model for each of these.}  
    }
    \label{fig:r-ndr-across-layers}
\end{figure}
\begin{figure}[h!]
    \centering
    \includegraphics[width=\textwidth]{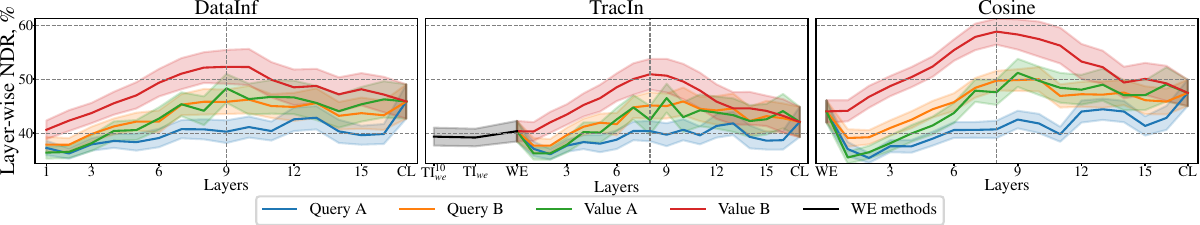}    
    \caption{Llama-3.2 1B NDR for 30\% threshold across modules and layers.  }
    \label{fig:l-ndr-across-layers}
\end{figure}
\begin{figure}[h!]
    \centering
    \includegraphics[width=\textwidth]{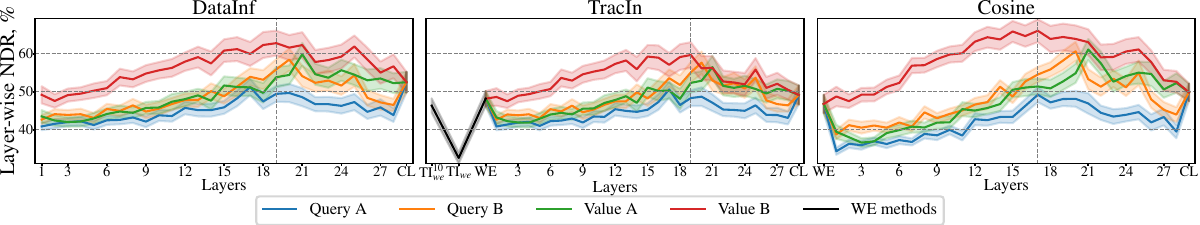}    
    \caption{Qwen-2.5 1.5B NDR for 30\% threshold across modules and layers.  }
    \label{fig:q-ndr-across-layers}
\end{figure}
\begin{figure}[h!]
    \centering
    \includegraphics[width=\textwidth]{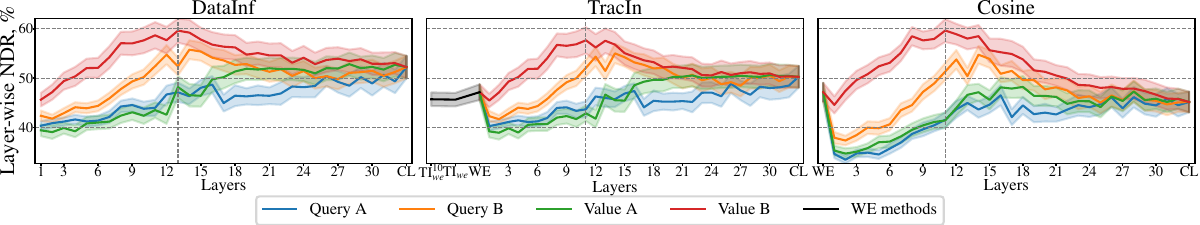}    
    \caption{Mistral 7B NDR for 30\% threshold across modules and layers.  }
    \label{fig:m-ndr-across-layers}
\end{figure}

Next, from the following result we can conclude that Value B on network middle layers captures the highest noise levels and is a good choice for influence estimation. 

\clearpage

\section{Best test set accuracy after filtering}
\label{sec:a1-test-set-accs}

The following tables~\ref{tab:r-test-acc},\ref{tab:l-test-acc},\ref{tab:m-test-acc} present the ranking of filtering methods based on different influence scoring and aggregation for 3 models under consideration. 

\input{tables/roberta/test-set-acc}

Rows are ordered according to mean rank from the best method to the worst. Two baselines are Random (dropping 30\% uniformly) and Full (removing all noise (20\%) and 10\% of random clean samples). A good filtering method should outperform the Random, the lower bound, and approach the Full, upper bound. The tables highlight the best non-baseline methods for each GLUE dataset. AUC measure is specified for the NDR curve; a high value indicates skew of noise towards the beginning of the influence range.

\input{tables/llama/test-set-acc}

\input{tables/mistral/test-set-acc}

\clearpage

\section{Proof of Cancellation Effect Theorem}
\label{sec:a-c-proof}

% \begin{theorem}[\textbf{Invalidating \citep{Yen-2022}'s Findings} \td{Not consistent with theorem in the main paper, also $\Delta I$ needs to be defined here similarly. Make sure everything is consistent.}]
% Let $X$ be a training set. Consider:  
% \begin{enumerate}
%     \item two samples $\bar{x}_1, \bar{x}_2 \in X$, where $\bar{x}_1$ is noisy and $\bar{x}_2$ is clean;
%     \item two parameter vectors $\theta$ and $\omega$ such that their cancellation scores satisfy 
%     $C(\theta) \ll C(\omega)$ with $C(\omega)\to\infty$ for $\{\bar{x}_1,\bar{x}_2\}$;
%     \item the TracIn influence scores $I_\theta$ (based on $\theta$ alone) and $I_{\theta,\omega}$ (based jointly on $\theta,\omega$).
% \end{enumerate}
% Then there exists a validation point $\bar{x}_3$ such that
% \[
% \Delta I_{\theta,\omega}(\bar{x}_1,\bar{x}_2\mid \bar{x}_3)
%    > \Delta I_\theta(\bar{x}_1,\bar{x}_2\mid \bar{x}_3) > 0,
% \]
% i.e.\ the separation between noisy and clean samples is strictly larger under $I_{\theta,\omega}$, thereby improving discrimination.
% \end{theorem}

\begin{theorem}[\textbf{Cancellation Can Improve Influence Estimation.}]
Let $X$ be a training set. Consider:  
\begin{enumerate}
    \item two samples $\bar{x}_1, \bar{x}_2 \in X$, where $\bar{x}_1$ is noisy and $\bar{x}_2$ is clean;
    \item two parameter vectors $\theta$ and $\omega$ such that their cancellation scores satisfy 
    $C(\theta) \ll C(\omega)$ with $C(\omega)\to\infty$ for $\{\bar{x}_1,\bar{x}_2\}$;
    \item the TracIn influence scores $I_\theta$ (based on $\theta$ alone) and $I_{\theta,\omega}$ (based jointly on $\theta,\omega$);
    \item influence score distance between $\bar{x}_1$ and $\bar{x}_2$ w.r.t. weights $\Theta$ and validation samples $X'$: $\Delta_\Theta I(\bar{x}_1,\bar{x}_2 \mid X') = |I(\bar{x}_1, X', \Theta) - I(\bar{x}_2, X', \Theta)|$.
\end{enumerate}
Then there exists a validation point $\bar{x}_3$ such that: 
% \td{Need to define $\Delta I$ somewhere before here}
\[
\Delta I_{\theta,\omega}(\bar{x}_1,\bar{x}_2\mid \bar{x}_3)
   > \Delta I_\theta(\bar{x}_1,\bar{x}_2\mid \bar{x}_3),
\]
i.e.\ the separation between noisy and clean samples is strictly larger under $I_{\theta,\omega}$, thereby showing that contrary to the claim of \citep{Yen-2022}, the inclusion of weights with high cancellation can \textit{improve} influence estimates.
\end{theorem}

\begin{proof}
Suppose, for contradiction, that no such validation point exists. Then for all $\bar{x}_3$,
\[
\Delta I_{\theta,\omega} \leq \Delta I_\theta.
\]

Let
\[
a_i = \frac{\partial \ell(\bar{x}_i,\theta)}{\partial \theta}, 
\qquad 
b_i = \frac{\partial \ell(\bar{x}_i,\omega)}{\partial \omega}.
\]

\noindent
\textbf{Cancellation assumptions.}
Since $C(\omega)\to\infty$, the gradients at $\omega$ nearly cancel:
\[
b_1 b_2 < 0, 
\qquad |b_1+b_2| = \varepsilon \to 0^+.
\]
For $\theta$, cancellation is finite:
\[
a_1 a_2 < 0, \qquad |a_1+a_2| > 0.
\]

\noindent
\textbf{Influence relation.}
From the TracIn definition and the cancellation assumptions, without loss of generality, one obtains
\begin{align*}
\Delta I_{\theta} = |a_3 (a_1 - a_2)|,\ \Delta I_{\theta,\omega} = |\begin{bmatrix} 
    a_3 \\
    b_3
\end{bmatrix}^\top
\begin{bmatrix}
a_1 - a_2 \\
b_1 - b_2
\end{bmatrix}
| = |a_3 (a_1 - a_2)| |(1 + \frac{b_1 - b_2}{a_1 - a_2} \frac{b_3}{a_3})|
\end{align*}
\begin{align*}
C(\omega) \gg C(\theta) \implies \frac{b_1 - b_2}{\epsilon} \gg \frac{a_1 - a_2}{|a_1 + a_2|} \implies
\Delta I_{\theta,\omega} 
   \gg \Delta I_\theta\left( 1 + \frac{\varepsilon}{|a_1+a_2|}\cdot \frac{b_3}{a_3}\right).
\end{align*}

\noindent
\textbf{Contradiction.}
By assumption $\Delta I_{\theta,\omega}\leq \Delta I_\theta$, which requires
\[
\frac{\varepsilon}{|a_1+a_2|}\cdot \frac{b_3}{a_3} < 0.
\]
Since $\varepsilon>0$, this inequality forces $a_3 b_3 < 0$.

Thus the assumption can only hold if \emph{all} validation points $\bar{x}_3$ yield anti-aligned gradients under $\theta$ and $\omega$. Such universal anti-alignment is impossible, because $\theta$ and $\omega$ typically share co-directional gradients on at least some validation points. This yields a contradiction.

Therefore, there must exist a validation point $\bar{x}_3$ such that
\[
\Delta I_{\theta,\omega} > \Delta I_\theta,
\]
completing the proof.
\end{proof}

\noindent\textbf{Remark.} This shows that contrary to the claim of~\citep{Yen-2022}, inclusion of weights with high cancellation can \emph{improve} discrimination in influence scoring, although the improvement depends on the choice of validation sample. Indeed, there are also cases where $\Delta I_{\theta,\omega} < \Delta I_\theta$, so the proposition in~\citep{Yen-2022} cannot be concluded in full generality.

\clearpage

\section{Ranking and voting on Roberta and Llama}
\label{sec:a3-agg-methods}

The figures~\ref{fig:rank-vote-diffs-roberta},~\ref{fig:rank-vote-diffs-llama} show the change of Roberta-Large and Llama-3.2 1B performance when we filter the training set with rank and vote aggregations across different influence methods. Mistral 7B has the most benefit of these methods (see figure~\ref{fig:rank-vote-diffs-mistral}).  Transparent distributions correspond to statistically insignificant (Wilcoxon p-value 0.1, as we have 10 runs only) or less than 1 percent change. 

\begin{figure}[h!]
    \centering
    \includegraphics[width=\linewidth]{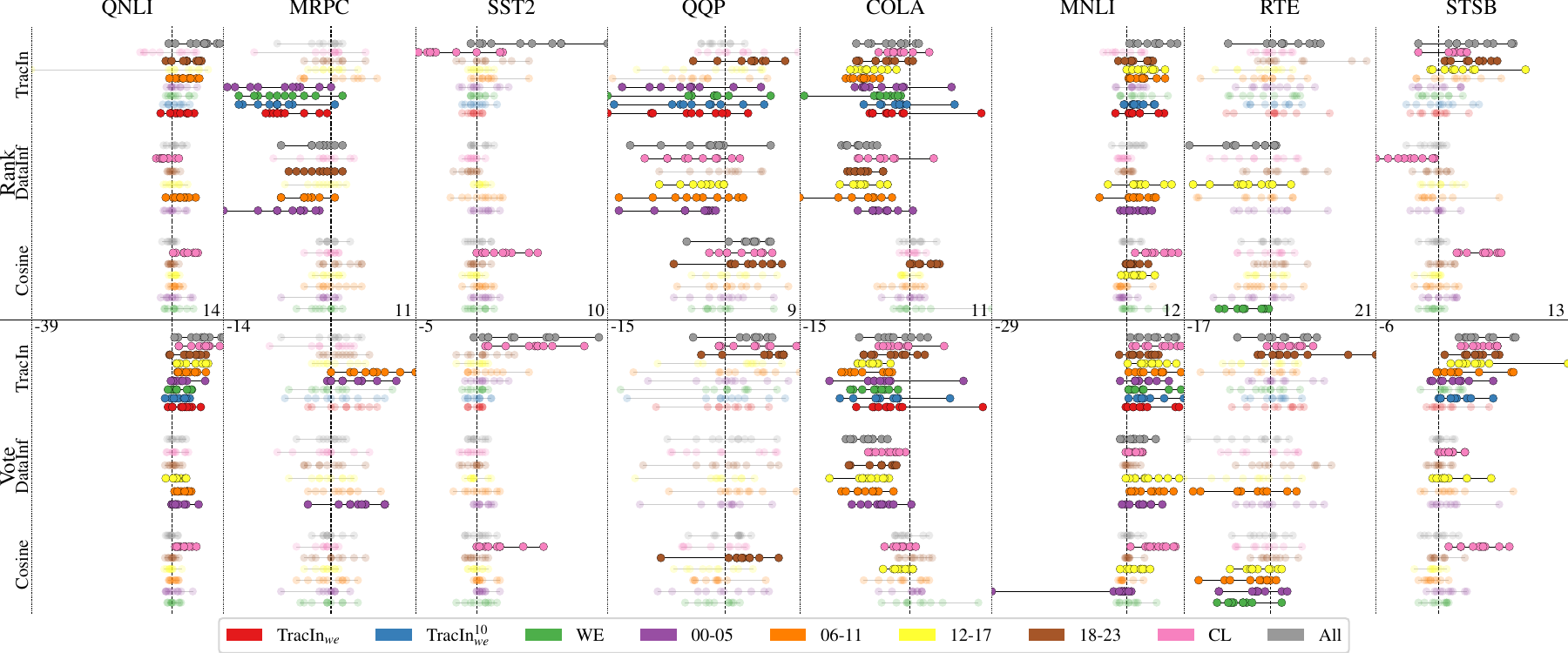}
    \caption{Roberta-Large accuracy change \% with ranking and voting compared to mean aggregation. Voting has the most positive effect on the TracIn influence method, greatly boosting the performance of the last layers (CL) scores. COLA is the ``toughest'' dataset (accuracy drop across layers and methods). The mean aggregation is still better here.  }
    \label{fig:rank-vote-diffs-roberta}
\end{figure}

\begin{figure}[h!]
    \centering
    \includegraphics[width=\linewidth]{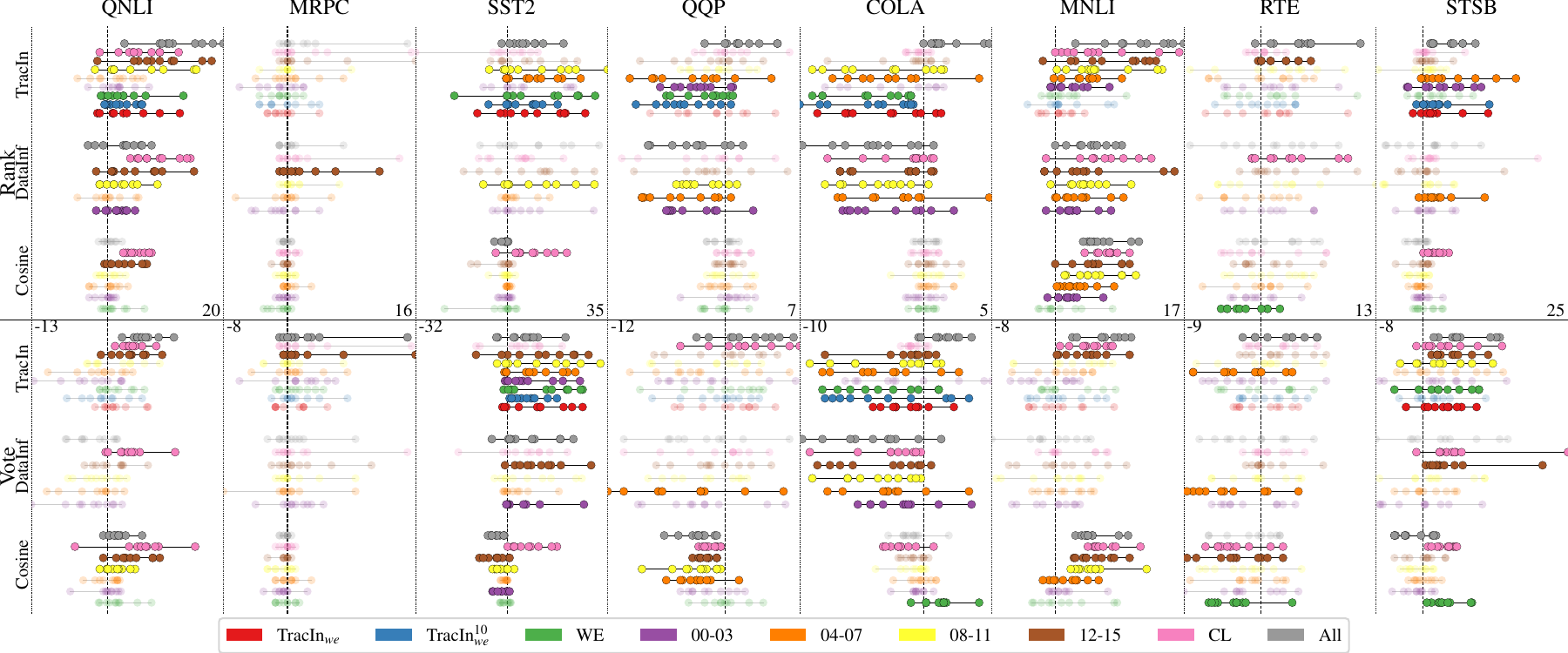}
    \caption{Llama-3.2 1B accuracy change \% with ranking and voting compared to mean aggregation. SST2 and MNLI are greatly boosted, while COLA and QQP demonstrate drops.  }
    \label{fig:rank-vote-diffs-llama}
\end{figure}

For NLI tasks, both Ranking and Voting yield consistent improvements, while on CoLA and QQP we observe a performance decline.

\clearpage

\section{Noise influence distribution on model layers }
\label{sec:a2-noise-distr}

This section presents noise distribution histograms for methods and layers.
% To visualize how scoring orders samples and where the noise is in these orders, we collect histogram information presented in figure~\ref{fig:r-noise-distr}. 
Every chart contains 10 quantiles, 450 samples each, depicting the relative noise count in the corresponding influence value range. The first and last quantiles represent the least and most influential samples, respectively.

\begin{figure}[h!]
    \centering
    \includegraphics[width=\textwidth]{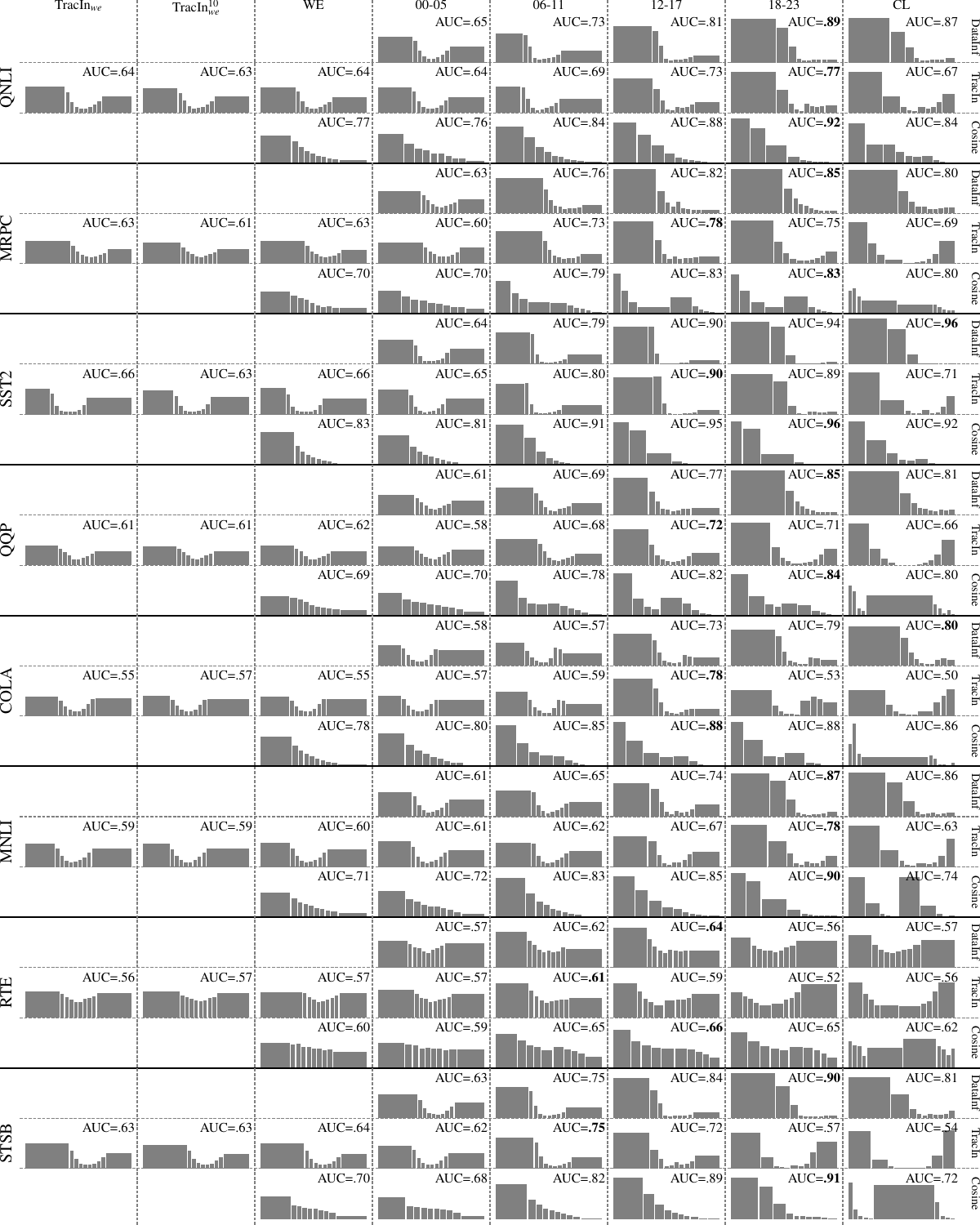}
    \caption{Roberta-Large noise distribution across mean-aggregated influence score range for influence methods, datasets, and selected layers. High bars represent a bigger percentage of noise in the corresponding quantile. The noise is usually in the first quantiles (least influential), as expected. Methods have the least discriminative power on the RTE dataset. 
    TracIn frequently has high noise in the last quantile across layers.     
    }
    \label{fig:r-noise-distr}
\end{figure}

An influence function is not useful for anomaly detection if it is closer to a uniform distribution. For selected methods, most mislabeled samples are concentrated at the beginning. However, some layers and methods have a high influential noise level. For instance, TracIn demonstrates relatively high spikes on both sides across configurations. TracIn on \texttt{WE} and \texttt{CL} has a high percentage of influential noise. Generally, a performant attribution should strive to minimize the noise entropy, producing unimodal but not uniform distributions.

The following figures~\ref{fig:r-noise-distr},~\ref{fig:l-noise-distr},~\ref{fig:m-noise-distr}
present how mislabeled samples are spread in the influence scores. In the experiments, we discard the first two quantiles.

\begin{figure}[h!]
    \centering
    \includegraphics[width=\textwidth]{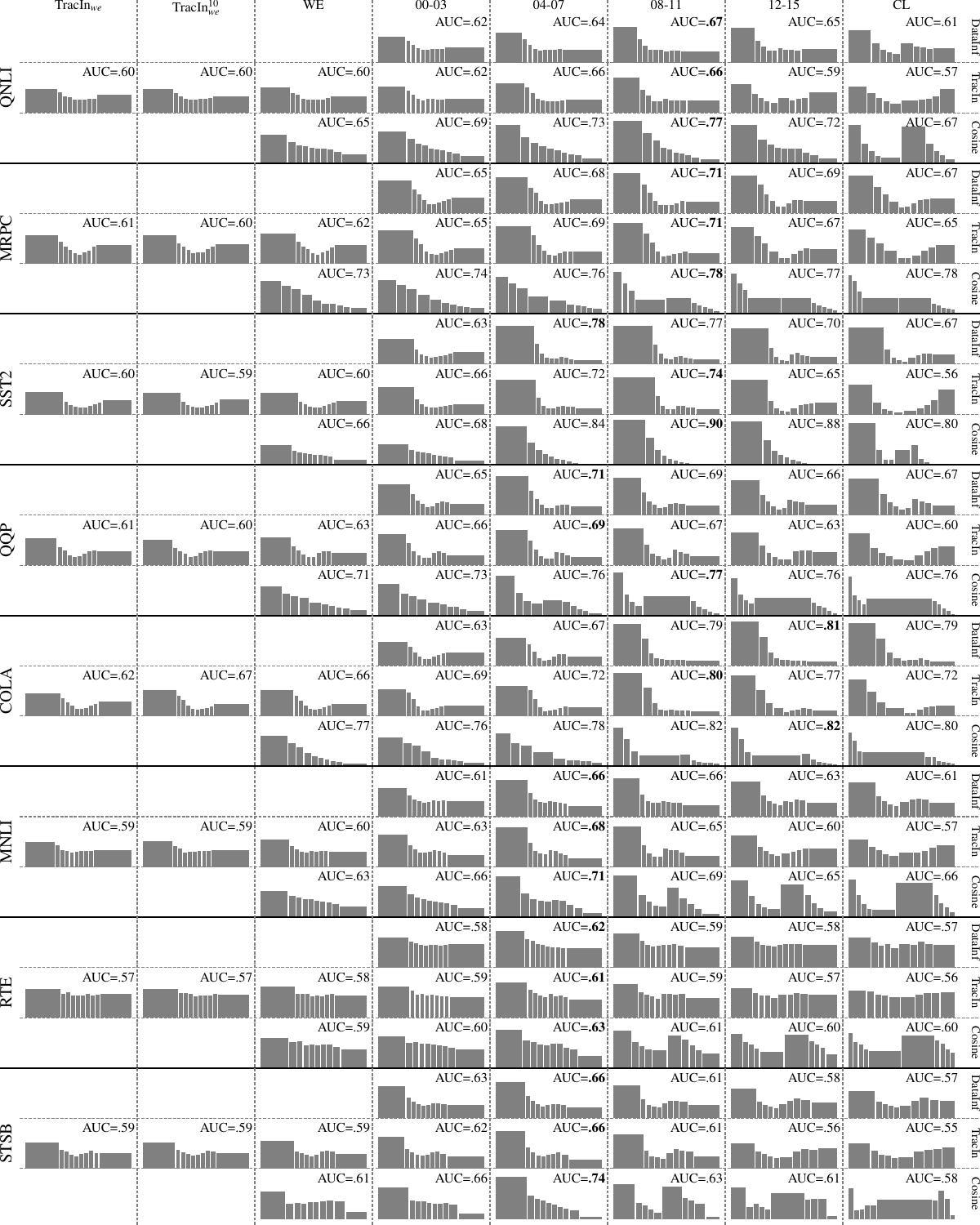}
    \caption{Llama-3.2 1B noise distribution across mean-aggregated influence score range for influence methods, datasets, and selected layers. Similar to Roberta-Large (figure~\ref{fig:r-noise-distr}), in many cases the noise in concentrated in the first quantile. However, Llama-3.2 1B scores demonstrate less discriminative power, effectively spreading the noise more across the range. In other words, noise has more influence on Llama model. 
    }
    \label{fig:l-noise-distr}
\end{figure}

\begin{figure}[h!]
    \centering
    \includegraphics[width=\textwidth]{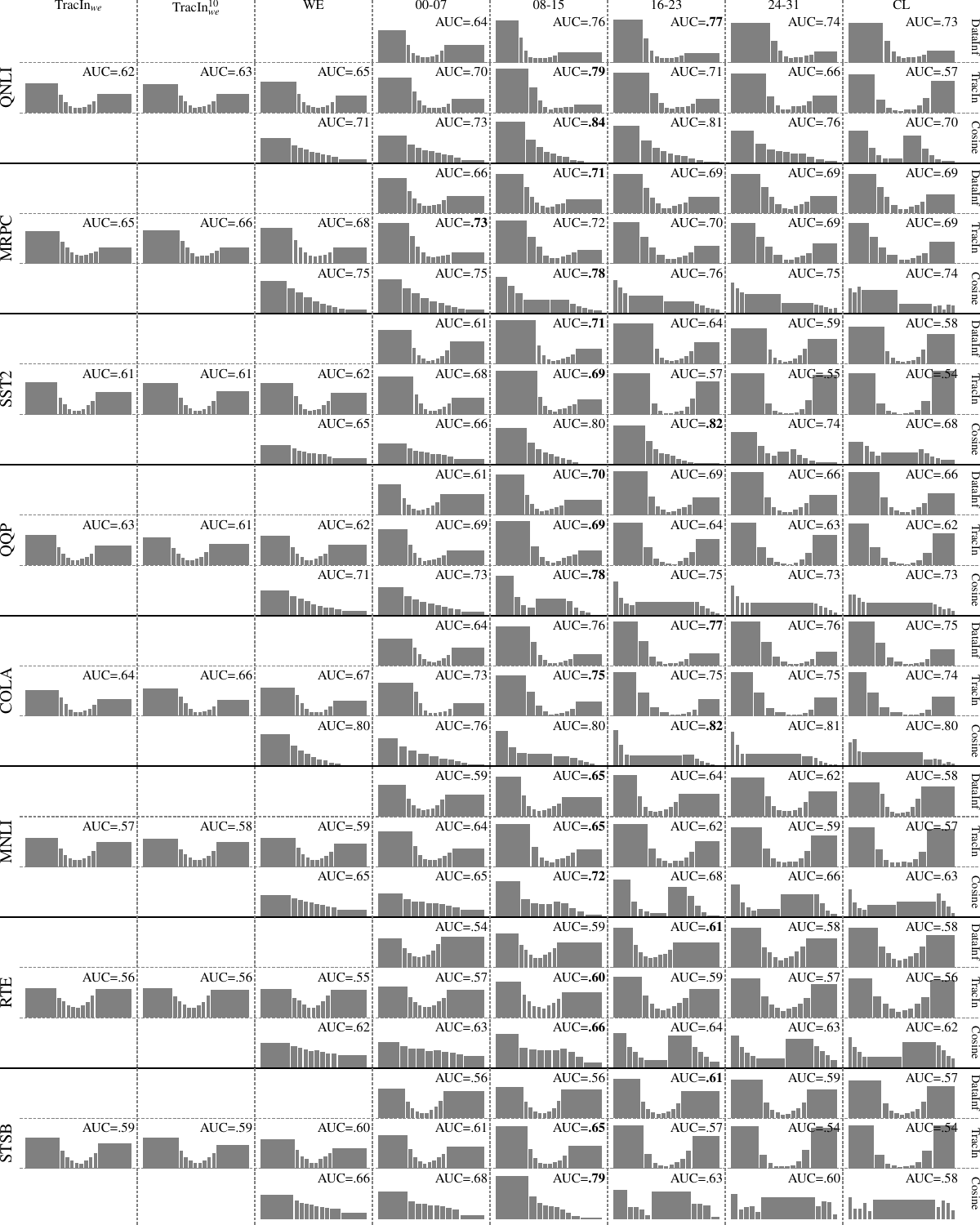}
    \caption{Mistral 7B noise distribution across mean-aggregated influence score range for influence methods, datasets, and selected layers. As with Llama (figure~\ref{fig:l-noise-distr}), we observe that noise has more influence on deeper Mistral. For instance, even for the DataInf method, we observe a high concentration of mislabeled samples in the last quantiles. Cosine similarity has a noise spike in the middle of influence range instead. The middle layers of the Mistral are less susceptible to the influence of noise across methods. 
    }
    \label{fig:m-noise-distr}
\end{figure}

% The influence value range is divided into 10 quantiles, bins of 450 training samples, and the number of noisy samples is counted. Figures demonstrate that the least influential samples contain more noise. Lower bar width indicates a dense concentration of samples. In the experiments, we discard the first two quantiles. However, figures demonstrate that noise could also concentrate on the other side of the influence range.

\clearpage

\section{Influential data identification on autoregressive datasets}
\label{sec:a-autoreg-ds} \vspace{-2mm}

\textbf{Datasets}. To further validate the RQ2 (first is not better than last), we conduct the experiment on the autoregressive datasets introduced in \cite{Kwon-2024}: \textbf{Grammar}, \textbf{Math}, and \textbf{Math (With Reasoning)}. Unlike the settings used in the main text, where detrimental samples were synthetically injected, these datasets contain semantically related instruction categories. In this context, a desirable behavior is that training instructions exert a stronger influence on test samples from the same semantic category. In this way, we assess how well influence scores \emph{recover meaningful relational structure} rather than merely detecting artificial noise. Such scenarios are highly relevant for practical applications, including knowledge probing, debugging data pipelines, and detecting harmful or misaligned training examples.

\textbf{New methods}. Additionally, we extend our experiments to incorporate more recent gradient- and activation-based influence estimation techniques. We report the following results for \textbf{Outlier Gradient}~\cite{Chhabra-2025}, Kronfluence (\textbf{EKFAC})~\cite{Grosse-2023}, and \textbf{RepSim}~\cite{Li-2025}. We detect gradient outliers by scoring training samples with OneClassSVM using an RBF kernel on a per-semantic-category basis. EKFAC scoring combines forward and backward signals to estimate influence. RepSim is based purely on hidden representations measured after each attention layer either on \emph{last} sample token or averaged across all tokens of the sample (\emph{mean}).

\textbf{Models}. The results are provided for \textbf{Qwen2.5-1.5B} (lr=3e-4) and \textbf{Mistral-7B-v0.3} (lr=1e-3).

\textbf{Metrics}. We use the evaluation metrics, \textbf{AUC} and \textbf{Recall}, from the original work~\cite{Kwon-2024}. Higher values indicate bettser alignment between the sample ordering by the influence technique and the actual semantic category of the test samples.

\textbf{Conclusions}. Figures~\ref{fig:q-ds-auc},~\ref{fig:q-ds-recall} present the AUC and Recall across layers for Qwen, and ~\ref{fig:m-ds-auc},~\ref{fig:m-ds-recall} - for Mistral.
We observe that RQ2 conclusion still holds for new datasets.

\textbf{First}, early layers consistently exhibit weaker performance in AUC and Recall for both models across methods. Similarly, the last attention layers show a drop in AUC-Recall in the majority of cases. Unfortunately, the best-performing layers vary with model-dataset-method, but generally, we observe the performance maximum for 20-24 layers. 

\textbf{Second}, for both Qwen and Mistral, the \textbf{LoRA B} module (attached to \textbf{value projections}) delivers the strongest performance across most configurations. Our \textbf{practical recommendation} from consistent observations across all evaluated models and datasets is to rely on influence scores computed on these modules. These weights yield the highest NDR, AUC, Recall, and downstream task performance. 

\begin{figure}[h!]
    \centering
    \includegraphics[width=0.88\textwidth]{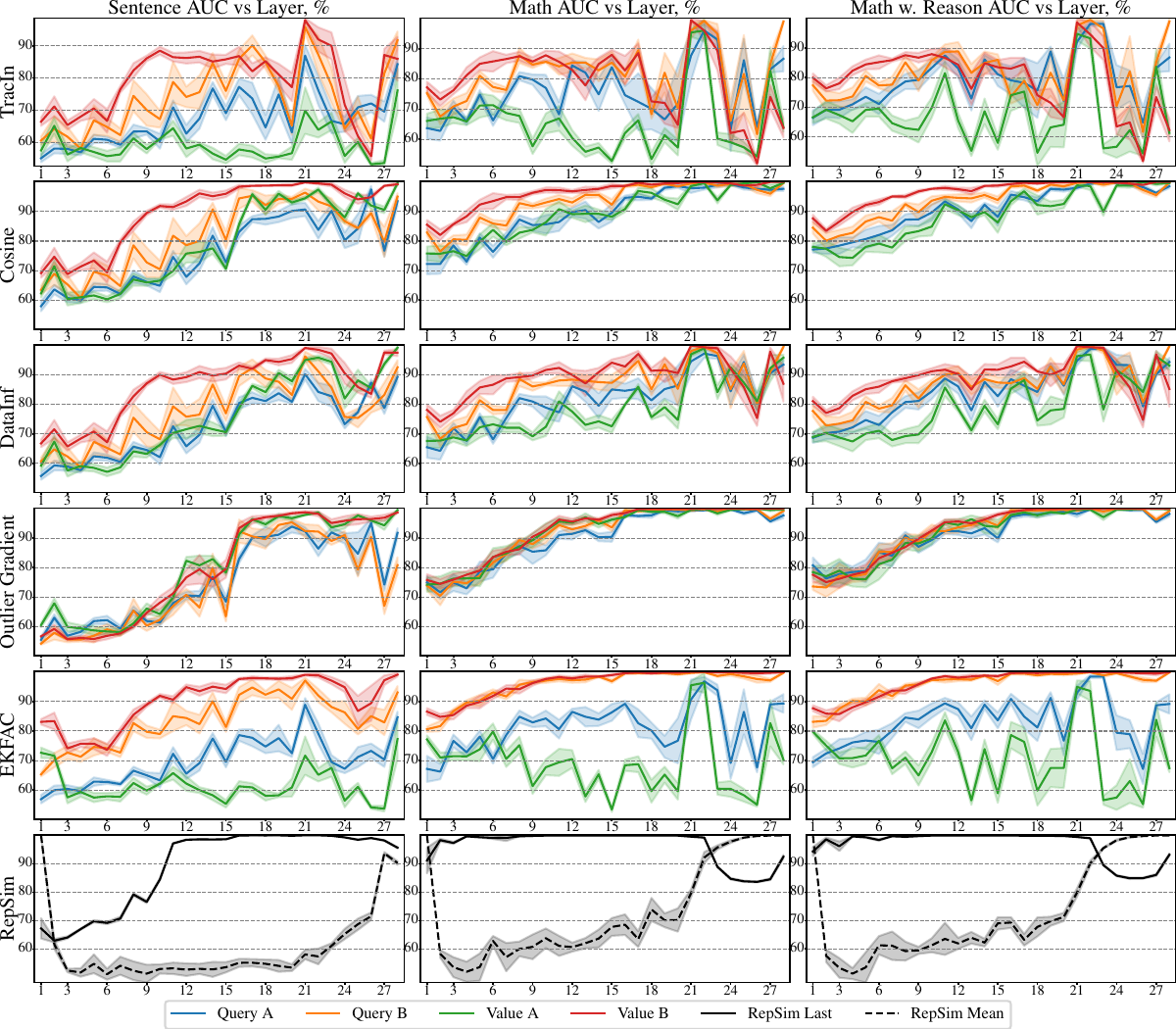}
    \caption{Qwen 2.5 1.5B AUC on influence datasets.}\vspace{-3mm}
    \label{fig:q-ds-auc}
\end{figure}
\begin{figure}[h!]
    \centering
    \includegraphics[width=0.88\textwidth]{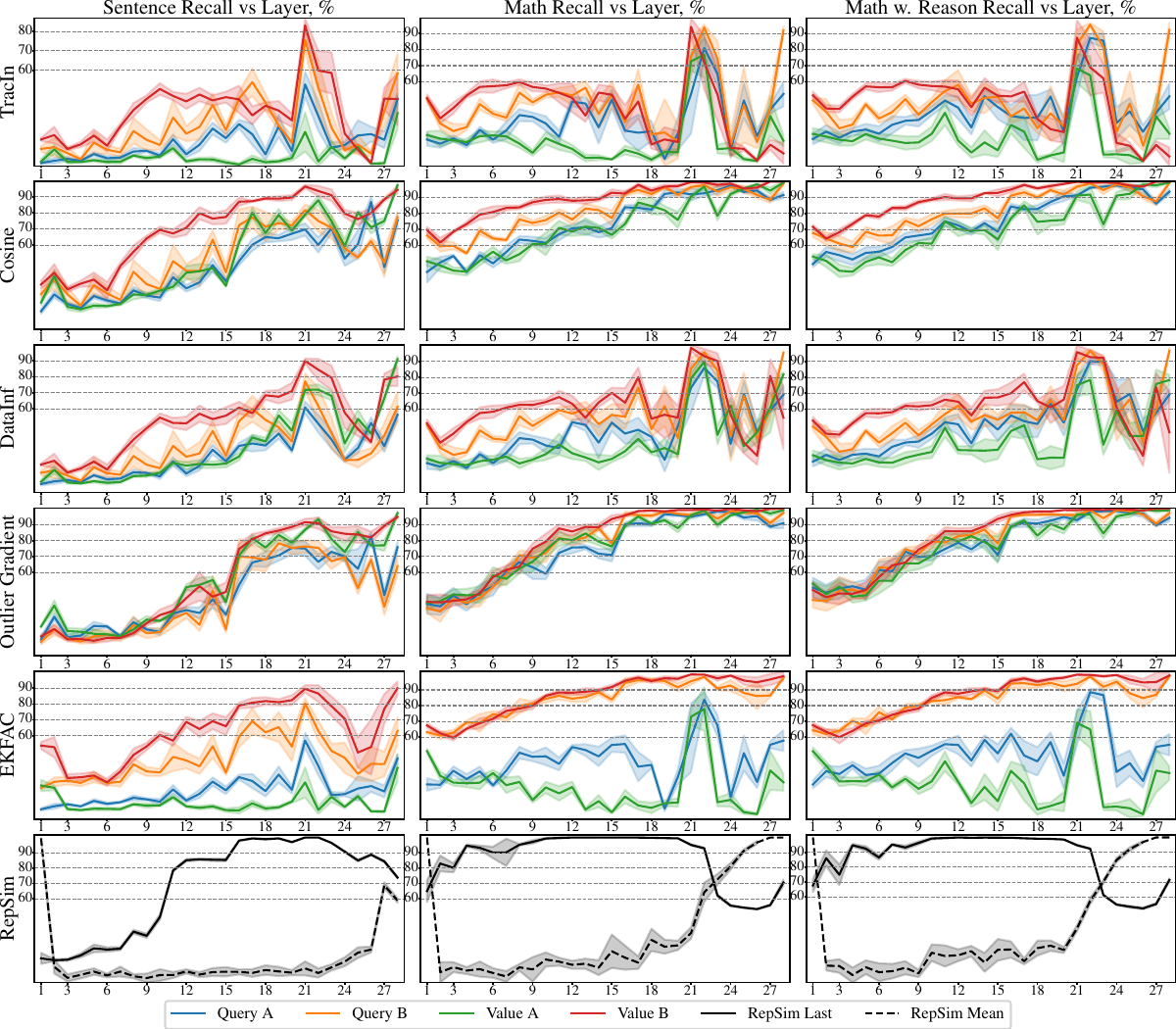}
    \caption{Qwen 2.5 1.5B Recall on influence datasets.}\vspace{-3mm}
    \label{fig:q-ds-recall}
\end{figure}

\begin{figure}[h!]
    \centering
    \includegraphics[width=0.88\textwidth]{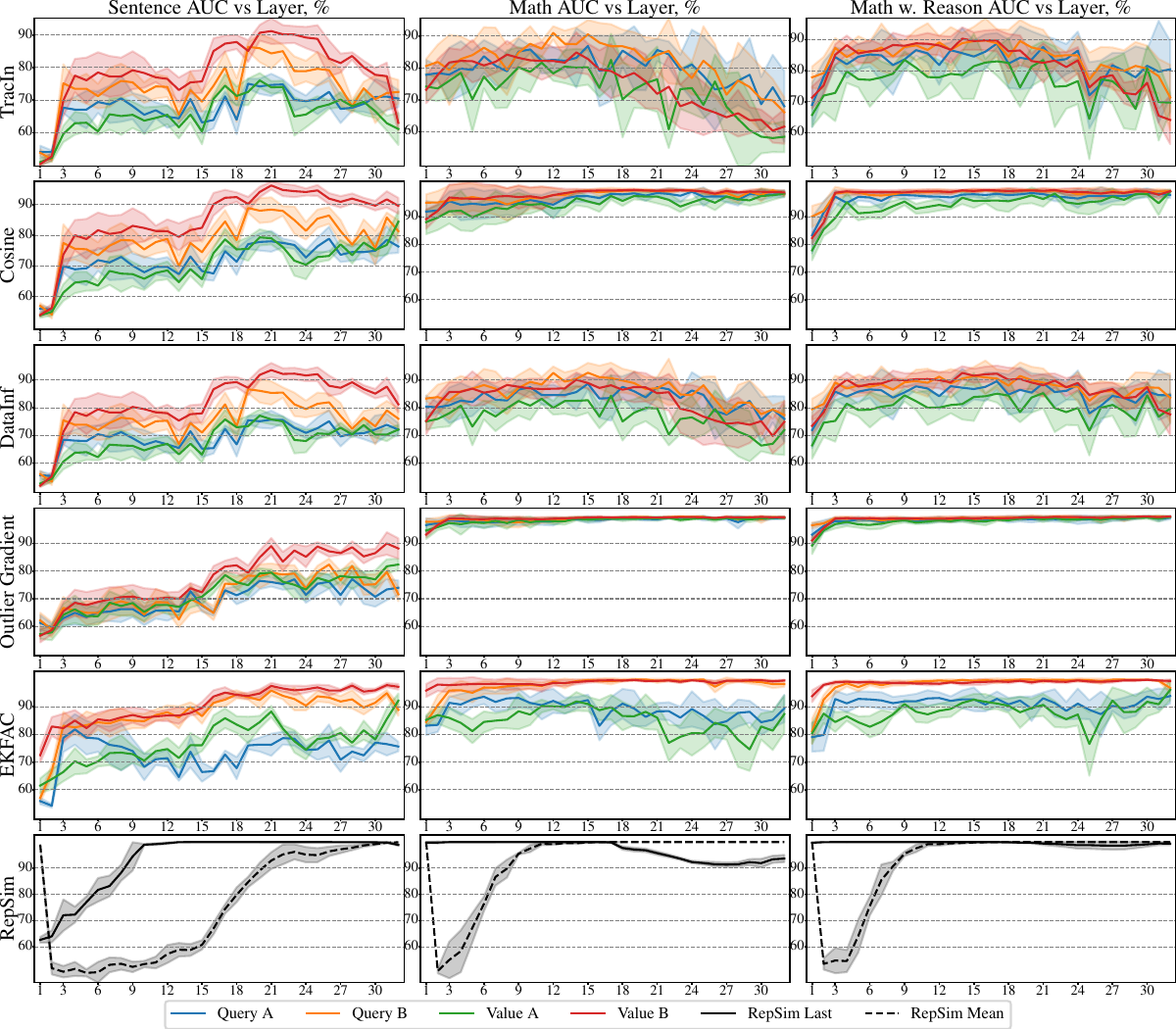}
    \caption{Mitral 7B AUC on influence datasets.}
    \label{fig:m-ds-auc}
\end{figure}
\begin{figure}[h!]
    \centering
    \includegraphics[width=0.88\textwidth]{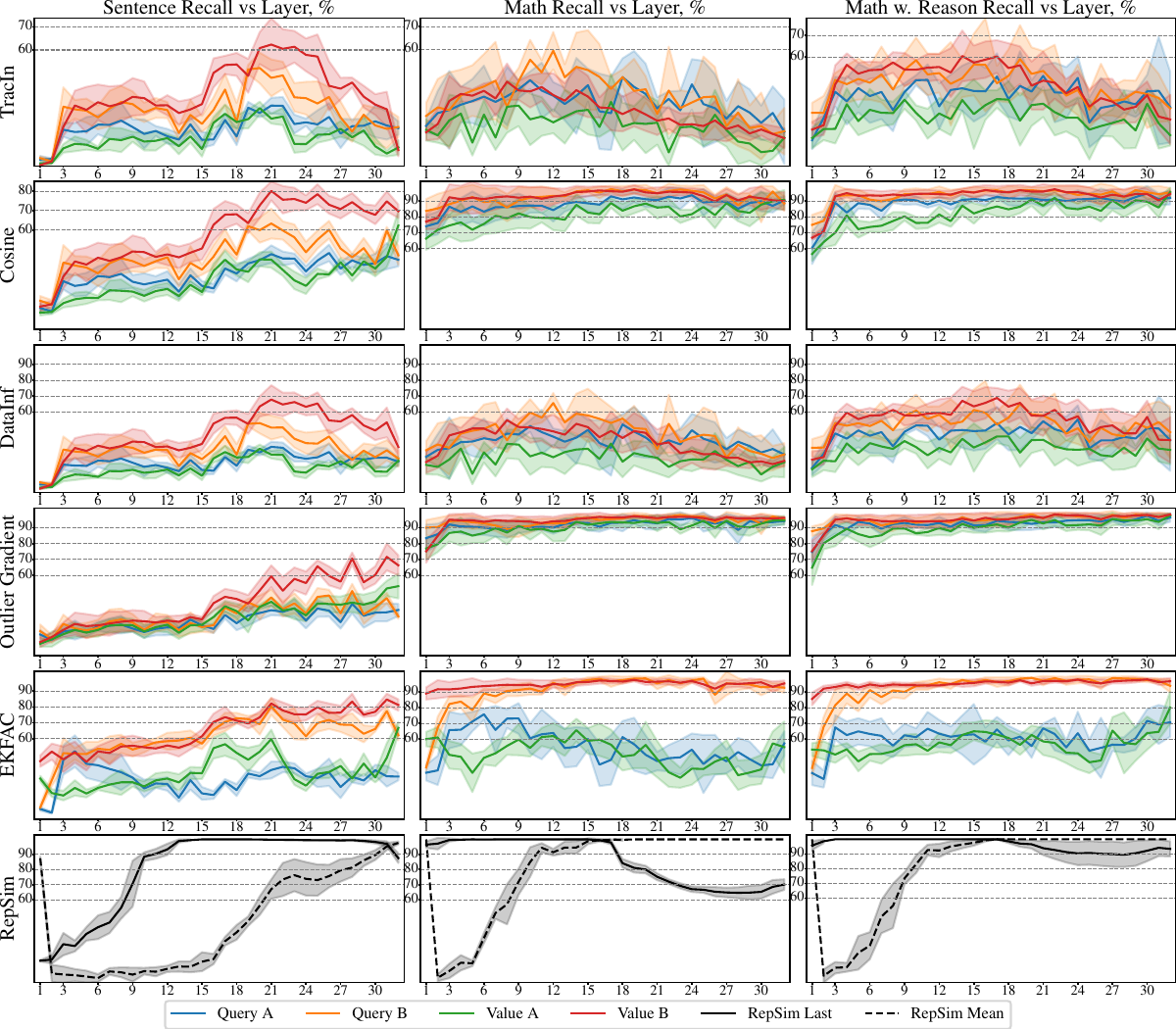}
    \caption{Mistral 7B Recall on influence datasets.}
    \label{fig:m-ds-recall}
\end{figure}

% \textbf{Third}, 

\clearpage

\section{Sample influence variation across model layers}
\label{sec:a-rank-var}

We provide the following visualization of the sample influence ranking trends 
% by the considered influence methods 
as additional interpretation of \emph{why middle layers have better performance}. Figures~\ref{fig:r-rank-var} and ~\ref{fig:q-rank-var} present these ranking, where highest rank means highest influence score, for Roberta-Large and Qwen-2.5 1.5B, respectively, and three ``kinds'' of samples: 
\begin{itemize}
    \item \textbf{Top Ranked Noise Sample} - the noisy sample with the highest influence across layers;
    \item \textbf{Avg Noise} - the averaged rank of all noisy samples per layer;
    \item \textbf{Avg Benign} - the averaged rank of ``clean'' samples per layer. 
\end{itemize}
  
\begin{figure}[h!]
    \centering
    \includegraphics[width=\textwidth]{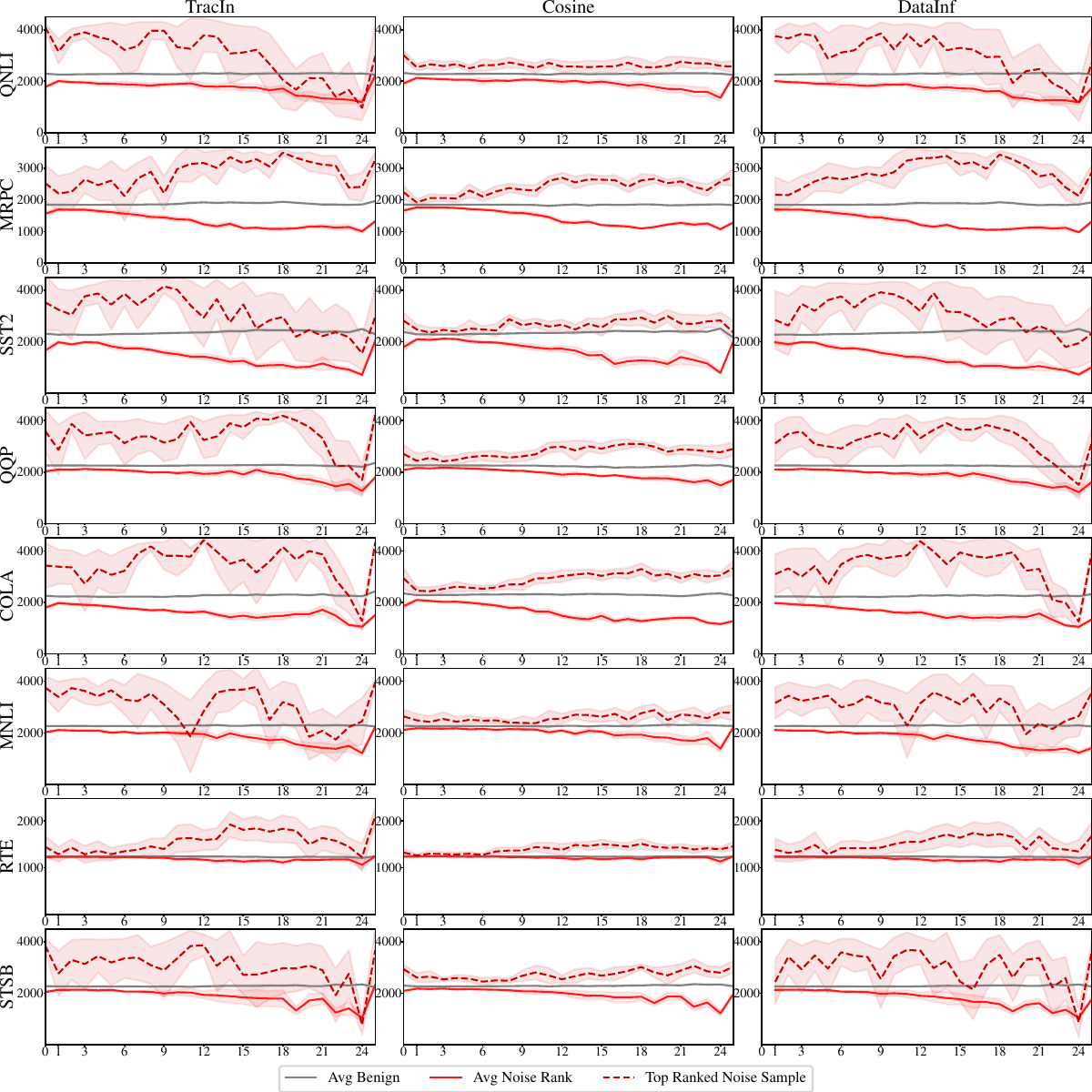}
    \caption{Roberta-Large sample influence rank variation with layer. Mean and 95\% confidence intervals are presented for 10 random seeds. }
    \label{fig:r-rank-var}
\end{figure}

\textbf{Conclusion}. We observe that the average noise influence rank drops with layer, i.e. noise becomes less influential. At the same time, the benign samples maintain the influence rank almost at same level. For Roberta-Large, the biggest difference between average ranks is observed at the last influence layers. This corresponds to the downstream performance maximum (Figure~\ref{fig:r-acc}), where the best accuracy is also observed on later attention layers. Similarly, for Qwen-2.5 1.5B, the noise average rank has minimum on layers 15-20, where the model has also the best downstream performance (Table~\ref{tab:first-vs-last}). 

At the same time, for Qwen, in many cases, the top ranked noisy sample maintain high values of influence when the average noise rank has minimum, signifying the presence of the noise outliers that may affect the performance. 

\begin{figure}[h!]
    \centering
    \includegraphics[width=\textwidth]{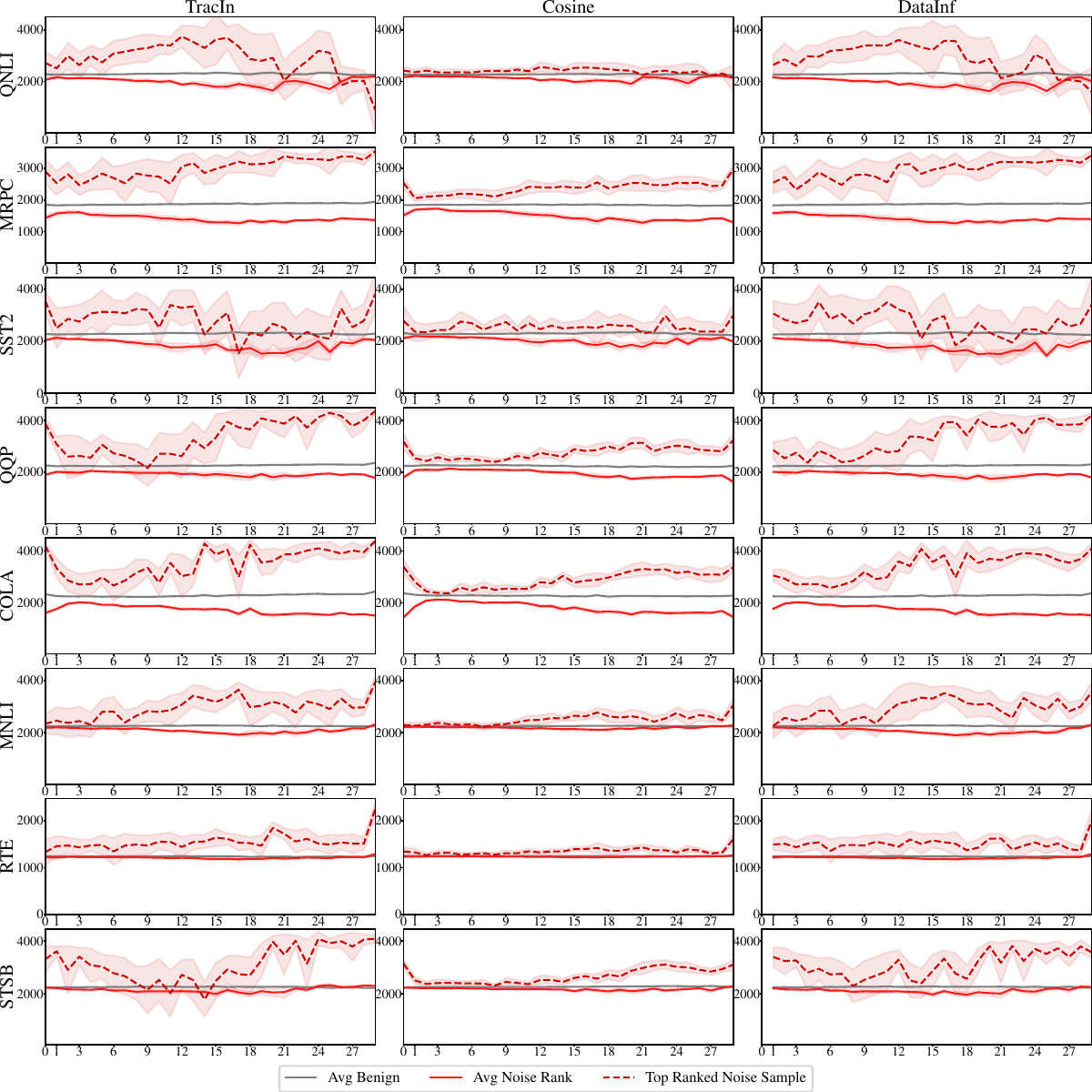}
    \caption{Qwen-2.5 1.5B sample influence rank variation with layer. Mean and 95\% confidence intervals are presented for 10 random seeds.}
    \label{fig:q-rank-var}
\end{figure}

\clearpage

\section{The effect of varying the number of votes in positional voting}
\label{sec:a-vote-k}

Formula~\ref{eq:vote-agg} contains the hyperparameter $k$, the number of votes that each test sample assigns to training samples. The reasonable question is \emph{how the noise detection rate changes with $k$} and \emph{what $k$ would correspond to the best performance}? The following figure~\ref{fig:q-vote-k} presents the NDR for $k$ in range $[10,100]$ with step $10$, for different layers and per NN module of Qwen-2.5 1.5B. The NDR value is averaged across datasets and random seeds. 

\begin{figure}[h!]
    \centering
    \includegraphics[width=0.9\textwidth]{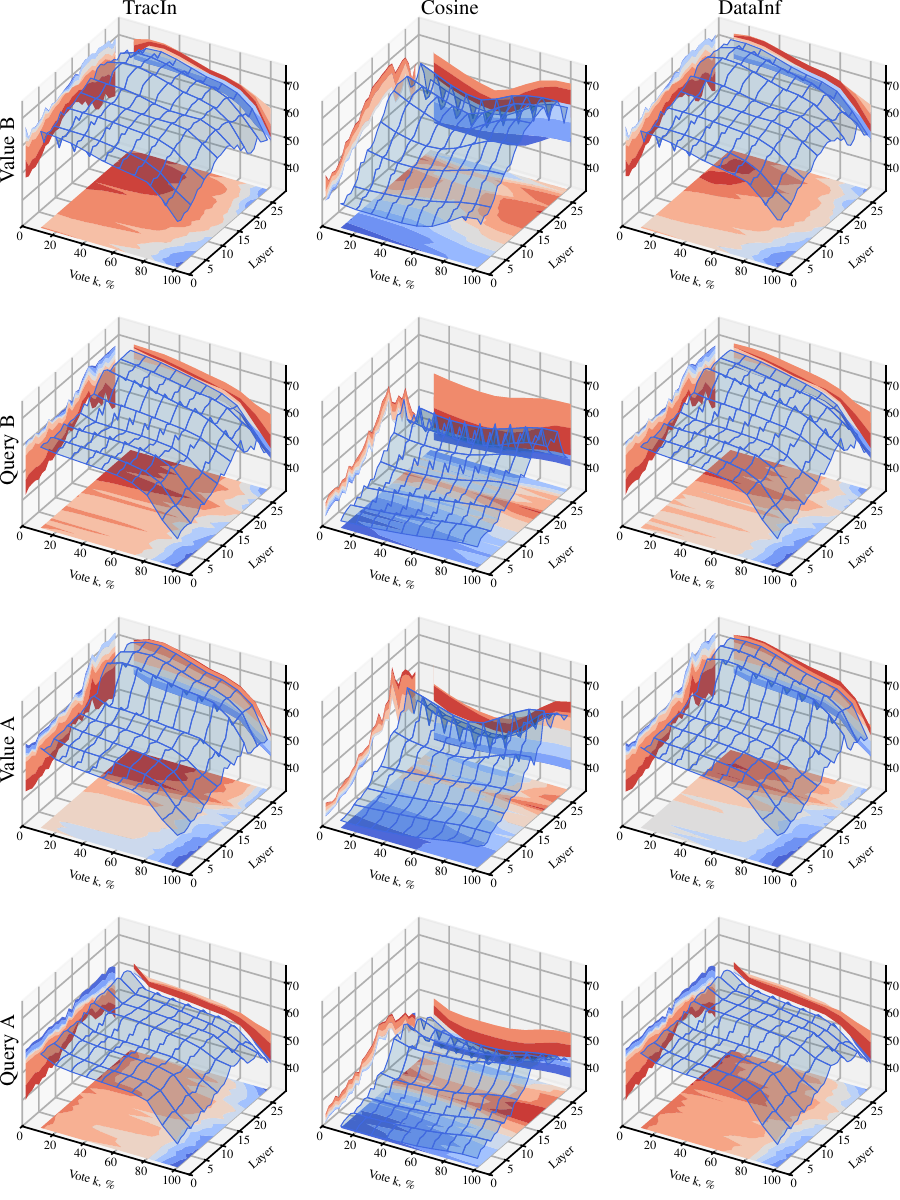}
    \caption{Qwen-2.5 1.5B NDR variation for different Vote $k$ and layer number.}
    \label{fig:q-vote-k}
\end{figure}

\textbf{Conclusion}. NDR values show a clear dependence on the choice of $k$. The range $k \in [10,50]$ yields the strongest performance for most settings. For DataInf and TracIn, NDR decreases monotonically beyond this range, suggesting reduced discriminative power at larger subset sizes. Cosine shows the opposite pattern, reaching its minimum at $k \approx 50$ and improving for larger $k$. We use $k=30$ (filtering threshold) in all main experiments. Determining how the optimal $k$ scales with the level of injected noise remains an open question.

\clearpage

\section{Llama-3.2 1B influence score performance discussion}
\label{sec:a-llama}

In Section~\ref{sec:layers-exp}, we compare the training data filtering by influence attribution with random baseline. We found, that particularly for LLaMA-3.2 1B, the influence functions fail to outperform random filtering. In this section, we 
% provide the discussion why this may happen and 
verify if the best by NDR module (Figure~\ref{fig:l-ndr-across-layers}), \textbf{LoRA B} attached to \textbf{value projection}, outperforms the baseline. 

The artifacts observed for LLaMA-3.2-1B likely arise from its reduced depth (\textbf{16} attention layers) relative to larger models (Roberta-Large - \textbf{24}, Qwen-2.5 1B - \textbf{28}, Mistral 7B - \textbf{32}). The variation in best-performing layers across architectures indicates that reliable influence attribution requires not only well-chosen probes but also effective layer localization. This layer dependence parallels findings in other domains where the success of model interventions hinges on identifying the most relevant layers.

The following table~\ref{tab:l-best} presents the accuracy on GLUE datasets after influence filtering by Value B score of \textbf{8th} and \textbf{9th} attention layers (best according to NDR). Results include \textbf{Mean}, \textbf{Rank}, and \textbf{Vote} aggregations. 

\input{tables/llama/tun2-best_accuracy_1-vb}

We observe that Cosine scores (gradient similarities) outperform Random with Vote aggregation ($k=20$) on selected modules. The differences are statistically significant. Wilcoxon test p-value is 0.0072 for configuration (Cosine, Vote, v-b-9) and Random; p-value is 0.0085 for (Cosine, Vote, v-b-8) and Random. At the same time, (DataInf, Vote, v-b-8) and Random have the p-value 0.47, signifying that the null hypothesis of their similarity cannot be rejected. At the same time, Mean aggregation cannot outperform the baseline even on the best NDR layers. 

\textbf{Conclusion}. With the new Vote aggregation strategy, influence scoring outperforms the baselines on locations that perform best according to NDR. Our practical recommendation for LLaMA-3.2-1B, and consistently across the evaluated models, is to rely on influence scores computed on the LoRA B value projection modules. These weights yield strong NDR performance also on autoregressive datasets (Appendix~\ref{sec:a-autoreg-ds}).

% NOTE: we eventually will remove this - just for us
% \begin{sidewaystable}
% \centering
% \caption{Friedman and post-hoc Nemenyi tests for \textbf{best test set accuracy} after filtering on QNLI, MRPC, QQP dataset. We use only 5 runs (not enough for good stat result). Friedman p-value = 1.9e-7}
% \resizebox{\textwidth}{!}{  
% \input{tables/roberta/friedman/test_set_acc/best_accuracy_1-qnli-friedman}
% }

% \smallskip

% \resizebox{\textwidth}{!}{  
% \input{tables/roberta/friedman/test_set_acc/best_accuracy_1-mrpc-friedman}
% }

% \smallskip

% \resizebox{\textwidth}{!}{  
% \input{tables/roberta/friedman/test_set_acc/best_accuracy_1-qqp-friedman}
% }

% \end{sidewaystable}

%%%%%%%%%%%%%%%%%%%%%%%%%%%%%%%%%%%%%%%%%%%%%%%%%%%%%%%%%%%%

\end{document}

%% file: math_commands.tex
\usepackage{amsmath,amsfonts,bm}

\def\eqref#1{equation~\ref{#1}}
\def\1{\bm{1}}

\DeclareMathAlphabet{\mathsfit}{\encodingdefault}{\sfdefault}{m}{sl}
\SetMathAlphabet{\mathsfit}{bold}{\encodingdefault}{\sfdefault}{bx}{n}

%% file: tables/roberta/test-set-acc.tex
\begin{table}[h!]
\centering
\caption{Roberta-Large best \textbf{test set accuracy} on GLUE after 30\% filtering with influence scores. Methods are ordered from best to worst with average ranking across 10 runs per setup. The highest accuracy means are highlighted.  }
\label{tab:r-test-acc}
\resizebox{\linewidth}{!}{  
\begin{tabular}{|lll|cccccccccc|}
\hline
 Method             & Agg   & Layer   &   Rank & Win Rate                        & QNLI                                     & MRPC                                     & SST2                                     & QQP                                      & COLA                                     & MNLI                                     & RTE                                      & STSB                                     \\
\hline
 Full               &       &         &      0 & 0.92 {\footnotesize $\pm$ 0.12} & 89.5 {\footnotesize $\pm$ 1.2}           & 86.0 {\footnotesize $\pm$ 2.5}           & 94.7 {\footnotesize $\pm$ 0.8}           & 83.2 {\footnotesize $\pm$ 1.5}           & 83.0 {\footnotesize $\pm$ 1.2}           & 87.3 {\footnotesize $\pm$ 0.8}           & 76.1 {\footnotesize $\pm$ 3.0}           & 88.0 {\footnotesize $\pm$ 1.6}           \\
 \hline
 Cosine             & Rank  & 18-23   &      1 & 0.66 {\footnotesize $\pm$ 0.15} & 87.9 {\footnotesize $\pm$ 1.9}           & 77.0 {\footnotesize $\pm$ 2.6}           & 94.2 {\footnotesize $\pm$ 0.9}           & 79.9 {\footnotesize $\pm$ 4.9}           & 79.6 {\footnotesize $\pm$ 1.3}           & 86.3 {\footnotesize $\pm$ 2.6}           & 62.5 {\footnotesize $\pm$ 7.0}           & 86.9 {\footnotesize $\pm$ 1.3}           \\
 DataInf            & Mean  & 18-23   &      1 & 0.66 {\footnotesize $\pm$ 0.16} & 87.4 {\footnotesize $\pm$ 2.1}           & 78.1 {\footnotesize $\pm$ 2.1}           & 94.4 {\footnotesize $\pm$ 0.6}           & 80.0 {\footnotesize $\pm$ 3.1}           & 83.3 {\footnotesize $\pm$ 1.2}           & 85.5 {\footnotesize $\pm$ 2.2}           & 60.5 {\footnotesize $\pm$ 5.6}           & 86.7 {\footnotesize $\pm$ 1.5}           \\
\hline
\end{tabular}

}

\end{table}

%% file: tables/llama/test-set-acc.tex
\begin{table}[h]
\centering
\caption{Llama-3.2 1B best \textbf{test set accuracy} on GLUE after 30\% filtering with influence scores. Methods are ordered from best to worst with average ranking across 10 runs per setup. The highest accuracy means are highlighted. In contrast to Roberta and Mistral, influence methods demonstrate poor performance of noise detection compared to Random baseline. Llama noise distribution on figure~\ref{fig:l-noise-distr} confirms that a lot of mislabeled training samples become influential. }
\label{tab:l-test-acc}
\resizebox{\linewidth}{!}{  
\begin{tabular}{|lll|cccccccccc|}
\hline
 Method             & Agg   & Layer   &   Rank & Win Rate                        & QNLI                                     & MRPC                                     & SST2                                     & QQP                                      & COLA                                     & MNLI                                     & RTE                                      & STSB                                     \\
\hline
 Full               &       &         &      0 & 0.97 {\footnotesize $\pm$ 0.12} & 83.8 {\footnotesize $\pm$ 1.5}           & 79.9 {\footnotesize $\pm$ 2.5}           & 93.5 {\footnotesize $\pm$ 0.9}           & 84.4 {\footnotesize $\pm$ 1.9}           & 83.4 {\footnotesize $\pm$ 0.5}           & 82.0 {\footnotesize $\pm$ 1.7}           & 64.9 {\footnotesize $\pm$ 4.3}           & 85.9 {\footnotesize $\pm$ 1.3}           \\
 \hline
 Cosine             & Rank  & 04-07   &      1 & 0.65 {\footnotesize $\pm$ 0.15} & 77.5 {\footnotesize $\pm$ 3.5}           & 70.2 {\footnotesize $\pm$ 1.3}           & 91.3 {\footnotesize $\pm$ 1.7}           & 73.3 {\footnotesize $\pm$ 2.3}           & 75.8 {\footnotesize $\pm$ 2.0}           & \textbf{77.3} {\footnotesize $\pm$ 2.7 } & 61.4 {\footnotesize $\pm$ 4.0}           & 80.9 {\footnotesize $\pm$ 3.0}           \\
 Random             &       &         &      1 & 0.64 {\footnotesize $\pm$ 0.15} & 76.7 {\footnotesize $\pm$ 5.1}           & 71.6 {\footnotesize $\pm$ 2.5}           & 89.5 {\footnotesize $\pm$ 2.0}           & \textbf{80.2} {\footnotesize $\pm$ 2.1 } & \textbf{80.5} {\footnotesize $\pm$ 0.9 } & 73.3 {\footnotesize $\pm$ 2.7}           & 56.3 {\footnotesize $\pm$ 3.9}           & \textbf{82.2} {\footnotesize $\pm$ 2.5 } \\
\hline
\end{tabular}

}

\end{table}

%% file: tables/mistral/test-set-acc.tex
\begin{table}[h!]
\centering
\caption{Mistral 7B best \textbf{test set accuracy} on GLUE after 30\% filtering with influence scores. Methods are ordered from best to worst with average ranking across 10 runs per setup. The highest accuracy means are highlighted.  }
\label{tab:m-test-acc}
\resizebox{\linewidth}{!}{  
\begin{tabular}{|lll|cccccccccc|}
\hline
 Method             & Agg   & Layer   &   Rank & Win Rate                        & QNLI                                     & MRPC                                     & SST2                                     & QQP                                      & COLA                                     & MNLI                                     & RTE                                      & STSB                                     \\
\hline
 Full               &       &         &      0 & 0.96 {\footnotesize $\pm$ 0.12} & 87.1 {\footnotesize $\pm$ 1.6}           & 77.2 {\footnotesize $\pm$ 1.2}           & 93.8 {\footnotesize $\pm$ 1.1}           & 84.2 {\footnotesize $\pm$ 2.1}           & 83.4 {\footnotesize $\pm$ 1.3}           & 82.1 {\footnotesize $\pm$ 2.8}           & 69.0 {\footnotesize $\pm$ 3.5}           & 78.5 {\footnotesize $\pm$ 3.2}           \\
 \hline
 DataInf            & Vote  & 00-07   &      1 & 0.84 {\footnotesize $\pm$ 0.22} & 82.9 {\footnotesize $\pm$ 1.7}           & \textbf{74.2} {\footnotesize $\pm$ 2.4 } & 89.3 {\footnotesize $\pm$ 2.2}           & \textbf{80.7} {\footnotesize $\pm$ 2.3 } & 80.6 {\footnotesize $\pm$ 1.9}           & 79.0 {\footnotesize $\pm$ 3.0}           & 66.6 {\footnotesize $\pm$ 2.9}           & 77.2 {\footnotesize $\pm$ 3.2}           \\
 TracIn             & Vote  & 00-07   &      1 & 0.83 {\footnotesize $\pm$ 0.23} & 83.5 {\footnotesize $\pm$ 1.6}           & 74.1 {\footnotesize $\pm$ 2.7}           & 89.2 {\footnotesize $\pm$ 2.1}           & 80.3 {\footnotesize $\pm$ 2.4}           & 80.1 {\footnotesize $\pm$ 1.6}           & 78.5 {\footnotesize $\pm$ 2.8}           & 66.6 {\footnotesize $\pm$ 3.2}           & 76.7 {\footnotesize $\pm$ 3.1}           \\
\hline
\end{tabular}

}

\end{table}

%% file: tables/llama/tun2-best_accuracy_1-vb.tex
\begin{table}[h!]
\centering
\caption{Llama-3.2 1B downstream performance of 8th and 9th attention layer influence scores measured for LoRA modules B that are attached to value projections.}
\label{tab:l-best}
\resizebox{\linewidth}{!}{  
\begin{tabular}{llllllllllll}
\hline
 Method             & Agg   & Layer   &   Rank                        & QNLI                                     & MRPC                                     & SST2                                     & QQP                                      & COLA                                     & MNLI                                     & RTE                                      & STSB     \\
\hline
 Full     &       &         & 1.20 {\footnotesize $\pm$ 0.80}  & 1.0 {\footnotesize $\pm$ 0.0}           & 1.0 {\footnotesize $\pm$ 0.0}           & 1.6 {\footnotesize $\pm$ 1.9}           & 1.0 {\footnotesize $\pm$ 0.0}           & 1.0 {\footnotesize $\pm$ 0.0}           & 1.0 {\footnotesize $\pm$ 0.0}           & 2.1 {\footnotesize $\pm$ 1.0}           & 1.0 {\footnotesize $\pm$ 0.0}           \\\hline
 Cosine   & Vote  & v-b-9   & 6.80 {\footnotesize $\pm$ 4.10}  & 6.0 {\footnotesize $\pm$ 4.5}           & 8.9 {\footnotesize $\pm$ 4.1}           & 12.8 {\footnotesize $\pm$ 2.2}          & 5.1 {\footnotesize $\pm$ 2.3}           & \textbf{4.6} {\footnotesize $\pm$ 1.5 } & 7.3 {\footnotesize $\pm$ 3.5}           & 6.6 {\footnotesize $\pm$ 3.7}           & 3.5 {\footnotesize $\pm$ 1.3}           \\
 Cosine   & Vote  & v-b-8   & 6.90 {\footnotesize $\pm$ 4.40}  & \textbf{5.2} {\footnotesize $\pm$ 2.0 } & 11.2 {\footnotesize $\pm$ 5.1}          & 12.0 {\footnotesize $\pm$ 3.4}          & 5.5 {\footnotesize $\pm$ 2.2}           & 6.4 {\footnotesize $\pm$ 2.9}           & 5.2 {\footnotesize $\pm$ 3.3}           & 7.2 {\footnotesize $\pm$ 5.0}           & \textbf{2.8} {\footnotesize $\pm$ 1.0 } \\
 DataInf  & Vote  & v-b-8   & 9.20 {\footnotesize $\pm$ 4.50}  & 13.1 {\footnotesize $\pm$ 4.9}          & 7.1 {\footnotesize $\pm$ 5.1}           & 9.8 {\footnotesize $\pm$ 2.5}           & 8.0 {\footnotesize $\pm$ 4.3}           & 6.4 {\footnotesize $\pm$ 3.7}           & 11.6 {\footnotesize $\pm$ 3.4}          & 10.6 {\footnotesize $\pm$ 4.8}          & 7.0 {\footnotesize $\pm$ 2.6}           \\
 Random   &       &         & 9.40 {\footnotesize $\pm$ 6.00}  & 8.3 {\footnotesize $\pm$ 4.8}           & 10.6 {\footnotesize $\pm$ 7.1}          & 13.3 {\footnotesize $\pm$ 5.1}          & \textbf{3.8} {\footnotesize $\pm$ 2.1 } & 6.8 {\footnotesize $\pm$ 3.2}           & 14.5 {\footnotesize $\pm$ 2.2}          & 14.4 {\footnotesize $\pm$ 6.5}          & 3.6 {\footnotesize $\pm$ 1.4}           \\
 DataInf  & Vote  & v-b-9   & 9.80 {\footnotesize $\pm$ 5.20}  & 16.0 {\footnotesize $\pm$ 4.0}          & 5.5 {\footnotesize $\pm$ 5.0}           & 8.8 {\footnotesize $\pm$ 3.3}           & 7.6 {\footnotesize $\pm$ 4.8}           & 5.2 {\footnotesize $\pm$ 3.0}           & 15.0 {\footnotesize $\pm$ 3.3}          & 11.0 {\footnotesize $\pm$ 3.1}          & 9.2 {\footnotesize $\pm$ 3.3}           \\
 TracIn   & Vote  & v-b-8   & 9.90 {\footnotesize $\pm$ 4.80}  & 14.6 {\footnotesize $\pm$ 3.0}          & 7.6 {\footnotesize $\pm$ 6.6}           & 8.6 {\footnotesize $\pm$ 4.4}           & 9.4 {\footnotesize $\pm$ 4.0}           & 7.2 {\footnotesize $\pm$ 2.5}           & 11.8 {\footnotesize $\pm$ 4.3}          & 11.2 {\footnotesize $\pm$ 5.5}          & 8.9 {\footnotesize $\pm$ 3.3}           \\
 Cosine   & Rank  & v-b-8   & 10.40 {\footnotesize $\pm$ 5.10} & 6.6 {\footnotesize $\pm$ 3.7}           & 15.8 {\footnotesize $\pm$ 3.3}          & 8.1 {\footnotesize $\pm$ 5.4}           & 12.2 {\footnotesize $\pm$ 3.5}          & 14.4 {\footnotesize $\pm$ 2.5}          & \textbf{4.6} {\footnotesize $\pm$ 2.8 } & 9.2 {\footnotesize $\pm$ 5.3}           & 12.2 {\footnotesize $\pm$ 1.7}          \\
 Cosine   & Rank  & v-b-9   & 10.60 {\footnotesize $\pm$ 5.00} & 8.8 {\footnotesize $\pm$ 4.4}           & 12.7 {\footnotesize $\pm$ 5.4}          & 6.8 {\footnotesize $\pm$ 4.1}           & 12.0 {\footnotesize $\pm$ 3.4}          & 12.4 {\footnotesize $\pm$ 3.2}          & 7.6 {\footnotesize $\pm$ 4.1}           & 8.4 {\footnotesize $\pm$ 4.7}           & 16.4 {\footnotesize $\pm$ 3.9}          \\
 TracIn   & Vote  & v-b-9   & 10.60 {\footnotesize $\pm$ 5.60} & 17.4 {\footnotesize $\pm$ 3.2}          & \textbf{5.2} {\footnotesize $\pm$ 4.4 } & 9.4 {\footnotesize $\pm$ 4.2}           & 9.2 {\footnotesize $\pm$ 4.8}           & 7.0 {\footnotesize $\pm$ 3.0}           & 15.2 {\footnotesize $\pm$ 4.5}          & 12.0 {\footnotesize $\pm$ 4.9}          & 9.0 {\footnotesize $\pm$ 4.0}           \\
 Cosine   & Mean  & v-b-8   & 10.70 {\footnotesize $\pm$ 5.20} & 10.5 {\footnotesize $\pm$ 5.6}          & 15.0 {\footnotesize $\pm$ 4.6}          & 5.3 {\footnotesize $\pm$ 3.3}           & 13.4 {\footnotesize $\pm$ 3.3}          & 15.2 {\footnotesize $\pm$ 3.0}          & 9.5 {\footnotesize $\pm$ 4.2}           & 7.4 {\footnotesize $\pm$ 5.6}           & 9.2 {\footnotesize $\pm$ 3.1}           \\
 DataInf  & Mean  & v-b-8   & 11.10 {\footnotesize $\pm$ 5.60} & 10.4 {\footnotesize $\pm$ 4.3}          & 12.6 {\footnotesize $\pm$ 5.3}          & 13.6 {\footnotesize $\pm$ 6.1}          & 9.9 {\footnotesize $\pm$ 6.6}           & 8.8 {\footnotesize $\pm$ 5.2}           & 12.6 {\footnotesize $\pm$ 5.1}          & 9.3 {\footnotesize $\pm$ 6.7}           & 11.8 {\footnotesize $\pm$ 5.5}          \\
 Cosine   & Mean  & v-b-9   & 11.40 {\footnotesize $\pm$ 5.40} & 10.2 {\footnotesize $\pm$ 2.4}          & 13.4 {\footnotesize $\pm$ 4.2}          & \textbf{5.0} {\footnotesize $\pm$ 3.7 } & 13.6 {\footnotesize $\pm$ 3.2}          & 16.5 {\footnotesize $\pm$ 2.9}          & 12.5 {\footnotesize $\pm$ 5.5}          & \textbf{6.2} {\footnotesize $\pm$ 5.4 } & 13.6 {\footnotesize $\pm$ 4.7}          \\
 DataInf  & Rank  & v-b-8   & 11.70 {\footnotesize $\pm$ 4.60} & 8.6 {\footnotesize $\pm$ 4.4}           & 11.4 {\footnotesize $\pm$ 4.5}          & 10.6 {\footnotesize $\pm$ 4.3}          & 16.4 {\footnotesize $\pm$ 1.6}          & 16.2 {\footnotesize $\pm$ 2.5}          & 7.4 {\footnotesize $\pm$ 3.0}           & 13.0 {\footnotesize $\pm$ 3.8}          & 10.1 {\footnotesize $\pm$ 3.4}          \\
 TracIn   & Rank  & v-b-8   & 12.70 {\footnotesize $\pm$ 4.80} & 10.7 {\footnotesize $\pm$ 6.0}          & 12.2 {\footnotesize $\pm$ 3.1}          & 12.0 {\footnotesize $\pm$ 3.8}          & 17.0 {\footnotesize $\pm$ 2.3}          & 16.4 {\footnotesize $\pm$ 3.1}          & 7.0 {\footnotesize $\pm$ 4.0}           & 14.0 {\footnotesize $\pm$ 4.3}          & 12.6 {\footnotesize $\pm$ 3.1}          \\
 TracIn   & Mean  & v-b-8   & 12.80 {\footnotesize $\pm$ 5.70} & 13.2 {\footnotesize $\pm$ 4.8}          & 13.6 {\footnotesize $\pm$ 5.5}          & 16.6 {\footnotesize $\pm$ 3.4}          & 12.0 {\footnotesize $\pm$ 7.0}          & 8.0 {\footnotesize $\pm$ 5.2}           & 14.3 {\footnotesize $\pm$ 5.5}          & 10.4 {\footnotesize $\pm$ 6.0}          & 14.6 {\footnotesize $\pm$ 5.1}          \\
 DataInf  & Rank  & v-b-9   & 12.80 {\footnotesize $\pm$ 4.50} & 10.8 {\footnotesize $\pm$ 4.9}          & 11.7 {\footnotesize $\pm$ 4.1}          & 12.4 {\footnotesize $\pm$ 4.9}          & 15.1 {\footnotesize $\pm$ 1.7}          & 15.8 {\footnotesize $\pm$ 3.0}          & 8.4 {\footnotesize $\pm$ 4.5}           & 12.6 {\footnotesize $\pm$ 5.0}          & 15.2 {\footnotesize $\pm$ 3.0}          \\
 DataInf  & Mean  & v-b-9   & 13.10 {\footnotesize $\pm$ 5.60} & 11.6 {\footnotesize $\pm$ 6.1}          & 12.4 {\footnotesize $\pm$ 3.9}          & 16.0 {\footnotesize $\pm$ 7.3}          & 9.0 {\footnotesize $\pm$ 4.9}           & 10.6 {\footnotesize $\pm$ 4.9}          & 16.2 {\footnotesize $\pm$ 5.2}          & 13.2 {\footnotesize $\pm$ 4.7}          & 16.0 {\footnotesize $\pm$ 3.6}          \\
 TracIn   & Rank  & v-b-9   & 14.00 {\footnotesize $\pm$ 4.90} & 11.5 {\footnotesize $\pm$ 5.8}          & 11.1 {\footnotesize $\pm$ 3.3}          & 12.2 {\footnotesize $\pm$ 5.9}          & 17.2 {\footnotesize $\pm$ 3.3}          & 17.9 {\footnotesize $\pm$ 2.0}          & 10.5 {\footnotesize $\pm$ 4.3}          & 15.2 {\footnotesize $\pm$ 4.7}          & 16.5 {\footnotesize $\pm$ 2.7}          \\
 TracIn   & Mean  & v-b-9   & 14.80 {\footnotesize $\pm$ 5.30} & 15.6 {\footnotesize $\pm$ 5.7}          & 11.2 {\footnotesize $\pm$ 5.1}          & 15.2 {\footnotesize $\pm$ 7.1}          & 12.6 {\footnotesize $\pm$ 5.9}          & 13.0 {\footnotesize $\pm$ 4.3}          & 17.8 {\footnotesize $\pm$ 4.2}          & 16.0 {\footnotesize $\pm$ 3.1}          & 17.0 {\footnotesize $\pm$ 3.4}          \\
\hline
\end{tabular}}
\end{table}